\documentclass[nohyperref]{article}
\usepackage{mathrsfs,amsmath,amsfonts,amssymb,comment}
\usepackage{mathtools,xcolor,enumerate}
\usepackage{dsfont}
\usepackage{bm}
\usepackage{ifthen}
\usepackage{tabularx,multirow,array,booktabs}
\usepackage{colortbl}
\usepackage[pdftex]{graphicx}
\graphicspath{{figures/}}%
\usepackage{float}
\usepackage{subcaption} 
\newcommand{\R}{\ifmmode\mathbb{R}\else$\mathbb{R}$\fi}
\newcommand{\N}{\ifmmode\mathbb{N}\else$\mathbb{N}$\fi}
\newcommand{\Z}{\ifmmode\mathbb{Z}\else$\mathbb{Z}$\fi}
\newcommand{\Q}{\ifmmode\mathbb{Q}\else$\mathbb{Q}$\fi}
\newcommand{\A}{\ifmmode\mathbb{A}\else$\mathbb{A}$\fi}

\newcommand{\bmx}{{\bm{x}}}
\newcommand{\bmz}{{\bm{z}}}

\newcommand{\bmw}{{\bm{w}}}
\newcommand{\bme}{{\bm{e}}}
\newcommand{\bma}{{\bm{a}}}

\newcommand{\bmh}{{\bm{h}}}
\newcommand{\bmb}{{\bm{b}}}
\newcommand{\bmxi}{{\bm{\xi}}}
\newcommand{\bmpsi}{{\bm{\psi}}}

\newcommand{\bmbeta}{{\bm{\beta}}}
\newcommand{\bmu}{{\bm{u}}}
\newcommand{\bmv}{{\bm{v}}}
\newcommand{\bmy}{{\bm{y}}}
\newcommand{\bmg}{{\bm{g}}}
\newcommand{\bmG}{{\bm{G}}}

\newcommand{\bmF}{{\bm{F}}}
\newcommand{\bmtheta}{{\bm{\theta}}}
\newcommand{\bmTheta}{{\bm{\Theta}}}

\newcommand{\bmA}{\bm{A}}
\newcommand{\bmphi}{{\bm{\phi}}}

\newcommand{\bmPhi}{{\bm{\Phi}}}
\newcommand{\bmzero}{{\bm{0}}}

\newcommand{\calO}{{\mathcal{O}}}
\newcommand{\calC}{{\mathcal{C}}}

\newcommand{\calS}{{\mathcal{S}}}
\newcommand{\calH}{{\mathcal{H}}}

\newcommand{\calA}{{\mathcal{A}}}

\newcommand{\calL}{{\mathcal{L}}}
\newcommand{\calT}{{\mathcal{T}}}

\newcommand{\hatN}{{\widehat{N}}}
\newcommand{\hatL}{{\widehat{L}}}

\newcommand{\hatd}{{\widehat{d}}}

\newcommand{\tildephi}{{\widetilde{\phi}}}

\newcommand{\tilded}{{\widetilde{d}}}

\newcommand{\tildex}{{\widetilde{x}}}
\newcommand{\tildey}{{\widetilde{y}}}

\newcommand{\tildeA}{{\widetilde{A}}}
\newcommand{\tildef}{{\widetilde{f}}}

\newcommand{\tildedelta}{{\widetilde{\delta}}}

\newcommand{\caltildeC}{{\mathcal{\widetilde{C}}}}
\newcommand{\bmtildeh}{{\bm{\widetilde{h}}}}
\newcommand{\bmtildeg}{{\bm{\widetilde{g}}}}
\newcommand{\bmhatg}{{\bm{\widehat{g}}}}
\newcommand{\bmcalG}{{\bm{\mathcal{G}}}}
\newcommand{\bmhatG}{{\bm{\widehat{G}}}}

\let\tildecalL\caltildeL
\newcommand{\hatcalL}{{\mathcal{\widehat{L}}}}
\newcommand{\bmcalL}{{\bm{\mathcal{L}}}}
\newcommand{\bmtildecalL}{{\bm{\mathcal{\widetilde{L}}}}}
\newcommand{\bmhatcalL}{{\bm{\mathcal{\widehat{L}}}}}

\def\one{{\ensuremath{\mathds{1}}}}
\newcommand{\middleValue}{{\tn{mid}}}

\newcommand{\nn}[5][]{\ensuremath{	{\hspace{0.6pt}\mathcal{N}\hspace{-1.9pt}\mathcal{N}\hspace{-0.725pt}}
	#1\{#2,\hspace{1.7pt} #3;\hspace{2.97pt} \R^{#4}\hspace{-1.0298pt}\to\hspace{-0.98pt}\R^{#5}#1\}
}}

\newcommand{\nnOneD}[5][]{\ensuremath{
		{\hspace{0.6pt}\mathcal{N}\hspace{-1.9pt}\mathcal{N}\hspace{-0.725pt}}
		#1\{#2,\hspace{1.7pt} #3;\hspace{2.97pt} {#4}\hspace{-1.0298pt}\to\hspace{-0.98pt}{#5}#1\}
}}

\newcommand{\bin}{\tn{bin}\hspace{1.5pt}}
\newcommand{\mystep}[2]{\par \vspace{0.25cm}\noindent\textbf{\hspace{8pt}Step }$#1\colon$ #2 \vspace{0.18cm} \par }

\usepackage{tikz}
\newcommand{\myto}[2][1]{\mathop{
		\vcenter{\hbox{\scalebox{1}[#1]{\tikz{\draw[->,line width=0.72pt] (0,0.5) to (0.69*#2,0.5);}}}}
}}

\let\epsilon\varepsilon
\let\eps\varepsilon

\let\tn\textnormal

\let\cdots\customcdots
\let\dotsc\cdots
\let\myforall\forall
\def\forall{{\myforall\ }}
\let\myexists\exists
\def\exists{{\myexists\ }}
\long\def\black#1{{\color{black}#1}}

\definecolor{mygray}{RGB}{230,230,230}

\newif\ifarXiv
\arXivtrue

\usepackage[accepted]{icml2023}

\usepackage{xurl}
\usepackage[bookmarks,colorlinks,breaklinks=true]{hyperref}
\definecolor{mydarkblue}{rgb}{0,0.08,0.45}
\hypersetup{ %
	pdftitle={On Enhancing Expressive Power via Compositions of Single Fixed-Size ReLU Network},
	pdfauthor={Shijun Zhang},
	pdfsubject={Proceedings of the International Conference on Machine Learning 2023},
	pdfkeywords={function composition, parameter sharing, deep neural network, dynamical system, function approximation},
	pdfborder=0 0 0,
	pdfpagemode=UseNone,
	colorlinks=true,
	linkcolor=mydarkblue,
	citecolor=mydarkblue,
	filecolor=mydarkblue,
	urlcolor=mydarkblue,
}

\makeatletter\ifarXiv
\def\ICML@appearing{\vspace*{-22.5pt}}
\fi\makeatother

\usepackage{amsmath}
\usepackage{amssymb}
\usepackage{mathtools}
\usepackage{amsthm}

\usepackage[capitalize,noabbrev]{cleveref}

\theoremstyle{plain}
\newtheorem{theorem}{Theorem}[section]
\newtheorem{proposition}[theorem]{Proposition}
\newtheorem{lemma}[theorem]{Lemma}
\newtheorem{corollary}[theorem]{Corollary}
\theoremstyle{definition}

\theoremstyle{remark}

\usepackage[textsize=tiny]{todonotes}

\icmltitlerunning{On Enhancing Expressive Power via Compositions of Single Fixed-Size ReLU Network}

\begin{document}
\makeatletter
\def\infigenv#1{\ifthenelse{\equal{\@currenvir}{figure}}{#1}{}}
\def\intabenv#1{\ifthenelse{\equal{\@currenvir}{table}}{#1}{}}
\makeatother

\let\mycaption\caption
\def\caption#1{
	\infigenv{\vskip -6.6pt} 
	\mycaption{#1}
	\intabenv{\vskip  5pt} 
	}

\captionsetup[subfigure]{aboveskip=2.1pt}


\let\myfigure\figure
\def\figure{\vskip -1pt\myfigure}
\let\myendfigure\endfigure
\def\endfigure{\myendfigure \vskip -1pt}

\let\mytable\table
\def\table{\vskip -2pt\mytable}
\let\myendtable\endtable
\def\endtable{\myendtable \vskip 4pt}


\twocolumn[
\icmltitle{On Enhancing Expressive Power via\\ Compositions of Single Fixed-Size ReLU Network}



\icmlsetsymbol{equal}{*}

\begin{icmlauthorlist}
\icmlauthor{Shijun Zhang}{duke}
\icmlauthor{Jianfeng Lu}{duke}
\icmlauthor{Hongkai Zhao}{duke}
\end{icmlauthorlist}

\icmlaffiliation{duke}{Department of Mathematics, Duke University, USA}

\icmlcorrespondingauthor{Shijun Zhang}{shijun.zhang@duke.edu}

\icmlkeywords{function composition, parameter sharing, deep neural network, dynamical system, function approximation}

\vskip 0.3in
]



\printAffiliationsAndNotice{}

\begin{abstract}
This paper explores the expressive power of deep neural networks through the framework of function compositions. We demonstrate that the repeated compositions of a single fixed-size ReLU network exhibit surprising expressive power, despite the limited expressive capabilities of the individual network itself. Specifically, we prove by construction that $\mathcal{L}_2\circ \bm{g}^{\circ r}\circ \bm{\mathcal{L}}_1$ can approximate $1$-Lipschitz continuous functions on $[0,1]^d$ with an error $\mathcal{O}(r^{-1/d})$, where $\bm{g}$ is realized by a fixed-size ReLU network, $\bm{\mathcal{L}}_1$ and $\mathcal{L}_2$ are two affine linear maps matching the dimensions, and $\bm{g}^{\circ r}$ denotes the $r$-times composition of $\bm{g}$. Furthermore, we extend such a result to generic continuous functions on $[0,1]^d$ with the approximation error characterized by the modulus of continuity. Our results reveal that a continuous-depth network generated via a dynamical system has immense approximation power even if its dynamics function is time-independent and realized by a fixed-size ReLU network.
\end{abstract}

\section{Introduction}
\label{sec:intro}

In recent years, 
there has been a notable increase in the exploration of the expressive power of deep neural networks,
driven by their 
impressive
success in various learning tasks. 
The increasing size of deep neural network models poses significant challenges in terms of training and computational requirements. 
Consequently,
numerous techniques have emerged to compress and expedite these models, 
with the goal of alleviating the associated computational complexity.
These techniques predominantly center around parameter-sharing schemes, which efficiently reduce the number of parameters, leading to reductions in memory usage and computation costs.


This paper explores the expressive power of deep neural networks, approaching it from the standpoint of function compositions. We focus on a novel network architecture constructed through the repeated compositions of a single fixed-size network, enabling parameter sharing. 
To illustrate our ideas and concepts, we specifically utilize the rectified linear unit (ReLU) activation function.
Our investigation reveals that
the repeated compositions of a single fixed-size ReLU network demonstrate surprising expressive power, even though the individual network itself has
limited expressive capabilities.
These findings provide new insights into the potential of parameter-sharing schemes in neural networks, showcasing their ability to reduce computational complexity while preserving a high level of expressive power.

For ease of notation, we employ $\nn{N}{L}{d_1}{d_2}$ to represent the set of functions $\bmphi:\R^{d_1}\to\R^{d_2}$ that can be
realized by ReLU networks of width $N\in \N^+$ and depth $L\in \N^+$. In our context, the width of a network means the maximum number of neurons in a hidden layer
and the depth refers to the number of hidden layers.
Let $\bm{g}^{\circ r}$ denote the $r$-times composition of $\bm{g}$, e.g., $\bmg^{\circ 3}=\bmg\circ \bmg\circ \bmg$. In the degenerate case, $\bmg^{\circ 0}$ 
represents
the identity map.
We use $C([0,1]^d)$ to denote the set of continuous functions on $[0,1]^d$ and define the modulus of continuity of a continuous function $f\in C([0,1]^d)$ via
{\fontsize{9.5}{12}\selectfont
\begin{equation*}
	\omega_f(t)\coloneqq \sup\big\{|f(\bmx)-f(\bmy)|: \|\bmx-\bmy\|_2\le t,\ \bmx,\bmy\in [0,1]^d\big\}
\end{equation*}}for any $t\ge0$.
Under these settings, we can construct 
$\mathcal{L}_2\circ \bm{g}^{\circ r}\circ \bm{\mathcal{L}}_1$ to approximate a continuous function $f\in C([0,1]^d)$ with an error $\mathcal{O}\big(\omega_f(r^{-1/d})\big)$, where $\bm{\mathcal{L}}_1$ and $\mathcal{L}_2$ are two affine linear maps and $\bm{g}$ is realized by a fixed-size ReLU network, as shown in the theorem below.

\begin{theorem}
	\label{thm:main:Lp}
	Given  any $f\in C([0,1]^d)$, $r\in \N^+$, and $p\in [1,\infty)$,
	there exist $\bmg\in \nn{69d+48}{5}{5d+5}{5d+5}$
	and two affine linear maps
	$\bmcalL_1:\R^d\to\R^{5d+5}$ and $\calL_2:\R^{5d+5}\to\R$
	such that 
	\begin{equation*}
		\big\|\calL_2\circ \bmg^{\circ (3r+1)}\circ \bmcalL_1-f\big\|_{L^p([0,1]^d)}\le 6\sqrt{d}\,\omega_f(r^{-1/d}).
	\end{equation*}
\end{theorem}

It should be noted that in Theorem~\ref{thm:main:Lp}, the affine linear maps $\bmcalL_1$ and $\calL_2$ are used to ensure matching dimensions and can be replaced by various other functions that achieve the desired input and output dimensions. 
In our research, we choose a straightforward approach by considering them as affine linear maps.
In Theorem~\ref{thm:main:Lp}, we propose a novel network architecture constructed via repeated compositions of a single sub-network, which will be referred to as 
repeated-composition networks (RCNets).
The hypothesis space of the RCNet corresponding to $\bmg$ is defined as
\begin{equation*}
\begin{split}
         \calH(\bmg)\coloneqq \Big\{\calL_2\circ \bmg^{\circ r}\circ\bmcalL_1:\, r\in \N,\,
         \bmcalL_1 \tn{ and } \calL_2 \tn{ are affine}\Big\}.
\end{split}
\end{equation*}
Then we have an immediate corollary as follows.
\begin{corollary}\label{coro:dense:in:Lp}
    Given any  $p\in [1,\infty)$, suppose $\calH(\bmg)$ is defined as mentioned above and set $\bmcalG=\nn{69d+48}{5}{5d+5}{5d+5}$. Then $\calH=\cup_{\bmg\in\bmcalG}\calH(\bmg)$  is dense in $L^p([0,1]^d)$ in terms of the $L^p$-norm.
\end{corollary}

The proof of Corollary~\ref{coro:dense:in:Lp} is straightforward.
Theorem~\ref{thm:main:Lp} implies $\calH$ is dense in $C([0,1]^d)$ in terms of the $L^p$-norm for any $p\in [1,\infty)$. 
Recall that $C([0,1]^d)$ is  dense in the Lebesgue spaces $L^p([0,1]^d)$ for any $p\in [1,\infty)$. Therefore, we can conclude that $\calH$ is dense in $L^p([0,1]^d)$ in terms of the $L^p$-norm for any $p\in [1,\infty)$.
Furthermore, it should be noted that the set $\bmcalG$ in Corollary~\ref{coro:dense:in:Lp} is generated by a fixed-size ReLU network. As a result, $\bmcalG$ is a set of continuous piecewise linear functions with (at most) a fixed number of pieces.

It is important to note that the approximation error in Theorem~\ref{thm:main:Lp} is quantified by the $L^p$-norm for any $p\in [1,\infty)$. However, it is possible to extend this result to the $L^\infty$-norm as well, although the associated constants will be significantly larger.

\begin{theorem}
	\label{thm:main:Linfty}
	Given  any $f\in C([0,1]^d)$ and $r\in \N^+$,
	there exist $\bmg\in \nn{4^{d+5}d}{3+2d}{\tilded}{\tilded}$
	and two affine linear maps
	$\bmcalL_1:\R^d\to\R^{\tilded}$ and $\calL_2:\R^{\tilded}\to\R$
	such that 
	\begin{equation*}		\big|\calL_2\circ \bmg^{\circ (3r+2d-1)}\circ \bmcalL_1(\bmx)-f(\bmx)\big|\le 6\sqrt{d}\,\omega_f(r^{-1/d})
	\end{equation*}
	for any $\bmx\in [0,1]^d$, where $\tilded=3^d(5d+4)-1$.
\end{theorem}

The main ideas for proving Theorems~\ref{thm:main:Lp} and \ref{thm:main:Linfty} are provided in  Section~\ref{sec:ideas} and the detailed proofs of these two theorems can be found in Section~\ref{sec:proof:main} of the  appendix.

In general, it is challenging to simplify the approximation error in Theorem~\ref{thm:main:Lp} (or \ref{thm:main:Linfty})
due to the complexity of $\omega_f(\cdot)$. 
However, in the case of special target function spaces like H\"older continuous function space, one can simplify the approximation error to make its dependence on $r$ explicit. If  $f$ is an H{\"o}lder continuous function on $[0,1]^d$ of order $\alpha\in(0,1]$ with an H\"older constant $\lambda>0$, we have 
\begin{equation*}
	|f(\bmx)-f(\bmy)|\leq \lambda \|\bmx-\bmy\|_2^\alpha\quad \tn{for any $\bmx,\bmy\in[0,1]^d$,}
\end{equation*}
implying $\omega_f(t)\le \lambda t^\alpha$ for any $t\ge 0$. Thus, the approximation error in Theorem~\ref{thm:main:Lp} (or \ref{thm:main:Linfty})
can be simplified to $6\lambda\sqrt{d}\,r^{-\alpha/d}$. 
In the special case of $\alpha=1$, where $f$ is a Lipschitz continuous function with a Lipschitz constant $\lambda>0$, the approximation error can be further simplified to $6\lambda\sqrt{d}\,r^{-1/d}$.

A constant-width ReLU network of depth $\calO(r)$ can be represented as $\calL_2\circ \bmg_r\circ\cdots\circ\bmg_2 \circ \bmg_1\circ \bmcalL_1$, where $\bmcalL_1$ and $\calL_2$ are affine linear maps and each $\bmg_i$ is a fixed-size ReLU network. 
It has been shown in \cite{shijun:2,yarotsky18a,shijun:thesis}  that the optimal approximation error is $\calO(r^{-2/d})$ when using $\calL_2\circ \bmg_r\circ \cdots \circ \bmg_2\circ \bmg_1\circ \bmcalL_1$ to approximate $1$-Lipschitz continuous functions on $[0,1]^d$. In contrast, our RCNet architecture $\calL_2\circ \bmg^{\circ r}\circ \bmcalL_1$ can approximate $1$-Lipschitz continuous functions on $[0,1]^d$ with an error $\calO(r^{-1/d})$, where $\bmg$ is a fixed-size ReLU network. That means, at a price of a slightly worse approximation error, our RCNet architecture $\calL_2\circ \bmg^{\circ r}\circ \bmcalL_1$ essentially shares most of the parameters in  $\calL_2\circ \bmg_r\circ \cdots \circ\bmg_2\circ \bmg_1\circ \bmcalL_1$ and reduce trainable parameters to a constant. 
Furthermore, our RCNet architecture $\mathcal{L}_2 \circ \bmg^{\circ r} \circ \bmcalL_1$ is anticipated to exhibit improved gradient behavior compared to $\mathcal{L}_2 \circ \bmg_r \circ \cdots \circ \bmg_2 \circ \bmg_1 \circ \bmcalL_1$ as the gradient with respect to the parameters in $\bmg$ is less likely to vanish for larger values of $r$.

Next, we point out some relations between our approximation results and dynamical systems. Our results reveal that a continuous-depth network generated via a dynamical system has enormous approximation power even if the dynamics is time-invariant
and realized by a fixed-size ReLU network. 
 Let us now delve into further details regarding this matter.
A dynamical system is generally described by an ordinary differential equation
(ODE)
\begin{equation}
    \label{eq:dynamical:system:def}
    \tfrac{d }{d t}\bmz(t)= \bmF\big(\bmz(t),t,\bmtheta\big),\quad t\in [0,T],\quad  \bmz(0)=\bmz_0,
\end{equation}
where $\bmF:\R^{n+1}\times \bmTheta\to\R^n$ is the dynamics function of this dynamical system, parameterized with $\bmtheta\in\bmTheta$, where $\bmTheta$ is the parameter space.  

For any $\bmy=\bmz_0\in\R^n$, $\bmz(T)$ can be regarded as a function of $\bmy$ and we denote this function by $\bmPhi(\cdot,\bmtheta):\R^n\to\R^n$.
Such a map is known as the flow map (or Poincar\'e map) of the dynamical system \eqref{eq:dynamical:system:def}. Then we can use $\calL_2\circ \bmPhi(\cdot , \bmtheta)\circ \bmcalL_1$ to approximate a given target function $f:\R^d\to\R$, where $\bmcalL_1$ and $\calL_2$ are two affine linear maps matching the dimensions. 

Choose a large $S\in\N^+$ and set $\delta=T/S$.
It follows from ODE~\eqref{eq:dynamical:system:def} that
\begin{equation*}
    \bmz\big(\delta  (s+1)\big)=\bmz(\delta   s) + \int_{ \delta s}^{\delta(s+1)} \bmF\big(\bmz(t),t,\bmtheta\big)dt
\end{equation*}
for $s=0,1,\dotsc,S-1$.
We denote $\bmz_{s}$ as the numerical solution and use it to approximate the true solution $\bmz(\delta  s)$ for $s=0,1,\dotsc,S$.
By using the forward Euler method to discretize ODE~\eqref{eq:dynamical:system:def}, we have
\begin{equation*}
    \bmz_{s+1}
    =
\bmz_s+\delta\bmF\big(\bmz_s,\delta  s,\bmtheta)
\end{equation*}
for $s=0,1,\dotsc,S-1$. Such an iteration step can be regarded as a residual network \cite{7780459}. Thus, a dynamical system can be viewed as a continuous-time version of a residual network. The network generated via a dynamical system is generally called continuous-depth network. The function $\calL_2\circ \bmPhi(\cdot , \bmtheta)\circ \bmcalL_1$ mentioned above is indeed generated by a continuous-depth network.
As we know, $\bmz_{S}$ can approximate $\bmz(\delta  S)=\bmz(T)$ arbitrarily well for sufficiently large $S$ with some proper conditions on the dynamics function $\bmF$.

Suppose $\bmF(\bmy,t,\bmtheta)$ given in ODE~\eqref{eq:dynamical:system:def} is independent of $t$ for any  $(\bmy,\bmtheta)\in\R^n\times \bmTheta$.
Define $\bmg_\bmtheta:\R^n\to\R^n$ via 
\begin{equation*}
    \bmg_\bmtheta(\bmy)\coloneqq \bmy+ \delta \bmF(\bmy,0,\bmtheta).
\end{equation*}
 Then, we have
\begin{equation*}
    \begin{split}
            \bmz_{s+1}&=
\bmz_s+\delta\bmF\big(\bmz_s,\delta  s,\bmtheta)\\
&=\bmz_s+\delta\bmF\big(\bmz_s,0,\bmtheta)=\bmg_\bmtheta(\bmz_s)
    \end{split}
\end{equation*}
for $s=0,1,\dotsc,S-1$, implying $\bmz_S=\bmg_\bmtheta^{\circ S}(\bmz_0)$.
It follows that, for any $\bmy=\bmz_0\in\R^n$ and $\bmtheta\in\bmTheta$, we have
\begin{equation*}
    \begin{split}
        \bmg_\bmtheta^{\circ S}(\bmy)
        &=\bmg_\bmtheta^{\circ S}(\bmz_0)
        =\bmz_S\approx 
        \bmz(\delta  S) \\&=\bmz(T)=\bmPhi(\bmz_0,\bmtheta)=\bmPhi(\bmy,\bmtheta).
    \end{split}
\end{equation*}

Our results imply that $\calL_2\circ \bmg_\bmtheta^{\circ S}\circ \bmcalL_1$ has immense approximation power even if $\bmg_\bmtheta$ is realized by a fixed-size ReLU network, where $\bmcalL_1$  and $\calL_2$ are affine maps matching the dimensions.
Define $\bmF:\R^{n+1}\times \bmTheta\to\R^n$ via
    \begin{equation}
    \label{eq:dynamics:func:def}\bmF(\bmy,t,\bmtheta)\coloneqq\big(\bmg_\bmtheta(\bmy)-\bmy\big)/\delta,
\end{equation}
where $\bmg_\bmtheta$ is realized by a fixed-size ReLU network.
Then, the function $\calL_2\circ \bmPhi(\cdot,\bmtheta)\circ \bmcalL_1$, modelled by a continuous-depth network, can approximate $\calL_2\circ \bmg_\bmtheta^{\circ S}\circ \bmcalL_1$ well and hence also has immense approximation power. The definition of the dynamics function $\bmF$ in Equation~\eqref{eq:dynamics:func:def} implies that $\bmF(\bmy,t,\bmtheta)$ is independent of $t$ for any $(\bmy,\bmtheta)\in \R^n\times \bmTheta$ and $\bmF$ can also be realized by a fixed-size ReLU network. In short, we have shown that
a continuous-depth network can also have immense approximation power even if its dynamics function is time-independent and realized by a fixed-size ReLU network.
One may refer to Section~\ref{sec:dynamical:systems} for a further discussion on dynamical systems.

The remaining sections of this paper are structured as follows.
In Section~\ref{sec:related:work}, we discuss the connections between our results and existing work.
Section~\ref{sec:ideas} outlines the main ideas behind the proofs of Theorems~\ref{thm:main:Lp} and \ref{thm:main:Linfty}.
Next, in Section~\ref{sec:experiment}, we provide two simple experiments to numerically validate our theoretical results.
Finally,  Section~\ref{sec:conclusion} concludes this paper with a brief discussion.

\section{Related Work}
\label{sec:related:work}

%
In this section, we will provide a comprehensive overview of previous research that is pertinent to our results. We commence by emphasizing the correlation between deep learning and dynamical systems. Subsequently, we delve into the subject of parameter-sharing schemes in neural networks. Finally, we compare our results with existing research from the standpoint of function approximation.
  
%
%
\subsection{Deep Learning via Dynamical Systems}
\label{sec:dynamical:systems}

A dynamical system is a mathematical framework that describes the evolution of a system over time. Its origins can be traced back to Newtonian mechanics. For a comprehensive overview of the history of dynamical systems, one may refer to \cite{Holmes:2007}.
In general, a dynamical system consists of two fundamental components. First, we have the state variable(s), which represent the variables that fully describe the state of the system. These variables capture the relevant properties or quantities of interest in the system. The second component is the time evolution rule, which specifies how the future states of the system evolve from the current state. It provides the mathematical equations or rules that govern the dynamics of the system over time.
By studying the behavior and properties of dynamical systems, we gain insights into how systems change and develop over time. This framework has found applications in various fields, including physics, biology, economics, and computer science. In the context of deep learning, the connection to dynamical systems highlights the temporal aspect of learning and the potential for capturing complex dynamics in neural networks.

In recent years, there has been a growing body of research establishing connections between dynamical systems and deep learning. One such work  \cite{WeinanE2017dynamicalsystems} introduces a novel concept that interprets the discretization of a continuous dynamical system as a continuous-depth residual network. This approach utilizes continuous dynamical systems to model high-dimensional nonlinear functions commonly encountered in machine learning tasks.
Another notable contribution by the authors in \cite{NEURIPS2018_69386f6b} parameterizes the derivative of the hidden state using a neural network, introducing continuous-depth residual networks. This work highlights several advantages of continuous-depth models, including constant memory cost.
The study in \cite{2019arXiv191210382L} establishes general sufficient conditions for the universal approximation property of continuous-depth residual networks, further connecting the dynamical systems approach to deep learning. 
A similar result is demonstrated in \cite{2022arXiv220808707L}, where the focus shifts to specific invariant functions instead of generic continuous functions.
Additionally, the universal approximation property of deep fully convolutional networks is explored from the perspective of dynamical systems in \cite{2022arXiv221114047L}. The authors demonstrate that deep residual fully convolutional networks, along with their continuous-depth counterparts of constant channel width, can achieve the universal approximation of specific symmetric functions.
These studies serve as exemplary demonstrations of the endeavors made to establish connections between dynamical systems and deep learning. They explore the
potential benefits and theoretical foundations of continuous-depth models in various contexts.


%

\subsection{Parameter Sharing in Neural Networks}
\label{sec:parameter:sharing}



In recent years, deep neural network models have demonstrated 
notable
accomplishments across various domains. Nonetheless, the growing size of deep neural network models frequently introduces complexities in terms of computation and memory usage. To tackle these challenges, several techniques for model compression and acceleration have been developed, many of which involve the concept of parameter sharing.
Parameter-sharing schemes are utilized in neural networks to minimize the total number of parameters, resulting in reduced memory and computational requirements. Our network architecture, which involves the repeated compositions of a single fixed-size network, can be viewed as a particular instance of a parameter-sharing scheme in neural networks.

To the best of our knowledge, parameter-sharing schemes in neural networks can be broadly categorized into three basic cases. The first case involves sharing parameters within the same layer, as seen in convolutional neural networks (CNNs), where kernels (filters) are shared across all image positions. The second case entails sharing parameters among different layers of neural networks, as in recurrent neural networks (RNNs). Our network architecture follows this second scheme by sharing parameters through repeated compositions of a single fixed-size network. We demonstrate that this approach can yield immense approximation capabilities by repeating a fixed number of parameters.
In addition to these two parameter-sharing schemes, there is also the practice of sharing parameters across different neural networks or models, which is often employed in multi-task learning scenarios. For further insights into parameter sharing in neural networks, interested readers can refer to the references \cite{savarese2018learning,2020arXiv200902386W,2006.10598,NEURIPS2020_42cd63cb,9859706,9879069}.

\subsection{Discussion from an Approximation Perspective}

The approximation power of neural networks has been extensively studied, with numerous publications focusing on constructing various neural networks to approximate a wide range of target functions. Some notable examples include \cite{Cybenko1989ApproximationBS,HORNIK1989359,barron1993,yarotsky18a,yarotsky2017,doi:10.1137/18M118709X,ZHOU2019,10.3389/fams.2018.00014,2019arXiv190501208G,2019arXiv190207896G,suzuki2018adaptivity,Ryumei,Wenjing,Bao2019ApproximationAO,2019arXiv191210382L,MO,shijun:1,shijun:2,shijun:3,shijun:thesis,shijun:intrinsic:parameters,shijun:arbitrary:error:with:fixed:size}.
In the early stages of this field, the focus was on exploring the universal approximation power of one-hidden-layer networks. The universal approximation theorem \cite{Cybenko1989ApproximationBS,HORNIK1991251,HORNIK1989359} demonstrated that a sufficiently large neural network can approximate a certain type of target function arbitrarily well, without explicitly estimating the approximation error in terms of the network size.
Subsequent work, such as \cite{barron1993,barron2018approximation}, analyzed the approximation error of one-hidden-layer networks of width $n$ and showed an asymptotic approximation error of $\calO(n^{-1/2})$ in the $L^2$-norm 
for target functions with certain smoothness properties.


Recent research has placed significant emphasis on the approximation of deep neural networks. 
Notably, the findings presented in \cite{shijun:2, yarotsky18a, shijun:thesis} indicate that ReLU networks with $n$ parameters can achieve an optimal approximation error of $\calO(n^{-2/d})$ when approximating $1$-Lipschitz continuous functions on $[0,1]^d$.
However, it is crucial to recognize that this optimal approximation rate suffers from the curse of dimensionality. 
To overcome the limitations imposed by the curse of dimensionality and achieve better approximation errors, various approaches have been proposed and explored. These approaches aim to enhance the quality of approximation or even directly address the challenges arising from the curse of dimensionality.
One approach is to consider smaller function spaces, such as smooth functions \cite{shijun:3,yarotsky:2019:06}, band-limited functions \cite{bandlimit}, and Barron spaces \cite{barron2018approximation,barron1993,2019arXiv190608039E}. By restricting the class of functions being approximated, it is possible to achieve better approximation errors with neural networks.
Another approach is to design new network architectures that can improve the approximation capabilities. Examples of such architectures include Floor-ReLU networks \cite{shijun:4}, Floor-Exponential-Step networks \cite{shijun:5}, (Sin, ReLU, $2^x$)-activated networks \cite{jiao2021deep}, and three-dimensional networks \cite{shijun:net:arc:beyond:width:depth}.
By exploring different function spaces and designing novel network architectures, researchers have been able to push the limits of approximation accuracy in neural networks, providing more flexibility and better performance for various tasks.
It is important to note that the literature on the approximation analysis of deep neural networks is vast, and the publications mentioned here represent only a subset of the existing research. 
Many other approaches, techniques, and architectures have been proposed to address the challenge of improving the approximation error for specific function classes.

In this paper, we propose a specific neural network architecture generated by repeated compositions of a single fixed-size network. Theorems~\ref{thm:main:Lp} and \ref{thm:main:Linfty} demonstrate that repeating a small ReLU network block can enhance the approximation power of our network architecture. We will conduct experiments in Section~\ref{sec:experiment} to numerically verify our theoretical results and evaluate the approximation capabilities of our network architecture.
\section{Ideas for Proving Theorems~\ref{thm:main:Lp} and \ref{thm:main:Linfty}}
\label{sec:ideas}

Let us outline the main ideas behind the proofs of Theorems~\ref{thm:main:Lp} and \ref{thm:main:Linfty}.
During the proofs, our main approach involves constructing a piecewise constant function to approximate the desired continuous function.
However, the continuity of ReLU networks poses a challenge in uniformly approximating piecewise constant functions.
To bridge this gap, we first  design ReLU networks
to realize piecewise constant functions outside a sufficiently small region to approximate the target function well. Then, we will introduce a theorem to deal specifically with the approximation inside this small region for achieving uniform approximation.

Based on the aforementioned ideas, let us delve into the specific details.
We divide $[0,1]^d$  into a union of ``important'' cubes $\{Q_\bmbeta\}_{\bmbeta\in \{0,1,\dotsc,K-1\}^d}$ and a small region $\Omega$, where $K$ is a proper integer determined later. Each $Q_\bmbeta$ is associated with a representative $\bm{x}_\bmbeta\in Q_\bmbeta$ for each $\bmbeta\in \{0,1,\dotsc,K-1\}^d$. 
See Figure~\ref{fig:idea:main} for an illustration of $\bmx_\bmbeta$, $\Omega$, and $Q_\bmbeta$.
Then, the construction of the desired network approximating the target function can be divided into three steps as follows.
\begin{enumerate}
	\item First, we design a sub-network to realize a vector-valued function $\bmPhi_1$ mapping the whole cube $Q_\bmbeta$ to its index $\bmbeta$ for each $\bmbeta$. That is, $\bmPhi_1(\bmx)=\bmbeta$ for any $\bmx\in  Q_\bmbeta$ and $\bmbeta\in \{0,1,\dotsc,K-1\}^d$.
 
	\item Next, we design a sub-network to realize a function $\phi_2$ mapping $\bmbeta$ approximately to $f(\bmx_\bmbeta)$ for each $\bmbeta$. That is, $\phi_2(\bmbeta)\approx f(\bmx_\bmbeta)$ for any $\bmbeta\in \{0,1,\dotsc,K-1\}^d$.
 
	\item Finally, by defining $\phi\coloneqq \phi_2\circ \bmPhi_1$, we have $\phi(\bmx)=\phi_2\circ\bmPhi_1(\bmx)=\phi_2(\bmbeta)\approx f(\bmx_\bmbeta)\approx f(\bmx)$ for any $\bmx\in Q_\bmbeta$ and each $\bmbeta\in\{0,1,\dotsc,K-1\}^d$.
 Additionally, we must also
 address the approximation occurring within $\Omega$ and demonstrate that 
 $\phi=\phi_2\circ \bmPhi_1$ can be represented in the desired form $\phi=\calL_2\circ \bmg^{\circ r}\circ \bmcalL_1$, where $\bmcalL_1$ and $\calL_2$ are affine linear maps and $\bmg$ is realized by a fixed-size ReLU network.
\end{enumerate}

\begin{figure}[htbp!] 
	\centering
 \vskip 3.1pt
	\includegraphics[width=0.995\linewidth]{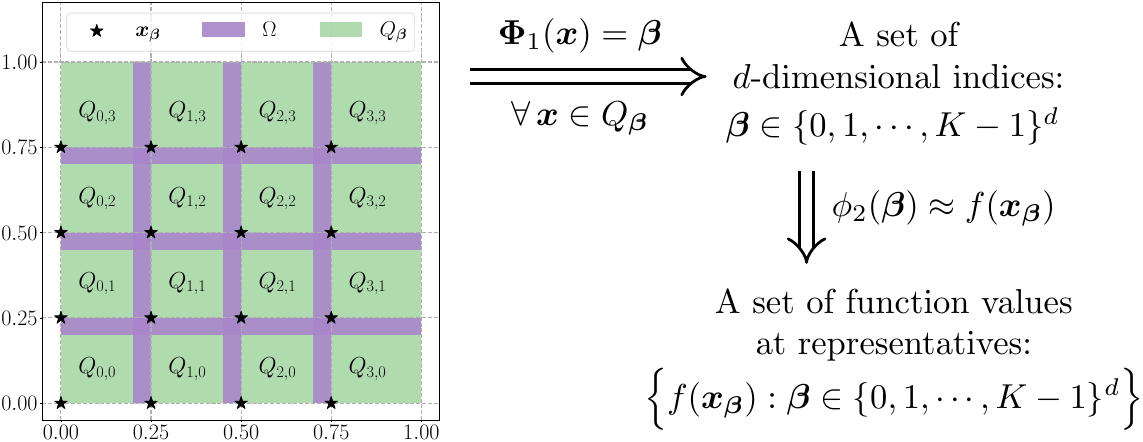}
 \vskip -0.1pt
	\caption{An illustration of the ideas for constructing the desired function $\phi=\phi_2\circ\bmPhi_1$. 
		Note that $\phi\approx f$ outside $\Omega$
		since $\phi(\bmx)=\phi_2\circ\bmPhi_1(\bmx)=\phi_2(\bmbeta)\approx f(\bmx_\bmbeta)\approx f(\bmx)$ for any $\bmx\in Q_\bmbeta$ and each $\bmbeta\in\{0,1,\dotsc,K-1\}^d$.}
	\label{fig:idea:main}
\end{figure}

See Figure~\ref{fig:idea:main} for an illustration of these three steps. More details on these three steps can be found below.

\vspace{4pt}
\textbf{Step 1}: Constructing $\bmPhi_1$.

As mentioned previously, the aim of $\bmPhi_1$ is to map $\bmx\in Q_\bmbeta$ to $\bmbeta$ for each $\bmbeta\in \{0,1,\dotsc,K-1\}^d$. Note that 
$\bmPhi_1$ can be defined/constructed via \[\bmPhi_1(\bmx)=\big(\phi_1(x_1),\,\dotsc,\,\phi_1(x_d)\big)\]
for any $\bmx=(x_1,\dotsc,x_d)\in\R^d$, where $\phi_1:\R\to\R$  is a  step function outside a small region and hence can be  realized by a ReLU network. 
It is generally challenging to design a ReLU network with a limited budget and the required architecture to realize such a  function  $\phi_1$. Thus, we establish a proposition, Proposition~\ref{prop:floor:approx} below, to do this step and place its proof in Section~\ref{sec:proof:prop:floor:approx} of the appendix.

\begin{proposition}
	\label{prop:floor:approx}
	Given any $\delta\in (0,1)$ and $n,m\in \N^+$ with $n\le m$,
	there exist $\bmg \in
	\nn[\big]{9}{1}{5}{5}$
	and two affine linear maps $\bmcalL_1:\R\to\R^5$ and $\calL_2:\R^5\to \R$ such that
	\begin{equation*}
		\calL_2\circ\bmg^{\circ (m-1)}\circ \bmcalL_1(x)=k
	\end{equation*}
if $x\in \big[k,\,k+1-\delta\cdot\one_{\{k\le n-2\}}\big]$ for $k=0,1,\dotsc,n-1$.
\end{proposition}

\vspace{4pt}
\textbf{Step 2}: Constructing $\phi_2$.


The objective of $\phi_2$ is to map $\bmbeta$ approximately to $f(\bmx_\bmbeta)$ for each $\bmbeta\in\{0,1,\dotsc,K-1\}^d$.
It is important to note that, during
in the construction of $\phi_2$, we only need to care about the values of $\phi_2$ sampled inside the set $\{0,1,\dotsc,K-1\}^d$, which is a key point to ease the design of a ReLU network realizing $\phi_2$. 
Indeed, if we can define a proper affine linear map $\calL:\R^d\to\R$, then we only need to construct $\tildephi_2:\R\to\R$ to  map $\calL(\bmbeta)$  approximatly to $f(\bmx_\bmbeta)$ since $\phi_2=\tildephi_2\circ \calL$ can map $\bmbeta$ approximately to $f(\bmx_\bmbeta)$.
It is still challenging to construct a ReLU network with a limited budget and the required architecture to realize $\tildephi_2$.
Thus, we establish Proposition~\ref{prop:point:fitting} below to simplify the construction of  $\tildephi_2$. The proof of Proposition~\ref{prop:point:fitting} is complicated and hence is placed in Section~\ref{sec:proof:prop:point:fitting} of the appendix.


\begin{proposition}
	\label{prop:point:fitting}
	Given any $\varepsilon>0$,\hspace{5pt}  $n,m\in \N^+$ with $n\le m$,
	and  
	${y}_k\ge 0$ for $k=0,1,\dotsc,n-1$  with 
	\[|y_{k}-y_{k-1}|\le \varepsilon\quad \tn{for $k=1,2,\dotsc,n-1,$}\] 
	there exist
	$\bmg\in\nn{16}{2}{6}{6}$ and two affine linear maps $\bmcalL_1:\R\to\R^6$ and $\calL_2:\R^6\to\R$
	such that
	\begin{equation*}
		\big| \calL_2\circ \bmg^{\circ (m-1)}\circ \bmcalL_1(k)-{y}_k\big|\le \eps
		\quad \tn{for $k=0,1,\dotsc,n-1$.}
	\end{equation*}
\end{proposition}

\vspace{4pt}
\textbf{Step 3}: Representing $\phi=\phi_2\circ\bmPhi_1$ properly.

With $\bmPhi_1$ and $\phi_2$ constructed in the first two steps, we can define $\phi\coloneqq \phi_2\circ \bmPhi_1$ and we have
\begin{equation*}
	\phi(\bmx)=\phi_2\circ\bmPhi_1(\bmx)=\phi_2(\bmbeta)\approx f(\bmx_\bmbeta)\approx f(\bmx)
\end{equation*}
for any $\bmx\in Q_\bmbeta$ and each $\bmbeta\in \{0,1,\dotsc,K-1\}^d$.
That means $\phi$ can approximate $f$ well outside $\Omega$. 
By making $\phi$ bounded and $\Omega$ sufficiently small, we can easily control the $L^p$-norm approximation error to prove Theorem~\ref{thm:main:Lp} for any $p\in [1,\infty)$.
To prove Theorem~\ref{thm:main:Linfty}, we require $\phi$ to pointwise approximate $f$ well. To this end, we use the idea of Lemma~$3.11$ in \cite{shijun:thesis} (or Lemma~$3.4$ in \cite{shijun:3}) to control the approximation error inside a small region. 

Apart from a good approximation error, we also need to show 
that $\phi$ can be represented as the desired form 
$\calL_2\circ \bmg^{\circ r}\circ \bmcalL_1$, where $\bmcalL_1$ and $\calL_2$ are two affine linear maps and  $\bmg$ is realized by a fixed-size ReLU network. Note that $\bmPhi_1$ and $\phi_2$ are constructed based on Propositions~\ref{prop:floor:approx} and \ref{prop:point:fitting}, respectively. Thus, 
both $\bmPhi_1$ and $\phi_2$ are expected to have the following form:
\begin{equation*}
	 \bmcalL_2\circ  \bmg^{\circ r}\circ \bmcalL_1,
\end{equation*}
where $\bmcalL_1,\bmcalL_2$ are affine linear maps and $\bmg$ is realized by fixed-size ReLU networks. Then,
$\phi=\phi_2\circ \bmPhi_1$ are expected to have the following form:
\begin{equation}
	\label{eq:phi:full:form}
	\phi= \tildecalL_3\circ \bmg_2^{\circ r_2}\circ \bmtildecalL_2\circ  \bmg_1^{\circ r_1}\circ \bmtildecalL_1,
\end{equation}
where $\bmtildecalL_1,\bmtildecalL_2,\tildecalL_3$ are affine linear maps and $\bmg_1,\bmg_2$ are realized by fixed-size ReLU networks. It is not trivial to convert  the form in Equation~\eqref{eq:phi:full:form} to the desired form.
A proposition is established to facilitate such a conversion. 

\begin{proposition}
	\label{prop:two:blocks:three:affine}
	Let $\bmtildecalL_1:\R^{d_0}\to \R^{d_1}$, $\bmtildecalL_2:\R^{d_1}\to \R^{d_2}$, and $\bmtildecalL_3:\R^{d_2}\to\R^{d_3}$ be three affine linear maps.  Suppose 
 \begin{equation*}
     \bmg_i\in \nn{N_i}{L_i}{d_i}{d_i}
 \end{equation*}
 and $r_i\in \N^+$ for $i=1,2$.
	For any $A>0$ and $d\in\N^+$ with $d\ge\max\{d_1,d_2\}$, there exist
$
	    \bmg	\in		\nn{N_1+N_2+6d+2}{\max\{L_1+2,L_2+1\}}{d+2}{d+2}
	$	and two affine linear maps $\bmcalL_1:\R^{d_0}\to \R^{d+2}$ and $\bmcalL_2: \R^{d+2} \to \R^{d_3}$
	such that
	\begin{equation*}
		\bmtildecalL_3\circ \bmg_2^{\circ r_2}\circ \bmtildecalL_2\circ  \bmg_1^{\circ r_1}\circ \bmtildecalL_1 (\bmx)
		=\bmcalL_2\circ \bmg^{\circ (r_1+r_2+1)}\circ \bmcalL_1(\bmx)
	\end{equation*} 
	for any $\bmx \in [-A,A]^{d_0}$.
\end{proposition}

The proof of Proposition~\ref{prop:two:blocks:three:affine} is technical and hence is deferred to Section~\ref{sec:proof:prop:two:blocks:three:affine} of the appendix.

\section{Numerical Experiments}
\label{sec:experiment}

The primary objective of this section is to numerically validate the theoretical results presented in Theorems~\ref{thm:main:Lp} and \ref{thm:main:Linfty}. To achieve this goal, we have designed two distinct experiments.
The first experiment, described in Section~\ref{sec:experiment:func:approx}, focuses on a function approximation task. Our aim is to demonstrate that increasing $r$ in our RCNet architecture, denoted as $\calL_2\circ \bmg^{\circ r}\circ \bmcalL_1$, improves the error of function approximation. In this architecture, $\bmcalL_1$ and $\calL_2$ represent affine linear maps, while $\bmg$ represents a ReLU network block.
The second experiment, outlined in Section~\ref{sec:experiment:classification}, focuses on a classification task. We intend to illustrate that increasing the value of $r$ in our RCNet architecture 
$\bmcalL_2\circ \bmg^{\circ r}\circ \bmcalL_1$
leads to enhanced classification performance. By evaluating the accuracy of the classification results, we can empirically validate the advantages of incorporating multiple compositions of the fixed-size ReLU network block.
Through these experiments, our goal is to provide numerical evidence that supports the theoretical claims and showcases the potential of our RCNet architecture in terms of improving approximation power. The results obtained from these experiments will contribute to a comprehensive understanding of the practical implications and advantages of our network design.

Next, let us briefly discuss the expected impact of increasing $r$ in our RCNet architecture $\calL_2\circ \bmg^{\circ r}\circ \bmcalL_1$ on the experiment results. In this discussion, we will focus on the ReLU activation function and consider the three main sources of errors in the test results: approximation error, generalization error, and optimization error.
In our experiments where $r$ is small, we assume that the optimization error is well-controlled due to the utilization of a good optimization algorithm. Thus, we can primarily focus on the effects of increasing $r$ on the approximation and generalization errors.
Increasing $r$ has the potential to reduce the approximation error, as demonstrated by our theoretical results. However, it may also lead to an increase in the generalization error. Therefore, the performance improvement associated with larger values of $r$ depends on whether the approximation error or the generalization error dominates.
If the approximation error is the leading term, then increasing $r$ can be beneficial. To emphasize the approximation error, we design our first experiment to involve a sufficiently complex target function. By choosing a challenging binary classification problem for our second experiment, we create conditions where the approximation error is relatively large. In both experiments, a sufficient number of samples are generated to control the generalization error, ensuring that it does not overshadow the effects of the approximation error.
By carefully setting up these experiments and controlling the different sources of errors, we can gain insights into the impact of increasing $r$ in our RCNet architecture. 

\subsection{Function Approximation}
\label{sec:experiment:func:approx}

To evaluate and compare the approximation capabilities of our RCNet architecture $\calL_2\circ\bmg_n^{\circ r}\circ\bmcalL_1$ across various values of $r$ and $n$, we will utilize it for a function approximation task. This architecture comprises two affine linear maps, $\bmcalL_1$ and $\calL_2$, along with a ReLU network block $\bmg_n$.
To facilitate a comprehensive comparison, we have specifically chosen a target function $f$ that exhibits a high degree of complexity.
The function $f:[0,1]^2\to\R$ is defined as 
\begin{equation*}
	f(\bmx)= \sum_{i=1}^2\sum_{j=1}^2 a_{i,j} \sin(b_i x_i+c_{i,j} x_i x_j) \cos (b_j x_j + d_{i,j} x_i^2)
\end{equation*}
for any $\bmx=(x_1,x_2)\in[0,1]^2$, where  the coefficient matrices are given by
\begin{equation*}
	(a_{i,j})=	\begin{bmatrix*}
		0.3 & 0.2 \\ 0.2 & 0.3
	\end{bmatrix*}, \quad 
	(b_{i})=	\begin{bmatrix*}
		2\pi \\ 4\pi
	\end{bmatrix*},
 \end{equation*}
\begin{equation*}
	(c_{i,j})=	\begin{bmatrix*}
		2\pi & 4\pi \\ 8\pi & 4\pi
	\end{bmatrix*}, \quad \tn{and} \quad 
	(d_{i,j})=	\begin{bmatrix*}
		4\pi & 6\pi \\ 8\pi & 6\pi
	\end{bmatrix*}.
\end{equation*}

To visually represent the target function $f$, we have included illustrations in Figure~\ref{fig:func:approx:f}.
By choosing this intricate function as our target, we can effectively assess the approximation capability of our network architecture across various values of $r$ and $n$.

\begin{figure}[htbp!]
    %
    %
    \begin{subfigure}[c]{0.47\linewidth}
    \centering            \includegraphics[height=0.9\textwidth]{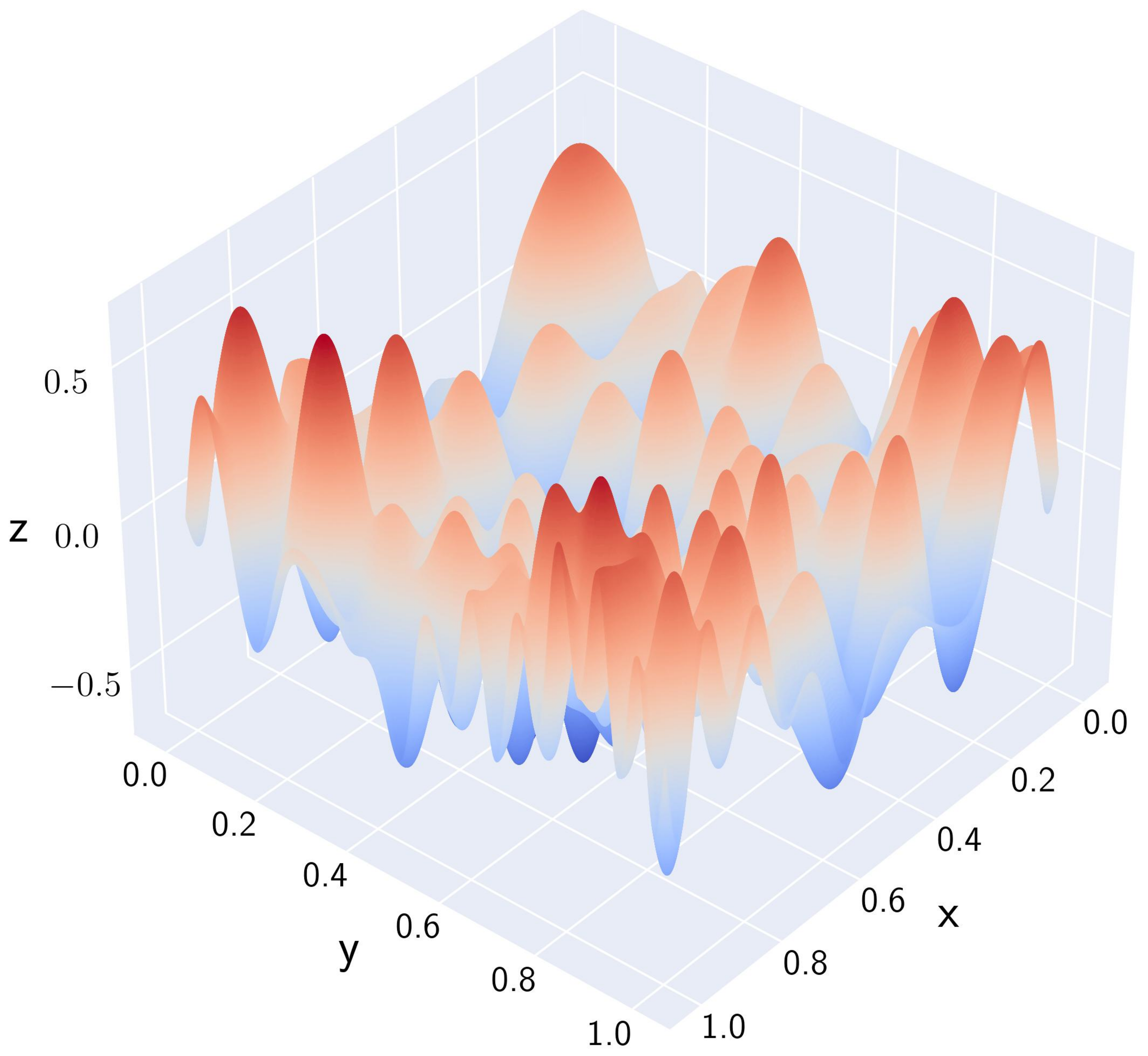}
    \end{subfigure}
    \hfill
    \begin{subfigure}[c]{0.47\linewidth}
    \centering            \includegraphics[height=0.84\textwidth]{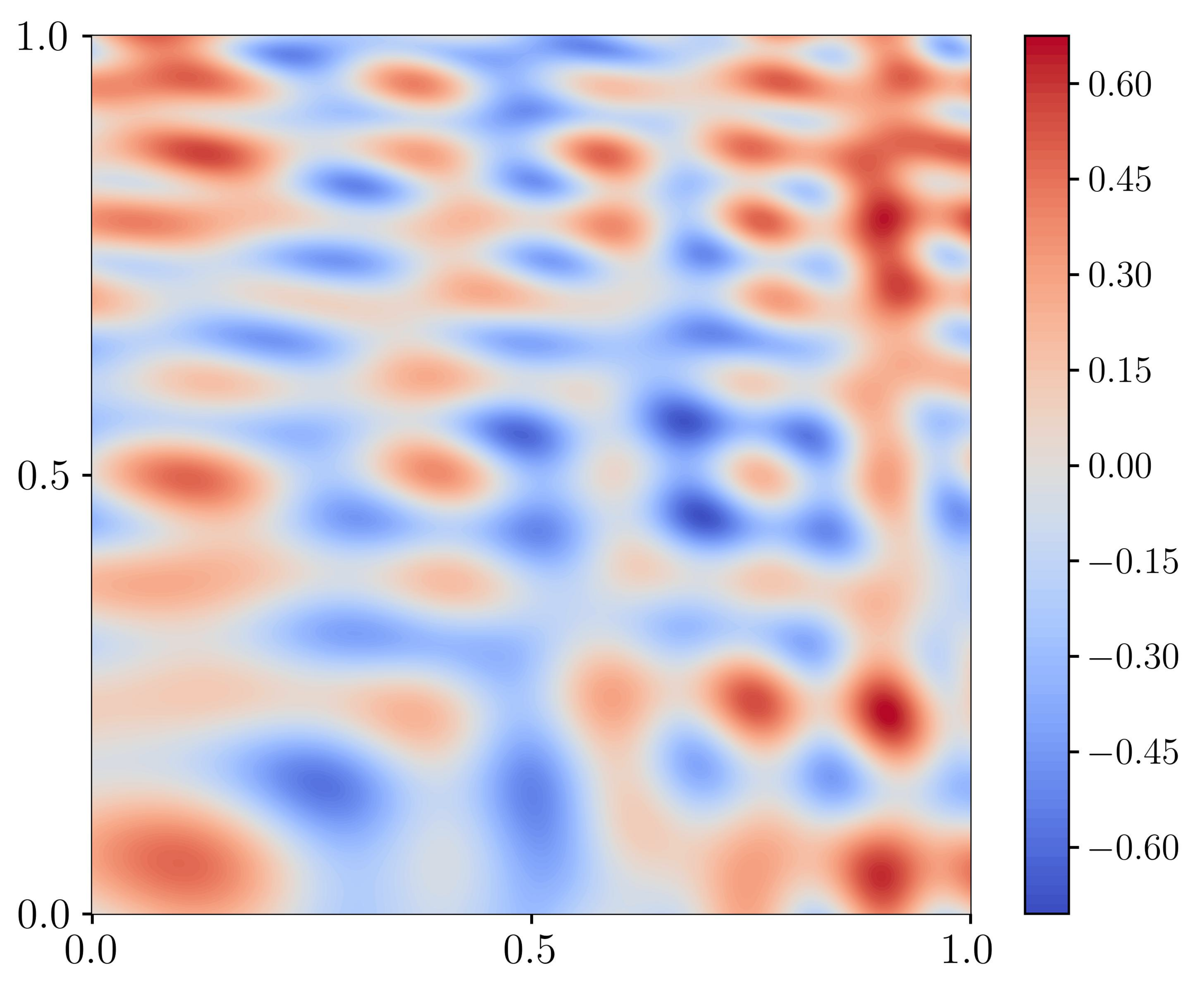}
    \end{subfigure}
    \caption{Illustrations of the target function $f$.}
    \label{fig:func:approx:f}
\end{figure}

In this experiment, we will employ the RCNet architecture $\calL_2\circ\bmg_n^{\circ r}\circ\bmcalL_1$ to approximate the target function $f$ for different values of $r$ and $n$.
Specifically, we consider $r=1,2,3,4$ and $n=100,200$.
The ReLU network block $\bmg_n$ is constructed by combining an affine linear map and the ReLU activation function, i.e., $\bmg_n$ is defined as
\begin{equation*}   
    \bmg_n(\bmx)\coloneqq\sigma(\bmA \bmx+\bmb)
\end{equation*}
for any $\bmx\in\R^{n}$, where $\bmA\in\R^{n\times n}$ and $\bmb\in\R^{n}$ are parameters and $\sigma$ is the ReLU activation function that can be applied element-wise to a vector.
Then, the dimensions of the input and output for the two affine linear maps, $\bmcalL_1:\R^2\to\R^{n}$ and $\calL_2:\R^{n}\to\R$, are determined accordingly.
Notably, the ReLU network block $\bmg_n$ consists of $n^2+n$ parameters. The affine linear map $\bmcalL_1$ contains $2n+n=3n$ parameters, while $\calL_2$ has $n +1$ parameters. Consequently, the total number of parameters in $\calL_2\circ \bmg_n^{\circ r}\circ \bmcalL_1$ is $(n^2+n)+3n+(n+1)=n^2+5n+1$ 
for different values of $r$ and $n$.
It is essential to note that when $r\geq 2$, the parameters of $\calL_2\circ \bmg_n^{\circ r}\circ \bmcalL_1$ are partially shared through repetitions of the ReLU network block $\bmg_n$. Our objective is to provide numerical evidence demonstrating that increasing the value of $r$ results in improved test losses for each fixed $n$.

Before presenting the numerical results, let us delve into the hyperparameters utilized for training our network architecture $\calL_2\circ \bmg_n^{\circ r}\circ \bmcalL_1$ for varying values of $r$ and $n$, specifically $r=1,2,3,4$ and $n=100,200$.
To generate the training and test samples, we employ the uniform distribution, resulting in $10^6$ training samples and $10^5$ test samples in $[0,1]^2$.
During the training process, we utilize the RAdam optimization method \cite{Liu2020On}, which aids in optimizing the network parameters.
We set the mini-batch size for training to $500$, which signifies the number of training samples processed in each iteration. Our training process comprises a total of $500$ epochs, representing complete passes through the training dataset.
The learning rate is adjusted every $5$ epochs. More specifically, the learning rate during epochs $5(i-1)+1$ to $5i$ is set to $0.002\times 0.9^{i-1}$ for $i=1,2,\dotsc,100$. This adaptive adjustment allows for fine-tuning the model during training.
To train our model, we employ the mean squared error (MSE) loss function, which measures the average squared difference between the network-generated function and the target function.
To ensure the reliability of our experiment, we repeat it $12$ times. From these repetitions, we discard $3$ top-performing and $3$ bottom-performing trials based on the average test accuracy of the last $100$ epochs. The target accuracy is then determined by taking the average of the test accuracies from the remaining $6$ trials for each epoch.

We are now prepared to present the experiment results that compare the numerical performances of $\calL_2\circ \bmg_n^{\circ r}\circ \bmcalL_1$ for different values of $r$ and $n$, specifically $r=1,2,3,4$ and $n=100,200$. The test losses over the last $100$ epochs are averaged to obtain the target losses, considering two types of loss functions: mean squared error (MSE) and maximum (MAX) loss functions.
Table~\ref{tab:loss:comparison} provides a comprehensive comparison of the numerical results obtained from $\calL_2\circ \bmg_n^{\circ r}\circ \bmcalL_1$ for the given values of $r$ and $n$. Additionally, Figure~\ref{fig:training:test:loss} illustrates the test losses measured in MSE on a logarithmic scale, allowing for an intuitive comparison.
The values presented in Table~\ref{tab:loss:comparison} and the trends observed in Figure~\ref{fig:training:test:loss} clearly indicate a significant improvement in test losses with increasing values of $r$.
These findings align with the theoretical results stated in Theorems~\ref{thm:main:Lp} and \ref{thm:main:Linfty}, providing further confirmation of the effectiveness of increasing $r$.

Lastly, it is crucial to acknowledge that
further increasing the value of $r$ may not lead to additional improvements in the results due to the inherent challenges involved in optimizing deep learning models. Issues such as local minima, saddle points, and vanishing gradients make it increasingly difficult to identify the global minimizer, particularly for larger values of $r$.





\begin{table}[htbp!]
	\centering  
 \caption{Test loss comparison.}
	\label{tab:loss:comparison}
	\resizebox{0.99\linewidth}{!}{ 
		\begin{tabular}{ccccccccc} 
			\toprule
			  \multirow{2}{*}{
     \raisebox{-3pt}{$\calL_2\circ\bmg_n^{\circ  r}\circ \bmcalL_1$}
     }  
     &     \multicolumn{2}{c}{$n=100$}
      & \multicolumn{2}{c}{$n=200$}\\
			 \cmidrule(lr){2-3}    
                \cmidrule(lr){4-5}
			 & MSE  & MAX & MSE  & MAX
    \\
			\midrule
			 \rowcolor{mygray}
			 $r=1$ &  
$1.22\times10^{-2}$ & $4.21\times10^{-1}$ & 
$8.92\times10^{-3}$ & $4.05\times10^{-1}$ 
\\	
			\midrule
			 $r=2$ & 
$2.02\times10^{-4}$ & $1.17\times10^{-1}$ & 
$4.60\times10^{-5}$ & $7.28\times10^{-2}$
\\		
    			\midrule
			 \rowcolor{mygray}
			 $r=3$  & 
$3.46\times10^{-5}$ & $4.65\times10^{-2}$ & 
$3.51\times10^{-6}$ & $1.66\times10^{-2}$\\	
			\midrule
			 $r=4$  & 
$1.27\times10^{-5}$ & $2.77\times10^{-2}$ & 
$1.11\times10^{-6}$ & $8.61\times10^{-3}$ \\	
			\bottomrule
		\end{tabular} 
	}
\end{table} 

\begin{figure}[htbp!]
\centering
    \begin{subfigure}[c]{0.49\linewidth}
    \centering            \includegraphics[width=0.998\textwidth]{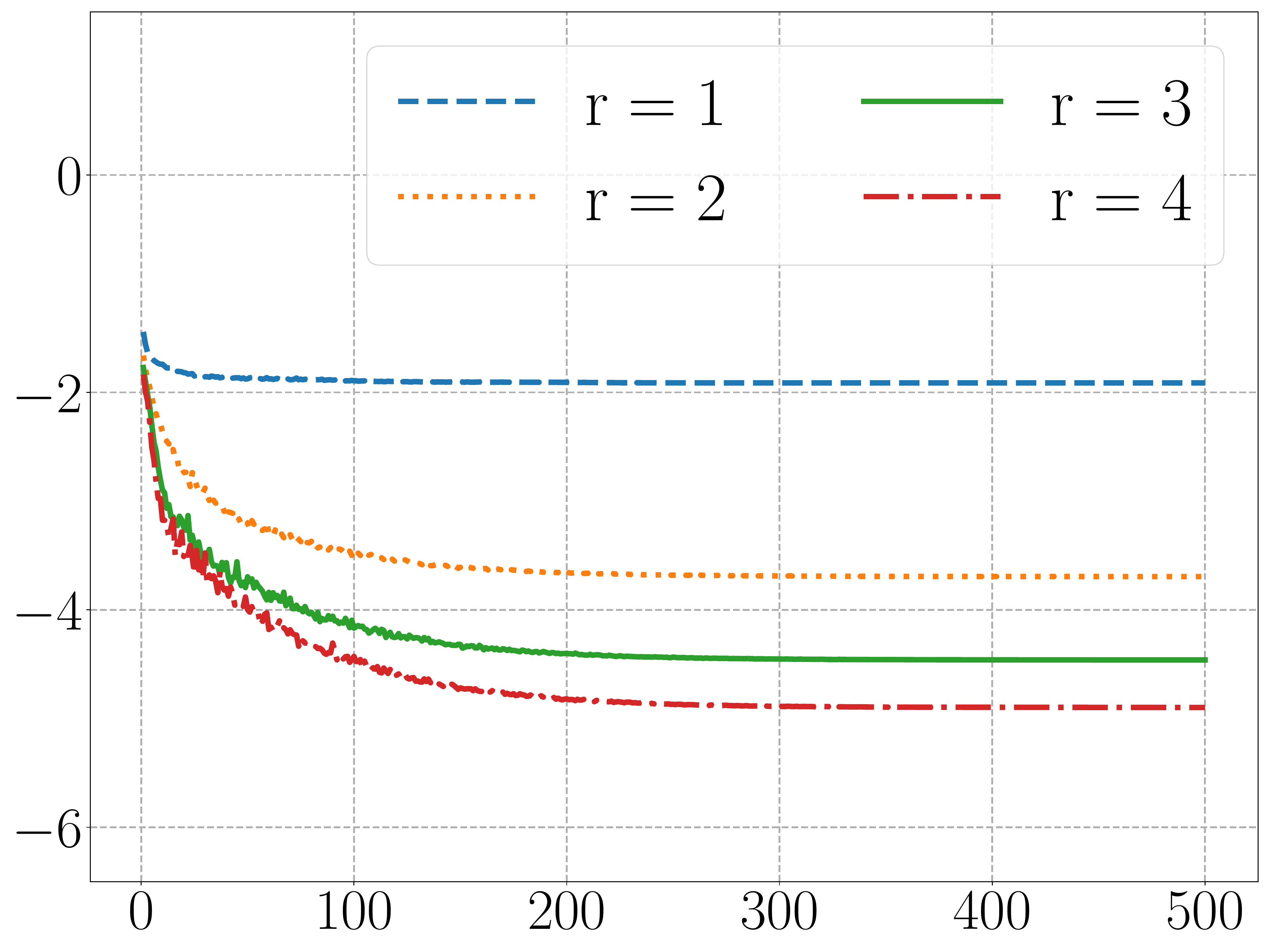}
    \subcaption{$n=100$.}
    \end{subfigure}
    \hfill
    \begin{subfigure}[c]{0.49\linewidth}
    \centering            \includegraphics[width=0.998\textwidth]{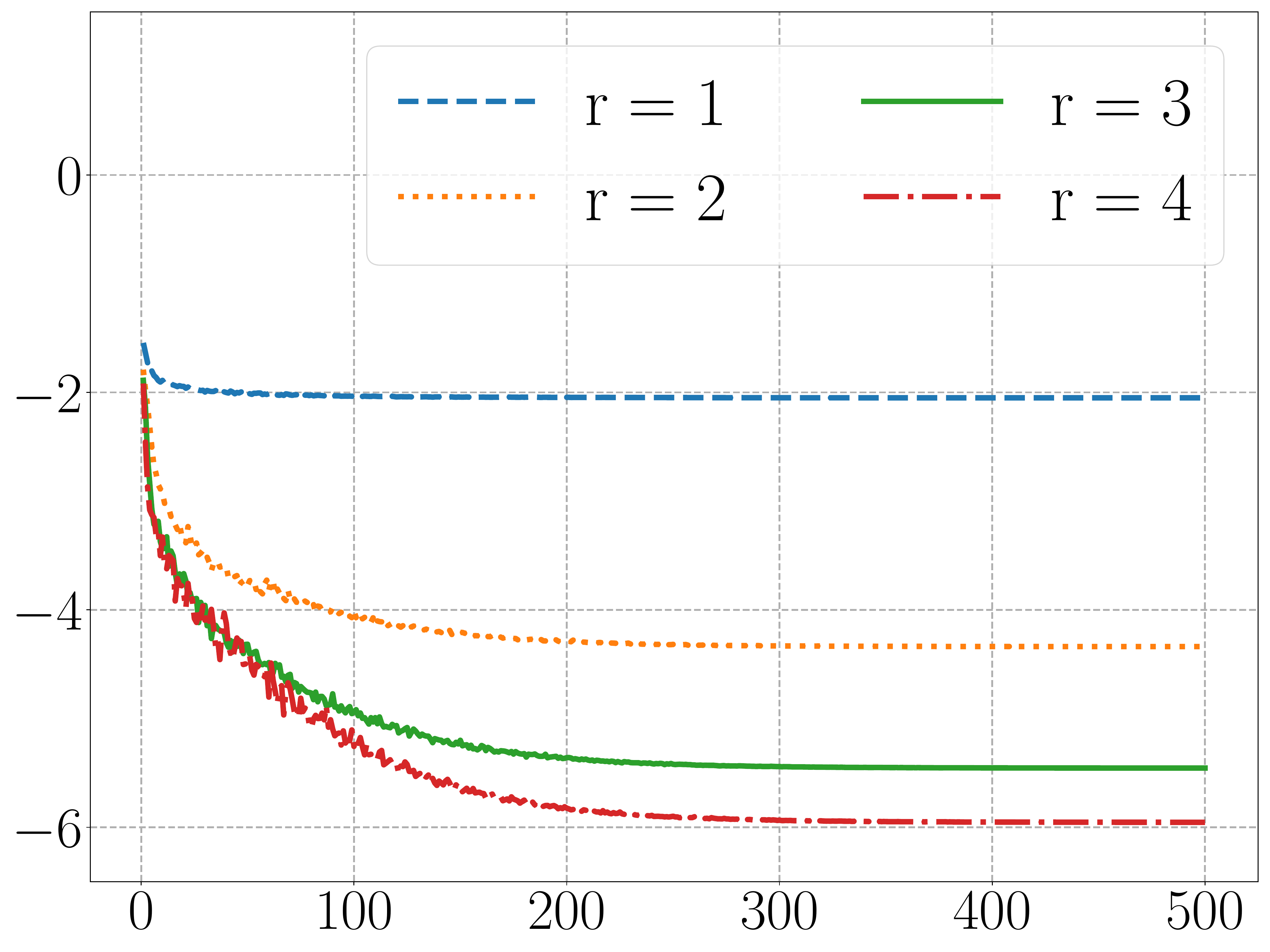}
    \subcaption{$n=200$.}
    \end{subfigure}
\caption{
Test losses measured in MSE across epochs: the $x$-axis represents the epoch number, while the $y$-axis corresponds to the base-10 logarithm of the test loss.
}
	\label{fig:training:test:loss}
\end{figure}



\subsection{Classification}
\label{sec:experiment:classification}



To assess and compare the approximation capabilities of our RCNet architecture $\bmcalL_2\circ\bmg_n^{\circ r}\circ\bmcalL_1$ across different values of $r$ and $n$, we will employ it for a classification task. This architecture consists of two affine linear maps, $\bmcalL_1$ and $\bmcalL_2$, along with a ReLU network block $\bmg_n$.
For a comprehensive comparison, we have selected a complex binary classification experiment utilizing the Archimedean spiral, as proposed in \cite{shijun:net:arc:beyond:width:depth}. The objective of this classification problem is to accurately classify samples from two distinct sets, denoted as $\calS_0$ and $\calS_1$. These two sets are constructed based on the Archimedean spiral, as illustrated in Figure~\ref{fig:spiral}.

\begin{figure}[htbp!]
    \centering
    \begin{subfigure}[c]{0.487\linewidth}
    \centering            \includegraphics[width=0.901\textwidth]{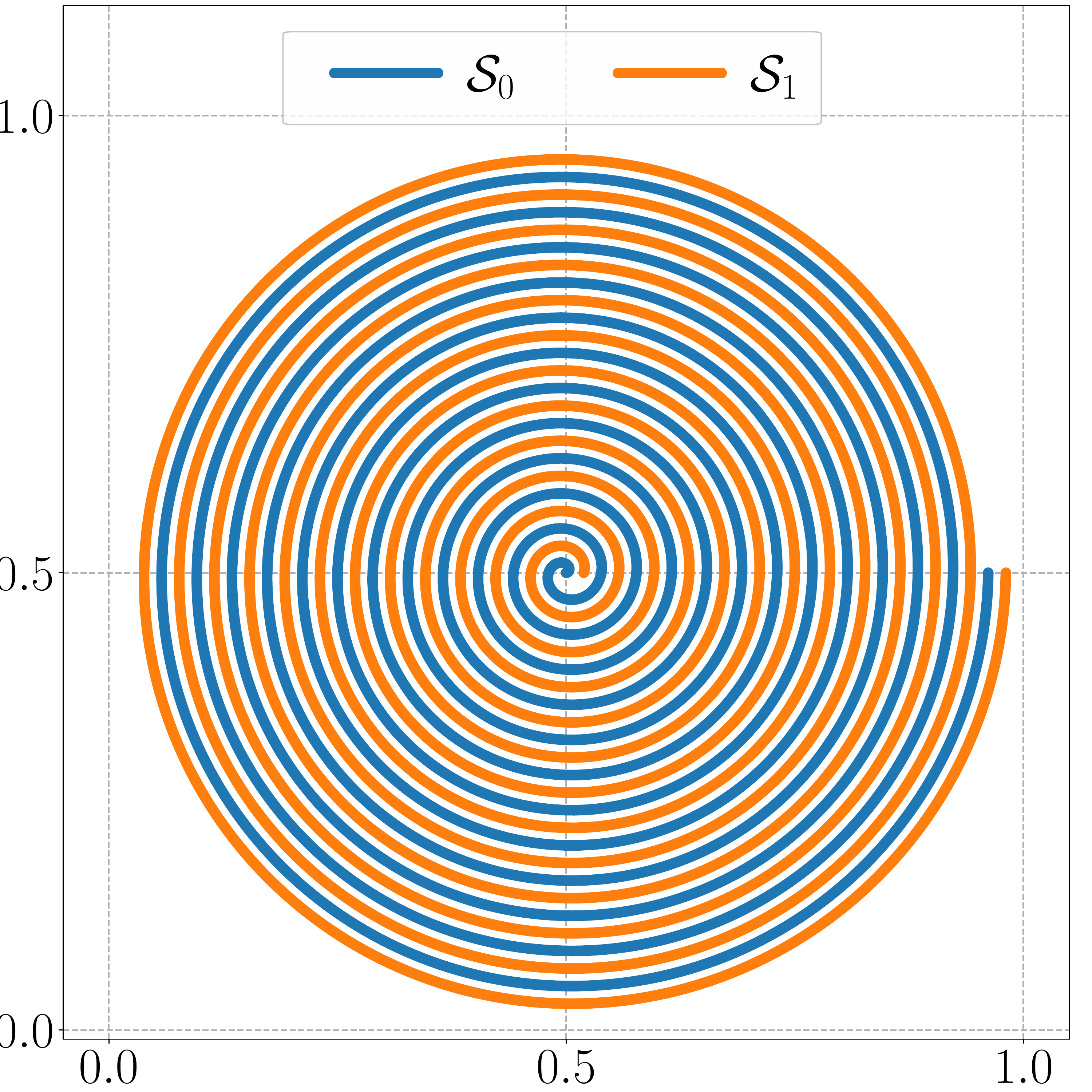}
    \end{subfigure}
    \hfill
    \begin{subfigure}[c]{0.487\linewidth}
    \centering            \includegraphics[width=0.901\textwidth]{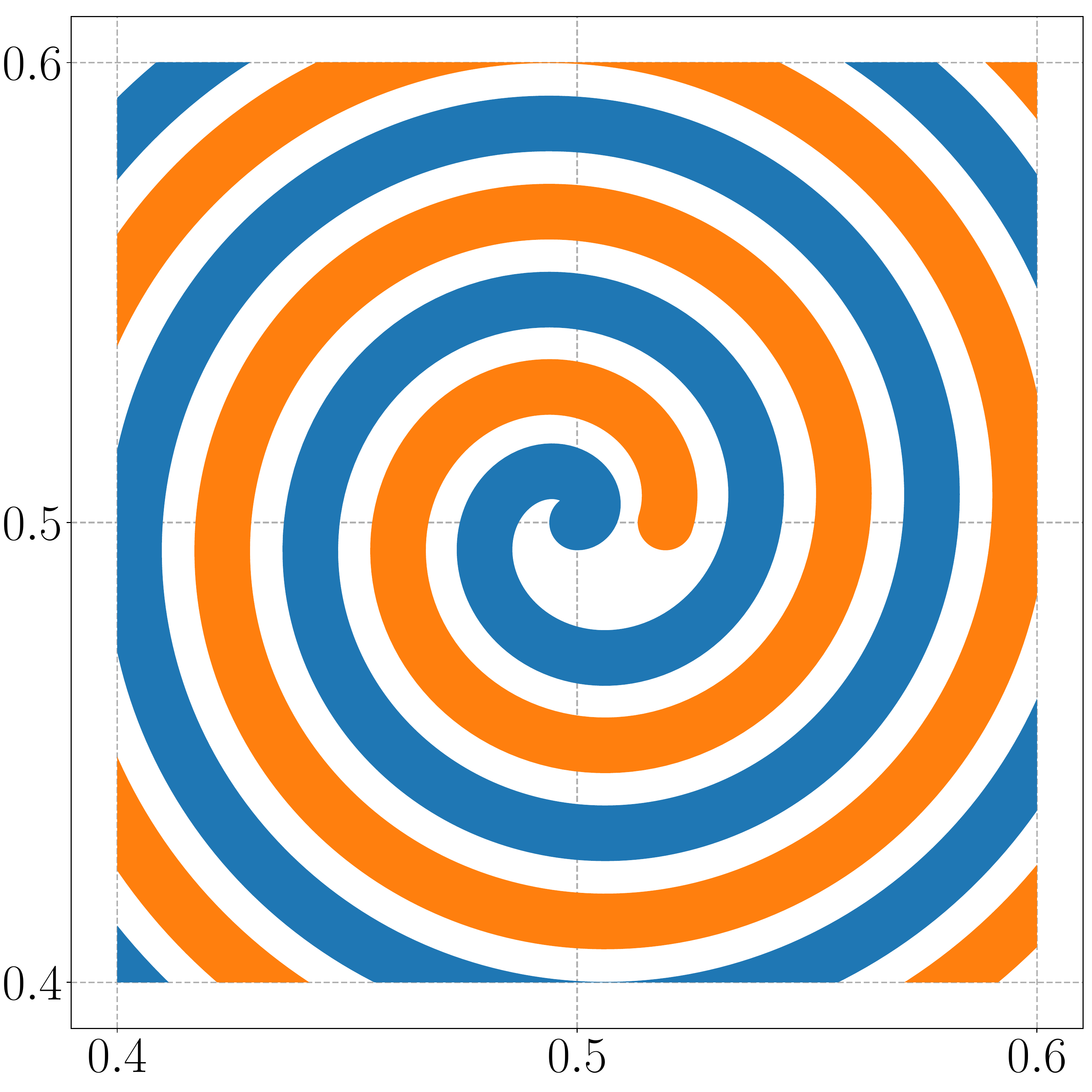}
    \end{subfigure}
    	\caption{Illustrations for $\calS_0$ and $\calS_1$. 
    	}
    	\label{fig:spiral}
\end{figure}

Let us delve into the construction details of the sets $\calS_0$ and $\calS_1$. An Archimedean spiral can be represented by
the equation
$r=a+b\theta$
in polar coordinates $(r,\theta)$ for proper $a,b\in\R$.
We begin by defining two curves
\begin{equation*}
	\caltildeC_i\coloneqq \Big\{\big(r_i\cos\theta,\,r_i\sin\theta\big):
	   r_i=a_i+b_i\theta, \,\   
	\theta\in [0,s\pi]
	\Big\}
\end{equation*}
for $i=0,1$, where $a_0=0$, $a_1=1$, $b_0=b_1={1}/{\pi}$, and $s=24$.
Next, 
we normalize $\caltildeC_i$ to obtain $\calC_i\subseteq [0,1]^2$ for each $i\in\{0,1\}$, where $\calC_i$ is defined as
\begin{equation*}
	{\calC}_i\coloneqq \bigg\{(x,y):x=\tfrac{\tildex+(s+2)}{2(s+2)},\    
	y=\tfrac{\tildey+(s+2)}{2(s+2)},\    
	(\tildex,\tildey)\in \caltildeC_i
	\bigg\}
\end{equation*}
for $i=0,1$. 
With $\calC_0$ and $\calC_1$ defined, we can construct the target sets $\calS_0$ and $\calS_1$ as
\begin{equation*}
	{\calS}_i\coloneqq \bigg\{(x,y):
	\sqrt{(x-u)^2+(y-v)^2}\le \varepsilon,\   
	(u,v)\in {\calC}_i
	\bigg\}
\end{equation*}
for $i=0,1$, where $\varepsilon=0.006$ in our experiments.
Refer to Figure~\ref{fig:spiral} for illustrations of $\calS_0$ and $\calS_1$.

In this experiment, we will employ the network architecture $\bmcalL_2\circ\bmg_n^{\circ r}\circ\bmcalL_1$ to classify samples in $\calS_0\cup \calS_1$ for different values of $r$ and $n$, specifically $r=1,2,3,4$ and $n=30,40$. 
Here, $\bmcalL_1$ and $\bmcalL_2$ represent two affine linear maps, and $\bmg_n$ corresponds to a ReLU network block. 
The construction of the ReLU network block $\bmg_n$ involves combining an affine linear map with the ReLU activation function.
Mathematically, $\bmg_n$ is defined as
\begin{equation*}   
    \bmg_n(\bmx)\coloneqq\sigma(\bmA \bmx+\bmb)
\end{equation*}
for any $\bmx\in\R^{n}$, where $\bmA\in\R^{n\times n}$ and $\bmb\in\R^{n}$ are parameters and $\sigma$ 
denotes the ReLU activation function that can be applied element-wise to a vector.
Subsequently, the input and output dimensions of the two affine linear maps, $\bmcalL_1:\mathbb{R}^2\rightarrow\mathbb{R}^{n}$ and $\bmcalL_2:\mathbb{R}^{n}\rightarrow\mathbb{R}^2$, are appropriately determined. 
The total number of parameters in $\bmcalL_2\circ \bmg_n^{\circ r}\circ \bmcalL_1$ can be easily verified to be $(2n+2)+(n^2+n)+(2n+n)=n^2+6n+2$ for different values of $r$ and $n$.
An important observation is that for $r\geq 2$, the parameters in $\bmcalL_2\circ \bmg_n^{\circ r}\circ \bmcalL_1$ are partially shared due to the repeated utilization of the ReLU network block $\bmg_n$.
Our objective is to provide numerical evidence demonstrating that increasing the value of $r$ results in improved test accuracies  for each fixed $n$.

Before proceeding with the numerical results, let us provide an overview of the hyperparameters employed in training our network architecture $\bmcalL_2\circ \bmg_n^{\circ r}\circ \bmcalL_1$ for different values of $r$ and $n$, specifically $r=1,2,3,4$ and $n=30,40$.
First, we generate training and test samples from $\calS_0$ and $\calS_1$ using the uniform distribution. Specifically, we randomly generate $3\times 10^5$ training samples and $3\times 10^4$ test samples for each class. These $6\times 10^5$ training samples are used for network training, while $6\times 10^4$ test samples are utilized to compute the test accuracy.
For optimization, we employ the RAdam method \cite{Liu2020On}. The training process consists of $1000$ epochs with a mini-batch size of $300$. The learning rate  is defined as $0.001\times0.95^{i-1}$  during epochs $5(i-1)+1$ to $5i$ for $i=1,2,\cdots,200$. 
To evaluate the model output, we apply the softmax activation function to the network output and employ the cross-entropy loss function to measure the loss between the target function and the network output.
To guarantee the standardization of training and test samples, we perform a rescaling procedure to adjust their mean to $0$ and standard deviation to $1$.
To ensure reliability of our results, we conduct the experiment $12$ times. Among these trials, 
we exclude $3$ top-performing and $3$ bottom-performing trials based on the average test accuracy over the last $100$ epochs. 
The target accuracy is then determined by averaging the test accuracies from the remaining $6$ trials for each epoch.

Let us now present the results of our experiments by comparing the numerical performances of $\bmcalL_2\circ \bmg_n^{\circ r}\circ \bmcalL_1$ for 
for varying values of $r$ and $n$, specifically $r=1,2,3,4$ and $n=30,40$. 
The target test accuracy is computed by averaging the test accuracies over the last $100$ epochs.
Table~\ref{tab:accuracy:comparison} provides a comprehensive overview of the test accuracies obtained by $\bmcalL_2\circ \bmg_n^{\circ r}\circ \bmcalL_1$ for different values of $r$ and $n$. To enhance the information presented in Table~\ref{tab:accuracy:comparison}, Figure~\ref{fig:test:accuracy} serves as a complementary visual representation. It showcases graphical depictions of the data, enabling an intuitive comparison and facilitating a visual analysis of the performance trends.
The results presented in Table~\ref{tab:accuracy:comparison} combined with the trends observed in Figure~\ref{fig:test:accuracy} confirm our initial expectation that increasing the value of $r$ leads to improved test accuracies. 
These experiment results provide compelling numerical evidence demonstrating the effectiveness of increasing $r$, which aligns with the theoretical results stated in Theorems~\ref{thm:main:Lp} and \ref{thm:main:Linfty}.

Finally, it is important to note that further increasing $r$ may not necessarily result in additional improvements in the results. This is because optimizing deep learning models is notoriously challenging, as it involves various issues such as local minima, saddle points, and vanishing gradients. In our experiments, the primary difficulty lies in identifying the global minimizer, especially when dealing with large values of $r$. 

\begin{table}[ht]
	\caption{Test accuracy comparison.} 
	\label{tab:accuracy:comparison}	
	\centering  
	\resizebox{0.912\linewidth}{!}{ 
		\begin{tabular}{ccccccccc} 
			\toprule
			$\bmcalL_2\circ \bmg_n^{\circ r}\circ \bmcalL_1$  & $r=1$ & $r=2$  & $r=3$ & $r=4$ \\
			\midrule
		\rowcolor{mygray}	 $n=30$ 
& $0.574781$
& $0.729258$
& $0.828750$
& $0.874866$ \\		
            \midrule
			$n=40$ 
& $0.575149$
& $0.797345$
& $0.871617$
& $0.904041$\\	
			\bottomrule
		\end{tabular} 
	}
\end{table} 


\begin{figure}[ht]
  \centering
      \begin{subfigure}[c]{0.487\linewidth}
    \centering            \includegraphics[width=0.9975\textwidth]{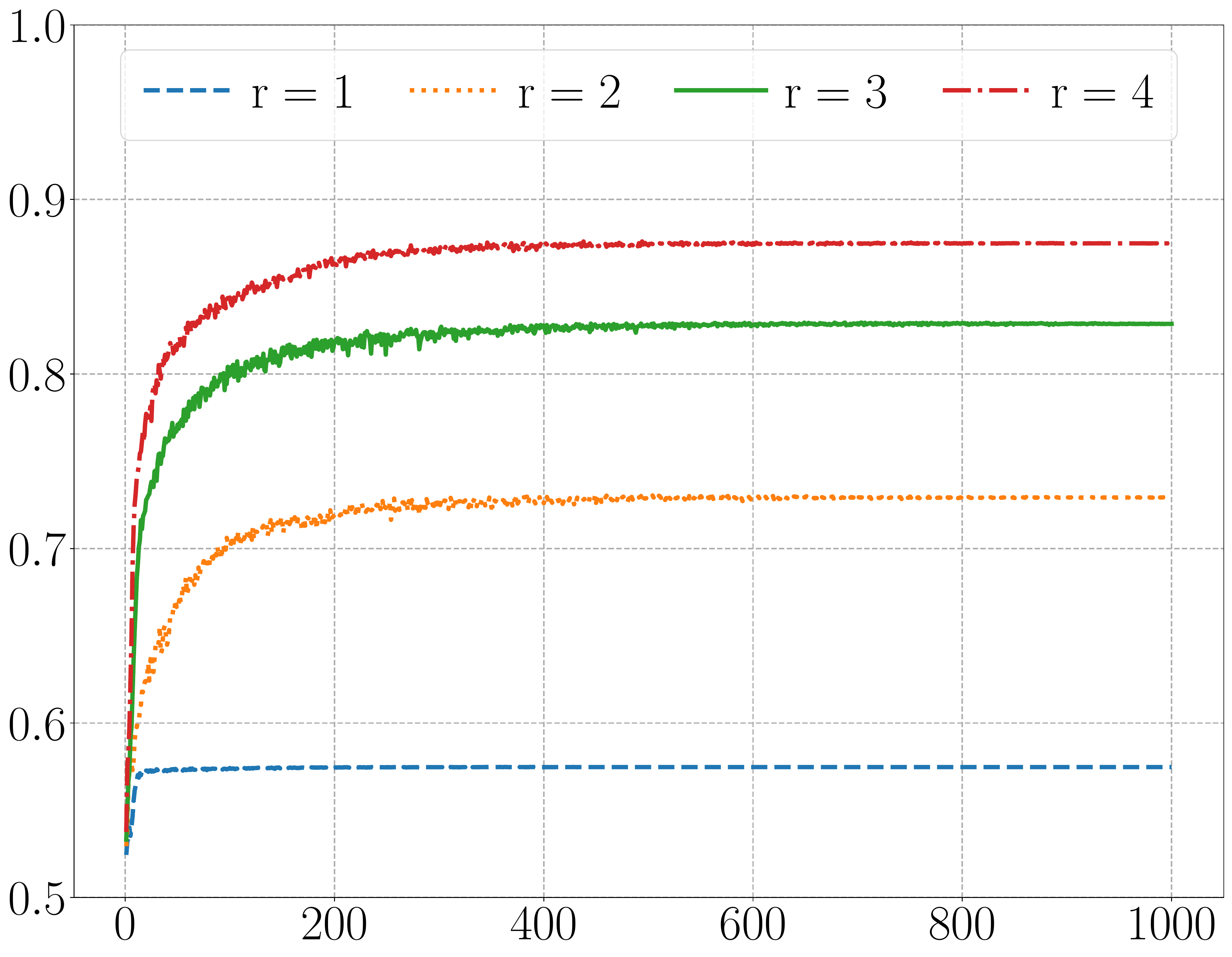}
    \subcaption{$n=30$.}
    \end{subfigure}
    \hfill
    \begin{subfigure}[c]{0.487\linewidth}
    \centering            \includegraphics[width=0.9975\textwidth]{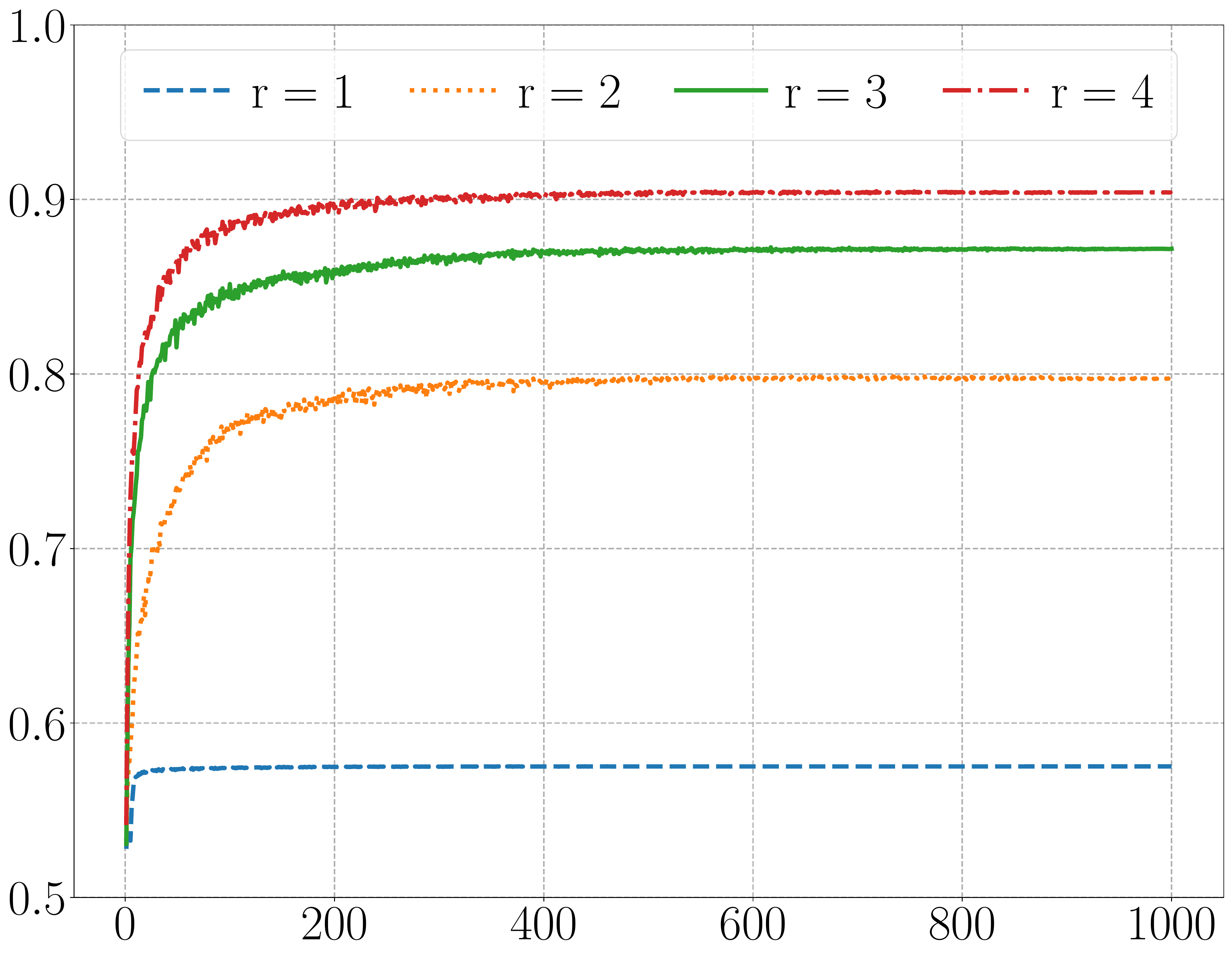}
    \subcaption{$n=40$.}
    \end{subfigure}
	\caption{
 Test accuracies across epochs: the $x$-axis represents the epoch number, while the $y$-axis corresponds to the test accuracy.
 }
	\label{fig:test:accuracy}

\end{figure}

\section{Conclusion}
\label{sec:conclusion}

This paper investigates the expressive power of deep neural networks from the perspective of function compositions. 
We demonstrate that  the repeated compositions of a single fixed-size ReLU network exhibit 
surprising
expressive power, despite the 
limited expressive capabilities of the individual network itself. 
As shown in Theorems~\ref{thm:main:Lp} and \ref{thm:main:Linfty}, our RCNet architecture $\mathcal{L}_2\circ \bm{g}^{\circ r}\circ \bm{\mathcal{L}}_1$ can approximate any continuous function $f\in C([0,1]^d)$ with an error $\mathcal{O}\big(\omega_f(r^{-1/d})\big)$. Here, $\bm{g}$ represents a fixed-size ReLU network, while $\bm{\mathcal{L}}_1$ and $\mathcal{L}_2$ correspond to two affine linear maps matching the dimensions.
Furthermore, we explore the connection between our findings and dynamical systems. Our results reveal that a continuous-depth network generated through a dynamical system possesses 
enormous
approximation capabilities, even when the dynamics function is time-independent and realized by a fixed-size ReLU network. 
Finally, we conduct experiments to provide numerical evidence that validates the theoretical results stated in Theorems~\ref{thm:main:Lp} and \ref{thm:main:Linfty}.


It is worth mentioning that our analysis is currently focused on the ReLU activation function and fully connected network architectures. Extending our results to other activation functions, such as the sigmoid and tanh functions, as well as different neural network architectures, such as convolutional neural networks, 
would be of great interest for future research.
Additionally, the numerical examples presented in this paper are relatively simple. Further exploration of the numerical performance of our network architecture and its application to real-world problems would be an intriguing direction for future studies. 
\ifdefined\isaccepted
\section*{Acknowledgements}

We extend our gratitude to Zuowei Shen, Yimin Zhong, and Haomin Zhou for their valuable insights and constructive feedback on our theoretical proofs and numerical experiments. Their contributions greatly enhanced the quality and rigor of our work.
Jianfeng Lu was partially supported by
NSF grants CCF-1910571 and DMS-2012286.
Hongkai Zhao was partially supported by NSF grant DMS-2012860.

\fi

%

\bibliography{references}

\begin{thebibliography}{45}
\providecommand{\natexlab}[1]{#1}
\providecommand{\url}[1]{\texttt{#1}}
\expandafter\ifx\csname urlstyle\endcsname\relax
  \providecommand{\doi}[1]{doi: #1}\else
  \providecommand{\doi}{doi: \begingroup \urlstyle{rm}\Url}\fi

\bibitem[Bao et~al.(2023)Bao, Li, Shen, Tai, Wu, and
  Xiang]{Bao2019ApproximationAO}
Bao, C., Li, Q., Shen, Z., Tai, C., Wu, L., and Xiang, X.
\newblock Approximation analysis of convolutional neural networks.
\newblock \emph{East Asian Journal on Applied Mathematics}, 13\penalty0
  (3):\penalty0 524--549, 2023.
\newblock ISSN 2079--7370.
\newblock \doi{https://doi.org/10.4208/eajam.2022-270.070123}.
\newblock URL
  \url{http://global-sci.org/intro/article_detail/eajam/21721.html}.

\bibitem[Barron(1993)]{barron1993}
Barron, A.~R.
\newblock Universal approximation bounds for superpositions of a sigmoidal
  function.
\newblock \emph{IEEE Transactions on Information Theory}, 39\penalty0
  (3):\penalty0 930--945, May 1993.
\newblock ISSN 0018-9448.
\newblock URL \url{https://doi.org/10.1109/18.256500}.

\bibitem[{Barron} \& {Klusowski}(2018){Barron} and
  {Klusowski}]{barron2018approximation}
{Barron}, A.~R. and {Klusowski}, J.~M.
\newblock Approximation and estimation for high-dimensional deep learning
  networks.
\newblock \emph{arXiv e-prints}, art. arXiv:1809.03090, September 2018.
\newblock URL \url{https://arxiv.org/abs/1809.03090}.

\bibitem[Bartlett et~al.(1998)Bartlett, Maiorov, and
  Meir]{Bartlett98almostlinear}
Bartlett, P., Maiorov, V., and Meir, R.
\newblock Almost linear {VC}-dimension bounds for piecewise polynomial
  networks.
\newblock \emph{Neural Computation}, 10\penalty0 (8):\penalty0 2159–2173,
  1998.
\newblock URL \url{https://doi.org/10.1162/089976698300017016}.

\bibitem[B{\"o}lcskei et~al.(2019)B{\"o}lcskei, Grohs, Kutyniok, and
  Petersen]{doi:10.1137/18M118709X}
B{\"o}lcskei, H., Grohs, P., Kutyniok, G., and Petersen, P.
\newblock Optimal approximation with sparsely connected deep neural networks.
\newblock \emph{SIAM Journal on Mathematics of Data Science}, 1\penalty0
  (1):\penalty0 8--45, 2019.
\newblock \doi{10.1137/18M118709X}.
\newblock URL \url{https://doi.org/10.1137/18M118709X}.

\bibitem[Chen et~al.(2019)Chen, Jiang, Liao, and Zhao]{Wenjing}
Chen, M., Jiang, H., Liao, W., and Zhao, T.
\newblock Efficient approximation of deep {ReLU} networks for functions on low
  dimensional manifolds.
\newblock In Wallach, H., Larochelle, H., Beygelzimer, A., d\textquotesingle
  Alch\'{e}-Buc, F., Fox, E., and Garnett, R. (eds.), \emph{Advances in Neural
  Information Processing Systems}, volume~32. Curran Associates, Inc., 2019.
\newblock URL
  \url{https://proceedings.neurips.cc/paper/2019/file/fd95ec8df5dbeea25aa8e6c808bad583-Paper.pdf}.

\bibitem[Chen et~al.(2018)Chen, Rubanova, Bettencourt, and
  Duvenaud]{NEURIPS2018_69386f6b}
Chen, R. T.~Q., Rubanova, Y., Bettencourt, J., and Duvenaud, D.~K.
\newblock Neural ordinary differential equations.
\newblock In Bengio, S., Wallach, H., Larochelle, H., Grauman, K.,
  Cesa-Bianchi, N., and Garnett, R. (eds.), \emph{Advances in Neural
  Information Processing Systems}, volume~31. Curran Associates, Inc., 2018.
\newblock URL
  \url{https://proceedings.neurips.cc/paper/2018/file/69386f6bb1dfed68692a24c8686939b9-Paper.pdf}.

\bibitem[Chui et~al.(2018)Chui, Lin, and Zhou]{10.3389/fams.2018.00014}
Chui, C.~K., Lin, S.-B., and Zhou, D.-X.
\newblock Construction of neural networks for realization of localized deep
  learning.
\newblock \emph{Frontiers in Applied Mathematics and Statistics}, 4:\penalty0
  14, 2018.
\newblock ISSN 2297-4687.
\newblock \doi{10.3389/fams.2018.00014}.
\newblock URL
  \url{https://www.frontiersin.org/article/10.3389/fams.2018.00014}.

\bibitem[Cybenko(1989)]{Cybenko1989ApproximationBS}
Cybenko, G.
\newblock Approximation by superpositions of a sigmoidal function.
\newblock \emph{Mathematics of Control, Signals, and Systems}, 2:\penalty0
  303--314, 1989.
\newblock URL \url{https://doi.org/10.1007/BF02551274}.

\bibitem[E(2017)]{WeinanE2017dynamicalsystems}
E, W.
\newblock A proposal on machine learning via dynamical systems.
\newblock \emph{Communications in Mathematics and Statistics}, 5:\penalty0
  1--11, 2017.
\newblock URL \url{https://doi.org/10.1007/s40304-017-0103-z}.

\bibitem[{E} et~al.(2022){E}, {Ma}, and {Wu}]{2019arXiv190608039E}
{E}, W., {Ma}, C., and {Wu}, L.
\newblock The {Barron} space and the flow-induced function spaces for neural
  network models.
\newblock \emph{Constructive Approximation}, 55:\penalty0 369--406, 2022.
\newblock URL \url{https://doi.org/10.1007/s00365-021-09549-y}.

\bibitem[Gribonval et~al.(2022)Gribonval, Kutyniok, Nielsen, and
  Voigtlaender]{2019arXiv190501208G}
Gribonval, R., Kutyniok, G., Nielsen, M., and Voigtlaender, F.
\newblock Approximation spaces of deep neural networks.
\newblock \emph{Constructive Approximation}, 55:\penalty0 259--367, 2022.
\newblock URL \url{https://doi.org/10.1007/s00365-021-09543-4}.

\bibitem[{G{\"u}hring} et~al.(2020){G{\"u}hring}, {Kutyniok}, and
  {Petersen}]{2019arXiv190207896G}
{G{\"u}hring}, I., {Kutyniok}, G., and {Petersen}, P.
\newblock Error bounds for approximations with deep {ReLU} neural networks in
  ${W}^{s,p}$ norms.
\newblock \emph{Analysis and Applications}, 18\penalty0 (05):\penalty0
  803--859, 2020.
\newblock \doi{10.1142/S0219530519410021}.
\newblock URL \url{https://doi.org/10.1142/S0219530519410021}.

\bibitem[He et~al.(2016)He, Zhang, Ren, and Sun]{7780459}
He, K., Zhang, X., Ren, S., and Sun, J.
\newblock Deep residual learning for image recognition.
\newblock In \emph{2016 IEEE Conference on Computer Vision and Pattern
  Recognition (CVPR)}, pp.\  770--778, June 2016.
\newblock URL \url{https://doi.org/10.1109/CVPR.2016.90}.

\bibitem[Holmes(2007)]{Holmes:2007}
Holmes, P.
\newblock {H}istory of dynamical systems.
\newblock \emph{Scholarpedia}, 2\penalty0 (5):\penalty0 1843, 2007.
\newblock URL \url{https://doi.org/10.4249/scholarpedia.1843}.

\bibitem[Hornik(1991)]{HORNIK1991251}
Hornik, K.
\newblock Approximation capabilities of multilayer feedforward networks.
\newblock \emph{Neural Networks}, 4\penalty0 (2):\penalty0 251--257, 1991.
\newblock ISSN 0893-6080.
\newblock \doi{https://doi.org/10.1016/0893-6080(91)90009-T}.
\newblock URL
  \url{http://www.sciencedirect.com/science/article/pii/089360809190009T}.

\bibitem[Hornik et~al.(1989)Hornik, Stinchcombe, and White]{HORNIK1989359}
Hornik, K., Stinchcombe, M., and White, H.
\newblock Multilayer feedforward networks are universal approximators.
\newblock \emph{Neural Networks}, 2\penalty0 (5):\penalty0 359--366, 1989.
\newblock ISSN 0893-6080.
\newblock \doi{https://doi.org/10.1016/0893-6080(89)90020-8}.
\newblock URL
  \url{http://www.sciencedirect.com/science/article/pii/0893608089900208}.

\bibitem[{Jiao} et~al.(2021){Jiao}, {Lai}, {Lu}, {Wang}, {Zhijian Yang}, and
  {Yang}]{jiao2021deep}
{Jiao}, Y., {Lai}, Y., {Lu}, X., {Wang}, F., {Zhijian Yang}, J., and {Yang}, Y.
\newblock Deep neural networks with {ReLU-Sine-Exponential} activations break
  curse of dimensionality on {H}{\"o}lder class.
\newblock \emph{arXiv e-prints}, art. arXiv:2103.00542, February 2021.
\newblock URL \url{https://arxiv.org/abs/2103.00542}.

\bibitem[{Li} et~al.(2022){Li}, {Lin}, and {Shen}]{2022arXiv220808707L}
{Li}, Q., {Lin}, T., and {Shen}, Z.
\newblock Deep neural network approximation of invariant functions through
  dynamical systems.
\newblock \emph{arXiv e-prints}, art. arXiv:2208.08707, August 2022.
\newblock URL \url{https://arxiv.org/abs/2208.08707}.

\bibitem[{Li} et~al.(2023){Li}, {Lin}, and {Shen}]{2019arXiv191210382L}
{Li}, Q., {Lin}, T., and {Shen}, Z.
\newblock Deep learning via dynamical systems: An approximation perspective.
\newblock \emph{Journal of the European Mathematical Society}, 25\penalty0
  (5):\penalty0 1671--1709, 2023.
\newblock URL \url{https://doi.org/10.4171/JEMS/1221}.

\bibitem[{Lin} et~al.(2022){Lin}, {Shen}, and {Li}]{2022arXiv221114047L}
{Lin}, T., {Shen}, Z., and {Li}, Q.
\newblock On the universal approximation property of deep fully convolutional
  neural networks.
\newblock \emph{arXiv e-prints}, art. arXiv:2211.14047, November 2022.
\newblock URL \url{https://arxiv.org/abs/2211.14047}.

\bibitem[Liu et~al.(2020)Liu, Jiang, He, Chen, Liu, Gao, and Han]{Liu2020On}
Liu, L., Jiang, H., He, P., Chen, W., Liu, X., Gao, J., and Han, J.
\newblock On the variance of the adaptive learning rate and beyond.
\newblock In \emph{International Conference on Learning Representations}, 2020.
\newblock URL \url{https://openreview.net/forum?id=rkgz2aEKDr}.

\bibitem[Lu et~al.(2021)Lu, Shen, Yang, and Zhang]{shijun:3}
Lu, J., Shen, Z., Yang, H., and Zhang, S.
\newblock Deep network approximation for smooth functions.
\newblock \emph{SIAM Journal on Mathematical Analysis}, 53\penalty0
  (5):\penalty0 5465--5506, 2021.
\newblock URL \url{https://doi.org/10.1137/20M134695X}.

\bibitem[Montanelli \& Yang(2020)Montanelli and Yang]{MO}
Montanelli, H. and Yang, H.
\newblock Error bounds for deep {ReLU} networks using the {Kolmogorov-Arnold}
  superposition theorem.
\newblock \emph{Neural Networks}, 129:\penalty0 1--6, 2020.
\newblock ISSN 0893-6080.
\newblock \doi{https://doi.org/10.1016/j.neunet.2019.12.013}.
\newblock URL
  \url{http://www.sciencedirect.com/science/article/pii/S0893608019304058}.

\bibitem[{Montanelli} et~al.(2021){Montanelli}, {Yang}, and {Du}]{bandlimit}
{Montanelli}, H., {Yang}, H., and {Du}, Q.
\newblock Deep {ReLU} networks overcome the curse of dimensionality for
  bandlimited functions.
\newblock \emph{Journal of Computational Mathematics}, 39\penalty0
  (6):\penalty0 801--815, 2021.
\newblock ISSN 1991-7139.
\newblock \doi{https://doi.org/10.4208/jcm.2007-m2019-0239}.
\newblock URL \url{http://global-sci.org/intro/article_detail/jcm/19912.html}.

\bibitem[Nakada \& Imaizumi(2020)Nakada and Imaizumi]{Ryumei}
Nakada, R. and Imaizumi, M.
\newblock Adaptive approximation and generalization of deep neural network with
  intrinsic dimensionality.
\newblock \emph{Journal of Machine Learning Research}, 21\penalty0
  (174):\penalty0 1--38, 2020.
\newblock URL \url{http://jmlr.org/papers/v21/20-002.html}.

\bibitem[Plummer et~al.(2022)Plummer, Dryden, Frost, Hoefler, and
  Saenko]{2006.10598}
Plummer, B.~A., Dryden, N., Frost, J., Hoefler, T., and Saenko, K.
\newblock Neural parameter allocation search.
\newblock In \emph{International Conference on Learning Representations}, 2022.
\newblock URL \url{https://openreview.net/forum?id=srtIXtySfT4}.

\bibitem[Savarese \& Maire(2019)Savarese and Maire]{savarese2018learning}
Savarese, P. and Maire, M.
\newblock Learning implicitly recurrent {CNN}s through parameter sharing.
\newblock In \emph{International Conference on Learning Representations}, 2019.
\newblock URL \url{https://openreview.net/forum?id=rJgYxn09Fm}.

\bibitem[Shen et~al.(2019)Shen, Yang, and Zhang]{shijun:1}
Shen, Z., Yang, H., and Zhang, S.
\newblock Nonlinear approximation via compositions.
\newblock \emph{Neural Networks}, 119:\penalty0 74--84, 2019.
\newblock ISSN 0893-6080.
\newblock \doi{https://doi.org/10.1016/j.neunet.2019.07.011}.
\newblock URL
  \url{http://www.sciencedirect.com/science/article/pii/S0893608019301996}.

\bibitem[Shen et~al.(2020)Shen, Yang, and Zhang]{shijun:2}
Shen, Z., Yang, H., and Zhang, S.
\newblock Deep network approximation characterized by number of neurons.
\newblock \emph{Communications in Computational Physics}, 28\penalty0
  (5):\penalty0 1768--1811, 2020.
\newblock ISSN 1991-7120.
\newblock URL \url{https://doi.org/10.4208/cicp.OA-2020-0149}.

\bibitem[Shen et~al.(2021{\natexlab{a}})Shen, Yang, and Zhang]{shijun:4}
Shen, Z., Yang, H., and Zhang, S.
\newblock Deep network with approximation error being reciprocal of width to
  power of square root of depth.
\newblock \emph{Neural Computation}, 33\penalty0 (4):\penalty0 1005--1036, 03
  2021{\natexlab{a}}.
\newblock ISSN 0899-7667.
\newblock \doi{10.1162/neco_a_01364}.
\newblock URL \url{https://doi.org/10.1162/neco\_a\_01364}.

\bibitem[Shen et~al.(2021{\natexlab{b}})Shen, Yang, and Zhang]{shijun:5}
Shen, Z., Yang, H., and Zhang, S.
\newblock Neural network approximation: {T}hree hidden layers are enough.
\newblock \emph{Neural Networks}, 141:\penalty0 160--173, 2021{\natexlab{b}}.
\newblock ISSN 0893-6080.
\newblock URL \url{https://doi.org/10.1016/j.neunet.2021.04.011}.

\bibitem[Shen et~al.(2022{\natexlab{a}})Shen, Yang, and
  Zhang]{shijun:arbitrary:error:with:fixed:size}
Shen, Z., Yang, H., and Zhang, S.
\newblock Deep network approximation: Achieving arbitrary accuracy with fixed
  number of neurons.
\newblock \emph{Journal of Machine Learning Research}, 23\penalty0
  (276):\penalty0 1--60, 2022{\natexlab{a}}.
\newblock URL \url{http://jmlr.org/papers/v23/21-1404.html}.

\bibitem[Shen et~al.(2022{\natexlab{b}})Shen, Yang, and
  Zhang]{shijun:intrinsic:parameters}
Shen, Z., Yang, H., and Zhang, S.
\newblock Deep network approximation in terms of intrinsic parameters.
\newblock In Chaudhuri, K., Jegelka, S., Song, L., Szepesvari, C., Niu, G., and
  Sabato, S. (eds.), \emph{Proceedings of the 39th International Conference on
  Machine Learning}, volume 162 of \emph{Proceedings of Machine Learning
  Research}, pp.\  19909--19934. PMLR, 17--23 Jul 2022{\natexlab{b}}.
\newblock URL \url{https://proceedings.mlr.press/v162/shen22g.html}.

\bibitem[{Shen} et~al.(2022){Shen}, {Yang}, and
  {Zhang}]{shijun:net:arc:beyond:width:depth}
{Shen}, Z., {Yang}, H., and {Zhang}, S.
\newblock Neural network architecture beyond width and depth.
\newblock In Koyejo, S., Mohamed, S., Agarwal, A., Belgrave, D., Cho, K., and
  Oh, A. (eds.), \emph{Advances in Neural Information Processing Systems},
  volume~35, pp.\  5669--5681. Curran Associates, Inc., 2022.
\newblock URL
  \url{https://proceedings.neurips.cc/paper_files/paper/2022/hash/257be12f31dfa7cc158dda99822c6fd1-Abstract-Conference.html}.

\bibitem[Suzuki(2019)]{suzuki2018adaptivity}
Suzuki, T.
\newblock Adaptivity of deep {ReLU} network for learning in {Besov} and mixed
  smooth {Besov} spaces: optimal rate and curse of dimensionality.
\newblock In \emph{International Conference on Learning Representations}, 2019.
\newblock URL \url{https://openreview.net/forum?id=H1ebTsActm}.

\bibitem[Wallingford et~al.(2022)Wallingford, Li, Achille, Ravichandran,
  Fowlkes, Bhotika, and Soatto]{9879069}
Wallingford, M., Li, H., Achille, A., Ravichandran, A., Fowlkes, C., Bhotika,
  R., and Soatto, S.
\newblock Task adaptive parameter sharing for multi-task learning.
\newblock In \emph{2022 IEEE/CVF Conference on Computer Vision and Pattern
  Recognition (CVPR)}, pp.\  7551--7560, 2022.
\newblock URL \url{https://doi.org/10.1109/CVPR52688.2022.00741}.

\bibitem[Wang et~al.(2020)Wang, Bai, Wu, Shi, Huang, King, Lyu, and
  Cheng]{NEURIPS2020_42cd63cb}
Wang, J., Bai, H., Wu, J., Shi, X., Huang, J., King, I., Lyu, M., and Cheng, J.
\newblock Revisiting parameter sharing for automatic neural channel number
  search.
\newblock In Larochelle, H., Ranzato, M., Hadsell, R., Balcan, M., and Lin, H.
  (eds.), \emph{Advances in Neural Information Processing Systems}, volume~33,
  pp.\  5991--6002. Curran Associates, Inc., 2020.
\newblock URL
  \url{https://proceedings.neurips.cc/paper/2020/file/42cd63cb189c30ed03e42ce2c069566c-Paper.pdf}.

\bibitem[{Wang} et~al.(2020){Wang}, {Cheng}, {Sapiro}, and
  {Qiu}]{2020arXiv200902386W}
{Wang}, Z., {Cheng}, X., {Sapiro}, G., and {Qiu}, Q.
\newblock {ACDC}: Weight sharing in atom-coefficient decomposed convolution.
\newblock \emph{arXiv e-prints}, art. arXiv:2009.02386, September 2020.
\newblock URL \url{https://arxiv.org/abs/2009.02386}.

\bibitem[Yarotsky(2017)]{yarotsky2017}
Yarotsky, D.
\newblock Error bounds for approximations with deep {ReLU} networks.
\newblock \emph{Neural Networks}, 94:\penalty0 103--114, 2017.
\newblock ISSN 0893-6080.
\newblock \doi{https://doi.org/10.1016/j.neunet.2017.07.002}.
\newblock URL
  \url{http://www.sciencedirect.com/science/article/pii/S0893608017301545}.

\bibitem[Yarotsky(2018)]{yarotsky18a}
Yarotsky, D.
\newblock Optimal approximation of continuous functions by very deep {ReLU}
  networks.
\newblock In Bubeck, S., Perchet, V., and Rigollet, P. (eds.),
  \emph{Proceedings of the 31st Conference On Learning Theory}, volume~75 of
  \emph{Proceedings of Machine Learning Research}, pp.\  639--649. PMLR, 06--09
  Jul 2018.
\newblock URL \url{http://proceedings.mlr.press/v75/yarotsky18a.html}.

\bibitem[Yarotsky \& Zhevnerchuk(2020)Yarotsky and
  Zhevnerchuk]{yarotsky:2019:06}
Yarotsky, D. and Zhevnerchuk, A.
\newblock The phase diagram of approximation rates for deep neural networks.
\newblock In Larochelle, H., Ranzato, M., Hadsell, R., Balcan, M.~F., and Lin,
  H. (eds.), \emph{Advances in Neural Information Processing Systems},
  volume~33, pp.\  13005--13015. Curran Associates, Inc., 2020.
\newblock URL
  \url{https://proceedings.neurips.cc/paper/2020/file/979a3f14bae523dc5101c52120c535e9-Paper.pdf}.

\bibitem[Zhang et~al.(2022)Zhang, Yang, Liu, and Guan]{9859706}
Zhang, L., Yang, Q., Liu, X., and Guan, H.
\newblock Rethinking hard-parameter sharing in multi-domain learning.
\newblock In \emph{2022 IEEE International Conference on Multimedia and Expo
  (ICME)}, pp.\  01--06, 2022.
\newblock URL \url{https://doi.org/10.1109/ICME52920.2022.9859706}.

\bibitem[Zhang(2020)]{shijun:thesis}
Zhang, S.
\newblock Deep neural network approximation via function compositions.
\newblock \emph{PhD Thesis, National University of Singapore}, 2020.
\newblock URL \url{https://scholarbank.nus.edu.sg/handle/10635/186064}.

\bibitem[Zhou(2020)]{ZHOU2019}
Zhou, D.-X.
\newblock Universality of deep convolutional neural networks.
\newblock \emph{Applied and Computational Harmonic Analysis}, 48\penalty0
  (2):\penalty0 787--794, 2020.
\newblock ISSN 1063-5203.
\newblock \doi{https://doi.org/10.1016/j.acha.2019.06.004}.
\newblock URL
  \url{http://www.sciencedirect.com/science/article/pii/S1063520318302045}.

\end{thebibliography}
\bibliographystyle{icml2023}

\newpage
\appendix
\onecolumn

\cleardoublepage
\vspace*{-10pt}
\tableofcontents

\newpage
\section{Proofs of Theorems~\ref{thm:main:Lp} and \ref{thm:main:Linfty}}
\label{sec:proof:main}



As we shall see later in the proofs of Theorems~\ref{thm:main:Lp} and \ref{thm:main:Linfty}, our main approach involves constructing a piecewise constant function to approximate the desired continuous function. However, the inherent continuity of ReLU networks hinders their ability to uniformly approximate piecewise constant functions effectively. To address this limitation, we introduce the concept of the trifling region $\Omega([0,1]^d,K,\delta)$ as defined in Equation~\eqref{eq:triflingRegionDef}. By utilizing ReLU networks, we can accurately represent piecewise constant functions outside the trifling region.

To streamline the proofs of Theorems~\ref{thm:main:Lp} and \ref{thm:main:Linfty}, we introduce an auxiliary theorem, referred to as Theorem~\ref{thm:main:gap} below, where we disregard the approximation within the trifling region $\Omega([0,1]^d,K,\delta)$.
\begin{theorem}
	\label{thm:main:gap}
	Given  a continuous function $f\in C([0,1]^d)$, for any $r\in \N^+$,
	there exist $\bmg\in \nn[\big]{39d+24}{3}{5d+3}{5d+3}$
	and two affine linear maps
	$\bmcalL_1:\R^d\to\R^{5d+3}$ and $\calL_2:\R^{5d+3}\to\R$
	such that 
	\begin{equation*}
		\big|\calL_2\circ \bmg^{\circ (3r-1)}\circ \bmcalL_1(\bmx)-f(\bmx)\big|\le 5\sqrt{d}\,\omega_f(r^{-1/d})\quad \tn{for any $\bmx\in [0,1]^d\backslash\Omega([0,1]^d,K,\delta)$},
	\end{equation*}
	where $K=\lfloor r^{1/d}\rfloor$ and $\delta$ is an arbitrary number in $(0,\tfrac{1}{3K}]$.
\end{theorem}

The proof of Theorem~\ref{thm:main:gap} will be presented in Section~\ref{sec:proof:thm:main:gap}.
Assuming the validity of Theorem~\ref{thm:main:gap}, we will provide the detailed proofs of  Theorems~\ref{thm:main:Lp} and \ref{thm:main:Linfty} in Sections~\ref{sec:proof:main:Lp} and \ref{sec:proof:main:Linfty}, respectively. 
To enhance clarity, Section~\ref{sec:notation} offers a concise overview of the notations employed throughout this paper.


\subsection{Notations}
\label{sec:notation}

Below is a summary of the fundamental notations employed in this paper.
\begin{itemize}
	\item The set difference of two sets $A$ and $B$ is denoted as $A\backslash B:=\{x:x\in A,\ x\notin B\}$. 
 
	\item The sets of natural numbers (including $0$), integers, rational numbers, and real numbers are denoted as $\N$, $\Z$, $\Q$, and $\R$, respectively. Set $\N^+=\N\backslash\{0\}$.

 	\item The indicator (characteristic) function of a set $A$ is denoted as $\one_{A}$,  which takes the value $1$ on elements of $A$ and $0$ otherwise.
	
	
	\item The floor and ceiling functions of a real number $x$ are denoted as 
 $\lfloor x\rfloor=\max \{n: n\le x,\ n\in \Z\}$ and $\lceil x\rceil=\min \{n: n\ge x,\ n\in \Z\}$.
	


	\item 	Vectors and matrices are represented by bold lowercase and uppercase letters, respectively. For example,
	$\bma=(a_1,\dotsc,a_d)\in\R^d$,
 $\bm{A}\in\mathbb{R}^{m\times n}$ is a real matrix of size $m\times n$, and $\bm{A}^T$ denotes the transpose of $\bm{A}$.  

	
		\item Slicing notation is used for a  vector $\bmx=(x_1,\dotsc,x_d)\in\R^d$,  where $[\bmx]_{[n:m]}$ denotes a slice of $\bmx$ from its $n$-th to the $m$-th entries and  $[\bmx]_{[n]}$ denotes the $n$-th entry of $\bmx$ for any $n,m\in \{1,2,\dotsc,d\}$ with $n\le m$. For example, if $\bmx=(x_1,x_2,x_3)\in\R^3$, then $[5\bmx]_{[2:3]}=(5x_2,5x_3)$ and $[6\bmx+1]_{[3]}=6x_3+1$.
	
	\item Given any $p\in [1,\infty]$, the $p$-norm (or $\ell^p$-norm) of a vector $\bmx=(x_1,\dotsc,x_d)\in\R^d$ is defined via
	\begin{equation*}
		\|\bmx\|_p=\|\bmx\|_{\ell^p}\coloneqq \big(|x_1|^p+\cdots+|x_d|^p\big)^{1/p}\quad \tn{if $p\in [1,\infty)$}
	\end{equation*}
and
 \begin{equation*}		\|\bmx\|_{\infty}=\|\bmx\|_{\ell^\infty}\coloneqq \max\big\{|x_i|: i=1,2,\dotsc,d\big\}.
	\end{equation*}
	
	\item By convention, $\sum_{j=n}^{m} a_j=0$
	if $n>m$, no matter what $a_j$ is for each $j$.
	
	\item For any $\theta= \sum_{i=1}^{n}\theta_i 2^{-i}\in[0,1)$, we use $\bin 0.\theta_1\theta_2\cdots \theta_L$ to denote the binary representation of $\theta$, i.e., $\theta=\sum_{i=1}^{n}\theta_i2^{-i}=\bin 0.\theta_1\theta_2\cdots \theta_n$. 
	
	\item 
	Given any $K\in \N^+$ and $\delta\in (0, \tfrac{1}{K})$, we define a trifling region  $\Omega([0,1]^d,K,\delta)$ of $[0,1]^d$ via
	\begin{equation}
		\label{eq:triflingRegionDef}
		\Omega([0,1]^d,K,\delta)\coloneqq\bigcup_{j=1}^{d} \bigg\{\bmx=(x_1,\dotsc,x_d)\black{\in [0,1]^d}: x_j\in \bigcup_{k=1}^{K-1}\Big(\tfrac{k}{K}-\delta,\   \tfrac{k}{K}\Big)\bigg\}.
	\end{equation}
	In the degenerate case $K=1$, $\Omega([0,1]^d,K,\delta)=\emptyset$. 
 Figure~\ref{fig:region} presents two examples of trifling regions.
	
	\begin{figure}[htbp!]  
 
		\centering
		\begin{minipage}{0.805\textwidth}
			\centering
			\begin{subfigure}[b]{0.33\textwidth}
				\centering            
				\includegraphics[width=0.999\textwidth]{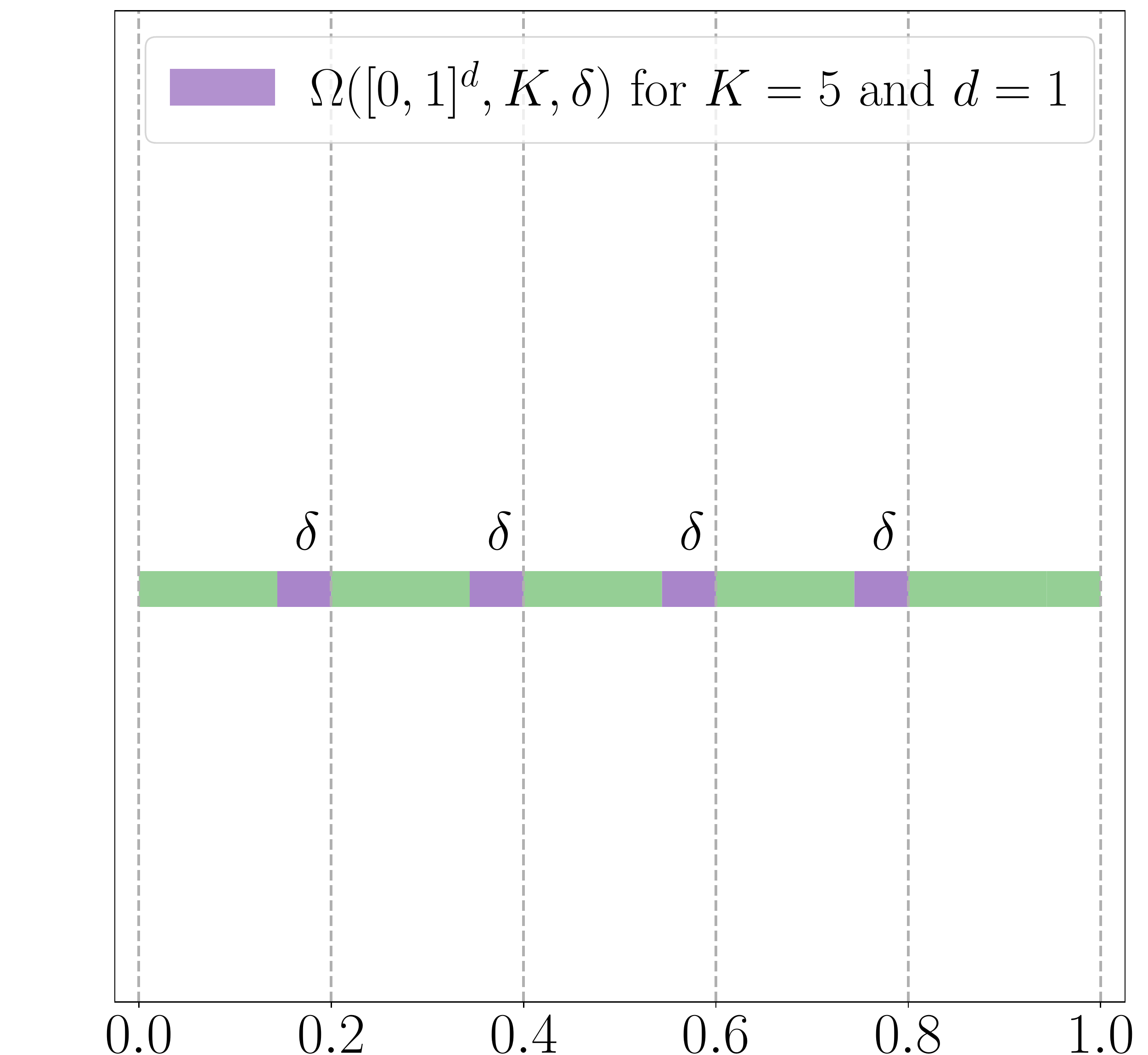}
				\subcaption{}
			\end{subfigure}
			\begin{minipage}{0.07\textwidth}
				\,
			\end{minipage}
			\begin{subfigure}[b]{0.33\textwidth}
				\centering            \includegraphics[width=0.999\textwidth]{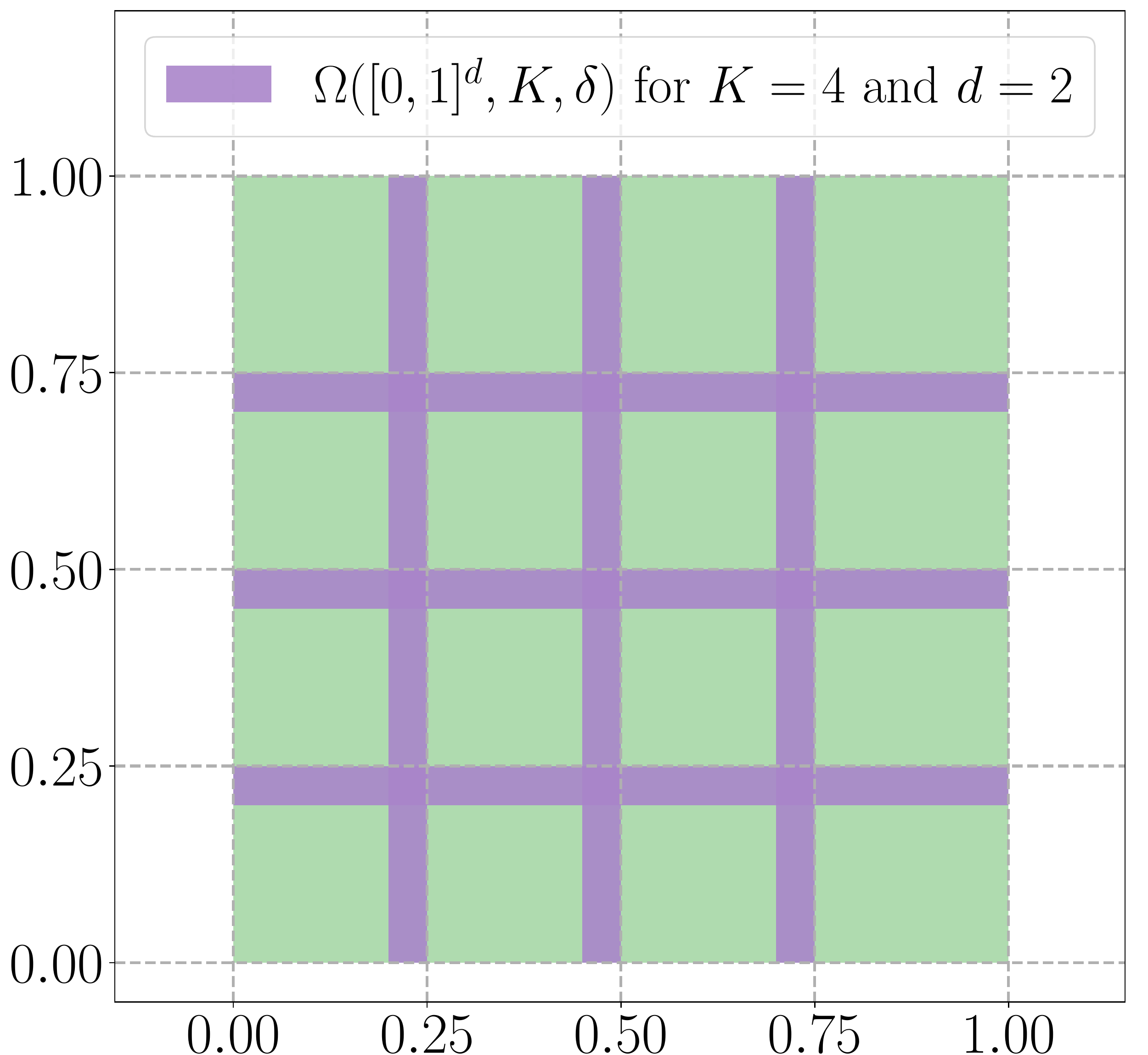}
				\subcaption{}
			\end{subfigure}
		\end{minipage}
		\caption{Two examples of trifling regions. (a)  $K=5,d=1$. (b) $K=4,d=2$.}
		\label{fig:region}
  
	\end{figure}
	
	
	
	%
	
	\item The rectified linear unit (ReLU) is denoted as $\sigma(x)=\max\{0,x\}$ for any $x\in\R$. With a slight abuse of notation, we allow $\sigma$ to be applied element-wise to a vector, i.e.,
$\sigma(\bmx)=\big(\sigma(x_1),\dotsc,\sigma(x_d)\big)$ 
	for any $\bmx=(x_1,\dotsc,x_d)\in \R^d$.
%
%

	\item 
 Suppose $\bmphi$ is a function realized by a ReLU network, whether scalar or vector-valued. Then, $\bmphi$ can be expressed as
	\begin{equation*}
		\begin{aligned}
			\bm{x}=\widetilde{\bm{h}}_0 
			\myto{2.42}^{\bm{W}_0,\ \bm{b}_0}_{\bmcalL_0} \bm{h}_1
			\myto{1.3015}^{\sigma} \widetilde{\bm{h}}_1 \quad \cdots\quad \myto{2.97}^{\bm{W}_{L-1},\ \bm{b}_{L-1}}_{\bmcalL_{L-1}} \bm{h}_L
			\myto{1.3015}^{\sigma} \widetilde{\bm{h}}_L
			\myto{2.42}^{\bm{W}_{L},\ \bm{b}_{L}}_{\bmcalL_L} \bm{h}_{L+1}=\bmphi(\bm{x}),
		\end{aligned}
	\end{equation*}
	where $\bm{W}_i\in \R^{N_{i+1}\times N_{i}}$ and $\bm{b}_i\in \R^{N_{i+1}}$ are the weight matrix and the bias vector in the $i$-th affine linear map $\bmcalL_i$, respectively, i.e., 
	\[\bm{h}_{i+1} =\bm{W}_i\cdot \bmtildeh_{i} + \bm{b}_i\eqqcolon \bmcalL_i(\bmtildeh_{i})\quad \tn{for $i=0,1,\dotsc,L$,}\]  
	and
	\[
	\widetilde{\bm{h}}_i=\sigma(\bm{h}_i)\quad \tn{for $i=1,2,\dotsc,L$.}
	\]
 Furthermore, $\bmphi$ can be expressed as a composition of functions. Specifically, it can be written as
	\begin{equation*}
		\bmphi =\bmcalL_L\circ\sigma\circ
		\ \cdots \  \circ 
		\bmcalL_1\circ\sigma\circ\bmcalL_0.
	\end{equation*}
 Refer to Figure~\ref{fig:ReLUeg} for an illustration.
	
	\begin{figure}[htbp!]     
		
		\centering            \includegraphics[width=0.7\textwidth]{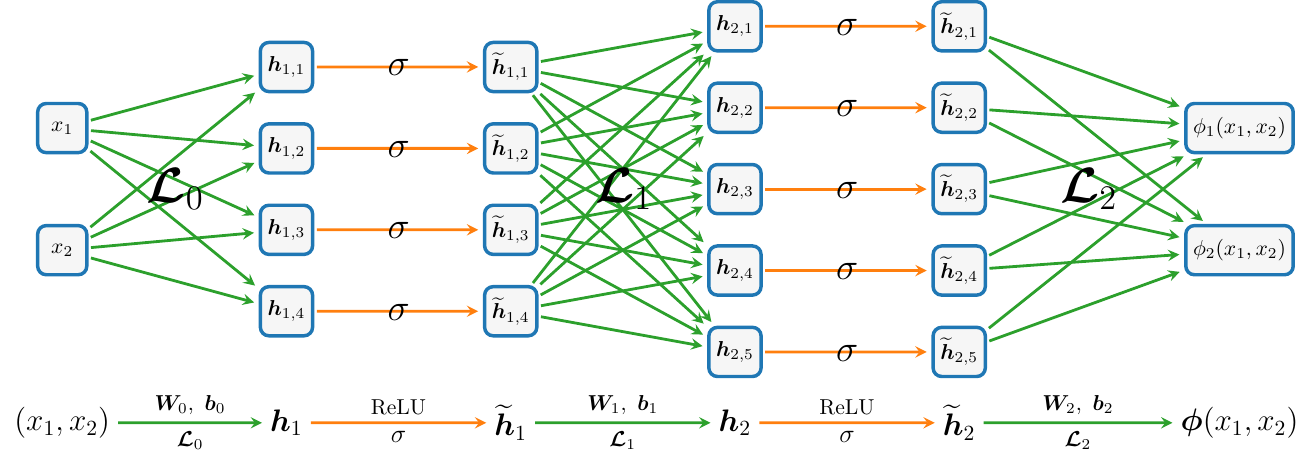}
\caption{An example of a ReLU network of width $5$ and depth $2$. The network realizes a vector-valued function $\bmphi=(\phi_1,\phi_2)$.}
		\label{fig:ReLUeg}
  
	\end{figure}

	\item A network is referred to as "a network of width $N$ and depth $L$" if it satisfies the following conditions.
 	\begin{itemize}
		\item The number of neurons in each  {hidden} layer of this network is less than or equal to $N$.
		\item The number of {hidden} layers of this network  is less than or equal to $L$.
	\end{itemize} 
 
\end{itemize}

\subsection{Proof of Theorem~\ref{thm:main:Lp} with Theorem~\ref{thm:main:gap}}
\label{sec:proof:main:Lp}

 By assuming the validity of Theorem~\ref{thm:main:gap}, we can proceed to prove Theorem~\ref{thm:main:Lp}.

\begin{proof}[Proof of Theorem~\ref{thm:main:Lp}]		
 We assume that $f$ is not a constant function, as considering constant functions would lead to a trivial case. Therefore, for any $t > 0$, we have $\omega_f(t) > 0$.
Set $K=\lfloor r^{1/d}\rfloor$ and let $\delta\in (0,\tfrac{1}{3K}]$ be an arbitrary number  determined later.
	By Theorem~\ref{thm:main:gap}, 	there exist \begin{equation*}
	    \bmg_1\in \nn[\big]{39d+24}{3}{5d+3}{5d+3}
	\end{equation*}
and two affine linear maps $\bmhatcalL_1:\R^d\to\R^{5d+3}$ and $\hatcalL_2:\R^{5d+3}\to\R$
	such that 
	\begin{equation}
    \label{eq:approx:outside}
		\big|\hatcalL_2\circ \bmg_1^{\circ (3r-1)}\circ \bmhatcalL_1(\bmx)-f(\bmx)\big|\le 5\sqrt{d}\,\omega_f(r^{-1/d})\quad \tn{for any $\bmx\in [0,1]^d\backslash\Omega([0,1]^d,K,\delta)$}.
	\end{equation}

That means the approximation error is well controlled outside the trifling region $\Omega([0,1]^d,K,\delta)$.
To control the $L^p$-norm of $\hatcalL_2\circ \bmg_1^{\circ (3r-1)}\circ \bmhatcalL_1-f$, we need to further bound it inside $\Omega([0,1]^d,K,\delta)$.
To this end,
	we  define 
	\begin{equation*}
		g_2(x)\coloneqq \begin{cases}
			M & \tn{if} \   x\ge M\\
			x &  \tn{if} \   |x|< M\\
			-M &  \tn{if}  \   x\le -M, 
		\end{cases}\quad \tn{where}\  M= M_f= \|f\|_{L^\infty([0,1]^d)}+5\sqrt{d}\,\omega_f(1).
	\end{equation*}
Clearly, $\big\|g_2\circ\hatcalL_2\circ \bmg_1^{\circ (3r-1)}\circ \bmhatcalL_1\big\|_{L^\infty(\R^d)}\le M$. Moreover, for any $\bmx\in [0,1]^d\backslash\Omega([0,1]^d,K,\delta)$, we have
\begin{equation*}
\begin{split}
		\hatcalL_2\circ \bmg_1^{\circ (3r-1)}\circ \bmhatcalL_1(\bmx)
		&\in \Big[f(\bmx)-5\sqrt{d}\,\omega_f(1),\   f(\bmx)+5\sqrt{d}\,\omega_f(1)\Big]\\
		&\subseteq \Big[-\|f\|_{L^\infty([0,1]^d)}-5\sqrt{d}\,\omega_f(1),\   \|f\|_{L^\infty([0,1]^d)}+5\sqrt{d}\,\omega_f(1)\Big]=  [-M,M],
\end{split}
\end{equation*}
implying 
\begin{equation*}
	g_2\circ \hatcalL_2\circ \bmg_1^{\circ (3r-1)}\circ \bmhatcalL_1(\bmx)=\hatcalL_2\circ \bmg_1^{\circ (3r-1)}\circ \bmhatcalL_1(\bmx).
\end{equation*}

	We claim $g_2\in \nnOneD{4}{2}{\R}{\R}$. To see this, we need to show how to realize $g_2$ by a ReLU network. Clearly, we have
	\begin{equation*}
		g_2(x)+M=\min\big\{\sigma(x+M),\, 2M\big\}\quad \tn{for any $x\in \R$,}
	\end{equation*}
	implying
	\begin{equation*}
		\begin{split}
					g_2(x)&=\min\big\{\sigma(x+M),\, 2M\big\}-M\\
			&=\tfrac{1}{2}\Big(\sigma\big(\sigma(x+M)+M\big)-\sigma\big(-\sigma(x+M)-M\big)-\sigma\big(\sigma(x+M)-M\big)-\sigma\big(-\sigma(x+M)+M\big)\Big)-M,
		\end{split}
	\end{equation*}
where the last equality comes from 
	\begin{equation*}
		\min\{a,b\}=\tfrac{1}{2}\big(a+b-|a-b|\big)=\tfrac{1}{2}\big(\sigma(a+b)-\sigma(-a-b)-\sigma(a-b)-\sigma(-a+b)\big)\quad \tn{for any $a,b\in \R$}.
	\end{equation*}
As shown in Figure~\ref{fig:g2clip}, $g_2\in \nnOneD{4}{2}{\R}{\R}$ as desired.
	
	\begin{figure}[ht]
		
		\begin{center}
\includegraphics[width=0.6\columnwidth]{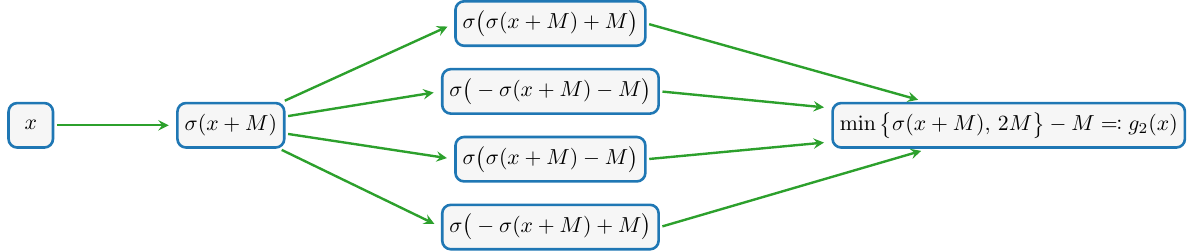}
			\caption{An illustration of the ReLU network realizing $g_2$.}
			\label{fig:g2clip}
		\end{center}
		
	\end{figure}

	Let $\hatcalL_3:\R\to\R$ as the identity map. Then, by Proposition~\ref{prop:two:blocks:three:affine} with 
	$N_1=39d+24$, $N_2=4$, $L_1=3$, $L_2=2$, $d_0=d$, $d_1=5d+3$, $d_2=d_3=1$ therein and setting $\tilded=5d+3\ge \max\{d_1,d_2\}$, 
	there exist 
	\begin{equation*}
		\begin{split}
			\bmg&\in \nn[\big]{(39d+24)+ 4+ 6\tilded  +2}{\max\{3+2,2+1\}}{\tilded+2}{\tilded+2}\\
			&=\nn[\big]{69d+48}{5}{5d+5}{5d+5}\\
		\end{split}
	\end{equation*}
	and two affine linear maps
	$\bmcalL_1:\R^d\to\R^{5d+5}$ and $\calL_2:\R^{5d+5}\to\R$
	such that 
	\begin{equation*}
		\hatcalL_3\circ g_2\circ \hatcalL_2\circ \bmg_1^{\circ (3r-1)}\circ \bmhatcalL_1(\bmx)
		=\calL_2\circ \bmg^{\circ (3r-1 +1 +1)}\circ \bmcalL_1(\bmx)
		=\calL_2\circ \bmg^{\circ (3r+1)}\circ \bmcalL_1(\bmx)
	\end{equation*}
for any $\bmx\in [-1,1]^{d}\supseteq [0,1]^d$. By defining $\phi\coloneqq \calL_2\circ \bmg^{\circ (3r+1)}\circ \bmcalL_1$, we have
\begin{equation*}
	\phi(\bmx)=\calL_2\circ \bmg^{\circ (3r+1)}\circ \bmcalL_1(\bmx)
	=	\hatcalL_3\circ g_2\circ \hatcalL_2\circ \bmg_1^{\circ (3r-1)}\circ \bmhatcalL_1(\bmx)
	=g_2\circ \hatcalL_2\circ \bmg_1^{\circ (3r-1)}\circ \bmhatcalL_1(\bmx)
\end{equation*}
for any $\bmx\in [0,1]^d$. Recall that
$\big\|g_2\circ\hatcalL_2\circ \bmg_1^{\circ (3r-1)}\circ \bmhatcalL_1\big\|_{L^\infty(\R^d)}\le M$ and 
\begin{equation*}
	g_2\circ \hatcalL_2\circ \bmg_1^{\circ (3r-1)}\circ \bmhatcalL_1(\bmx)=\hatcalL_2\circ \bmg_1^{\circ (3r-1)}\circ \bmhatcalL_1(\bmx)\quad \tn{for any $\bmx\in [0,1]^d\backslash\Omega([0,1]^d,K,\delta)$.}
\end{equation*}
Thus, we have
\begin{equation*}
	\big|\phi(\bmx)-f(\bmx)\big|\le \|\phi\|_{L^\infty([0,1]^d)}+\|f\|_{L^\infty([0,1]^d)}\le \big\|g_2\circ \hatcalL_2\circ \bmg_1^{\circ (3r-1)}\circ \bmhatcalL_1\big\|_{L^\infty([0,1]^d)}+M\le 2M
\end{equation*}
for any $\bmx\in [0,1]^d$ and 
\begin{equation*}
	\big|\phi(\bmx)-f(\bmx)\big|
	=  \Big|g_2\circ \hatcalL_2\circ \bmg_1^{\circ (3r-1)}\circ \bmhatcalL_1(\bmx)-f(\bmx)\Big|
	=  \Big|\hatcalL_2\circ \bmg_1^{\circ (3r-1)}\circ \bmhatcalL_1(\bmx)-f(\bmx)\Big|
	\le 5\sqrt{d}\,\omega_f(r^{-1/d})
\end{equation*}
for any $\bmx\in [0,1]^d\backslash\Omega([0,1]^d,K,\delta)$, where the last inequality comes from Equation~\eqref{eq:approx:outside}.

Observe that the Lebesgue measure of $\Omega([0,1]^d,K,\delta)$ is bounded by $Kd\delta$. Hence, by choosing a small $\delta\in (0,\tfrac{1}{3K}]$ with
\begin{equation*}
	Kd\delta(2M)^p= \lfloor r^{-1/d}\rfloor d\delta (2M)^p \le \Big(\omega_f(r^{-1/d})\Big)^p,
\end{equation*}
we have
\begin{equation*}
	\begin{split}
	\|\phi-f\|_{L^p([0,1]^d)}^p
		&=\int_{\Omega([0,1]^d,K,\delta)}|		\phi(\bmx)-f(\bmx)|^p\tn{d}\bmx+\int_{[0,1]^d\backslash\Omega([0,1]^d,K,\delta)}|\phi(\bmx)-f(\bmx)|^p\tn{d}\bmx\\
		&\le Kd\delta(2M)^p+ \Big(5\sqrt{d}\,\omega_f(r^{-1/d})\Big)^p\\
		&\le \Big(\omega_f(r^{-1/d})\Big)^p
		+ \Big(5\sqrt{d}\,\omega_f(r^{-1/d})\Big)^p
		\le \Big(6\sqrt{d}\,\omega_f(r^{-1/d})\Big)^p.
	\end{split}
\end{equation*}
Therefore, we can conclude that
$\|\calL_2\circ \bmg^{\circ (3r+1)}\circ \bmcalL_1-f\|_{L^p([0,1]^d)}=\|\phi-f\|_{L^p([0,1]^d)}\le 6\sqrt{d}\,\omega_f(r^{-1/d})$.
Thus, we have completed the proof of Theorem~\ref{thm:main:Lp}.
\end{proof}

\subsection{Proof of Theorem~\ref{thm:main:Linfty}  with Theorem~\ref{thm:main:gap}}
\label{sec:proof:main:Linfty}




To establish Theorem~\ref{thm:main:Linfty}, we will rely on Theorem~\ref{thm:main:gap}, which permits unbounded approximation errors in the trifling region $\Omega([0,1]^d,K,\delta)$. However, when it comes to proving Theorem~\ref{thm:main:Linfty} using pointwise approximation, it becomes essential to control the approximation error within the trifling region. To address this, we introduce a separate theorem that specifically deals with the approximation within the trifling region.

\begin{theorem}[Lemma~$3.11$ of \cite{shijun:thesis} or Lemma~$3.4$ of \cite{shijun:3}]
	\label{thm:gap:handling}
	Given any $\varepsilon>0$, $K\in \N^+$, and $\delta\in (0, \tfrac{1}{3K}]$,
	assume $f\in C([0,1]^d)$ and $g:\R^d\to \R$ is a general function with
	\begin{equation*}
		|g(\bmx)-f(\bmx)|\le \varepsilon\quad \tn{for any $\bmx\in [0,1]^d\backslash \Omega([0,1]^d,K,\delta)$.}
	\end{equation*}
	Then
	\begin{equation*}
		|\phi(\bmx)-f(\bmx)|\le \varepsilon+d\cdot\omega_f(\delta)\quad \tn{for any $\bmx\in [0,1]^d$,}
	\end{equation*}
	where $\phi\coloneqq  \phi_d$ is defined by induction through $\phi_0\coloneqq g$ and
	\begin{equation*}
		\phi_{i+1}(\bmx)\coloneqq  \middleValue\big(\phi_{i}(\bmx-\delta\bme_{i+1}),\, \phi_{i}(\bmx),\,\phi_{i}(\bmx+\delta\bme_{i+1})\big)\quad \tn{$i=0,1,\dotsc,d-1$,}
	\end{equation*}
	where $\{\bme_i\}_{i=1}^d$ is the standard basis in $\mathbb{R}^d$ and $\middleValue(\cdot,\cdot,\cdot)$ is the function returning the middle value of three inputs. 	 
\end{theorem}

Now, we are prepared to provide the detailed proof of Theorem~\ref{thm:main:Linfty} by assuming the validity of Theorem~\ref{thm:main:gap}.
\begin{proof}[Proof of Theorem~\ref{thm:main:Linfty}]	
	We may assume $f$ is not a constant function since it is  a trivial case. Then $\omega_f(t)>0$ for any $t>0$. 
	Set $K=\lfloor r^{1/d}\rfloor$ and choose a sufficiently small $\delta\in(0,\tfrac{1}{3K}]$ such that 
	\begin{equation*}
		d\cdot \omega_f(\delta)\le \omega_f\big(r^{-1/d}\big).
	\end{equation*}	
	By Theorem~\ref{thm:main:gap}, 		there exist $\bmg_0\in \nn[\big]{39d+24}{3}{5d+3}{5d+3}$
and two affine linear maps
$\bmcalL_{0,1}:\R^d\to\R^{5d+3}$ and $\calL_{0,2}:\R^{5d+3}\to\R$
such that 
\begin{equation*}
	\big|\calL_{0,2}\circ \bmg_0^{\circ (3r-1)}\circ \bmcalL_{0,1}(\bmx)-f(\bmx)\big|\le 5\sqrt{d}\,\omega_f(r^{-1/d})\quad \tn{for any $\bmx\in [0,1]^d\backslash\Omega([0,1]^d,K,\delta)$}.
\end{equation*}

Define $\phi_0\coloneqq \calL_{0,2}\circ \bmg_0^{\circ (3r-1)}\circ \bmcalL_{0,1}$. 
	By Theorem~\ref{thm:gap:handling}  with $g=\phi_0$ and $\varepsilon=5\sqrt{d}\,\omega_f\big(r^{-1/d}\big)>0$ therein, we have
	\begin{equation}\label{eq:phi-f:all:x}
		|\phi(\bmx)-f(\bmx)|\le \varepsilon+d\cdot \omega_f(\delta)\le  6\sqrt{d}\,\omega_f\big(r^{-1/d}\big)\quad \tn{for any $\bmx\in [0,1]^d$},
	\end{equation}
	where $\phi\coloneqq  \phi_d$ is defined by induction through 
	\begin{equation*}
		\phi_{i+1}(\bmx)\coloneqq  \middleValue\Big(\phi_{i}(\bmx-\delta\bme_{i+1}),  \  \phi_{i}(\bmx),
	\   \phi_{i}(\bmx+\delta\bme_{i+1})\Big)\quad \tn{for any $\bmx\in \R^d$ and $i=0,1,\dotsc,d-1$.}
	\end{equation*}
	Here, $\{\bme_i\}_{i=1}^d$ is the standard basis in $\mathbb{R}^d$ and $\middleValue(\cdot,\cdot,\cdot)$ is the function returning the middle value of three inputs.

It remains to show $\phi=\phi_d$ can be represented as the desired form.	
	We claim that $\phi_i$ can be represented as 
	\begin{equation*}
		\phi_i=\calL_{i,2}\circ \bmg_i^{\circ r_i}\circ\bmcalL_{i,1} \quad \tn{on $[-A_i,A_i]^d$\quad for $i=0,1,\dotsc,d$},
	\end{equation*}
where 
$r_i$, $A_i$, $\bmcalL_{i,1}$, $\calL_{i,2}$, and $\bmg_i$ satisfy the following conditions:
\begin{itemize}
	\item $r_i=3r+2i-1$ and $A_i=d+1-i$;
	\item $\bmcalL_{i,1}:\R^d\to\R^{d_i}$ and $\calL_{i,2}:\R^{d_i}\to\R$ are two affine linear maps with $d_i=3^i(5d+4)-1$;
	\item $\bmg_i\in \nn{N_i }{L_i}{d_i}{d_i}$ with $N_i=4^{i+5}d$ and $L_i=3+2i$. 
\end{itemize}

	We will prove this claim by induction. First, let us consider the base case $i=0$. Clearly, $\phi_0= \calL_{0,2}\circ \bmg_0^{\circ (3r-1)}\circ \bmcalL_{0,1}=\calL_{0,2}\circ \bmg_0^{\circ r_0}\circ \bmcalL_{0,1}$ on $\R^d\supseteq [-A_0,A_0]^d$, where 	$d_0=3^{0}(5d+4)-1=5d+3$, $\bmcalL_{0,1}:\R^d\to\R^{d_0}$ and $\calL_{0,2}:\R^{d_0}\to\R$ are two affine linear maps and
	\begin{equation*}
		\bmg_0\in \nn[\big]{39d+24}{3}{5d+3}{5d+3}\subseteq \nn[\Big]{N_0=4^{0+5} d=1024d}{L_0=3+0=3}{d_0}{d_0}.
	\end{equation*}

	Next, let us assume the claim holds for the case $i=j\in \{0,1,\dotsc,d-1\}$. We will prove the claim for the case $i=j+1$.
	By the induction hypothesis, 	 $\phi_i$ can be represented as 
	\begin{equation*}
		\phi_j=\calL_{j,2}\circ \bmg_j^{\circ r_j}\circ\bmcalL_{j,1} \quad \tn{on $[-A_j,A_j]^d$\quad for $j=0,1,\dotsc,d$},
	\end{equation*}
	where $\bmcalL_{j,1}:\R^d\to\R^{d_j}$ and $\calL_{j,2}:\R^{d_j}\to\R$ are two affine linear maps  and $\bmg_j\in \nn{N_j}{L_j}{d_j}{d_j}$.

	Define $\bmhatcalL_{j+1,1}:\R^{d}\to\R^{3d_j}$ via 
	\begin{equation*}
		\bmhatcalL_{j+1,1}(\bmx)
		\coloneqq \Big( \bmcalL_{j,1}(\bmx-\delta  \bme_{j+1}), \
		\ \bmcalL_{j,1}(\bmx), \   \bmcalL_{j,1}(\bmx+\delta  \bme_{k+1})\Big)\quad \tn{for any $\bmx\in\R^d$,}
	\end{equation*}
$\bmhatg_{j+1}:\R^{3d_j}\to\R^{3d_j}$ via
\begin{equation*}
	\bmhatg_{j+1}(\bmu,\bmv,\bmw)\coloneqq \Big(\bmg_j(\bmu),\   \bmg_j(\bmv),\  \bmg_j(\bmw)\Big)\quad \tn{for any $\bmu,\bmv,\bmw\in \R^{d_j}$,}
\end{equation*}
$\bmhatcalL_{j+1,2}:\R^{3d_j}\to\R^3$ via
\begin{equation*}
		\bmhatcalL_{j+1,2}(\bmu,\bmv,\bmw)
	\coloneqq \Big( \calL_{j,2}(\bmu), 
	\ \calL_{j,2}(\bmv), \   \calL_{j,2}(\bmw)\Big)\quad \tn{for any $\bmu,\bmv,\bmw\in\R^{d_j}$,}
\end{equation*}
$\bmhatG:\R^3\to\R^3$ via
\begin{equation*}
	\bmhatG(y_1,y_2,y_3)\coloneqq \Big(\middleValue(y_1,y_2,y_3), \ 0, \ 0\Big)\quad \tn{for any $(y_1,y_2,y_3)\in \R^3$,}
\end{equation*}
and $\hatcalL_3:\R^3\to\R$ via
\begin{equation*}
	\hatcalL_3(y_1,y_2,y_3)\coloneqq y_1\quad \tn{for any $(y_1,y_2,y_3)\in\R^3$.}
\end{equation*}

Note that $A_{j+1}=d+1-(j+1)=A_j-1\le A_j-\delta$.
For any $\bmx\in [-A_{j+1},A_{j+1}]^d\subseteq [-A_j+\delta,A_j-\delta]^d$, we have $\bmx-\delta \bme_{j+1},\, \bmx,\, \bmx+\delta \bme_{j+1}\in [-A_j,A_j]^d$, implying
\begin{equation*}
	\begin{split}
		\phi_{j+1}(\bmx)
		&=\middleValue\Big(\phi_j(\bmx-\delta  \bme_{j+1}), \ 
		\phi_j(\bmx), \   \phi_j(\bmx+\delta  \bme_{j+1})\Big)
		= \hatcalL_3\circ \bmhatG \Big(\phi_j(\bmx-\delta  \bme_{j+1}), \ 
		\phi_j(\bmx), \   \phi_j(\bmx+\delta  \bme_{j+1})\Big)\\
		&= \hatcalL_3\circ \bmhatG \Bigg(\calL_{j,2}\circ\bmg_j^{\circ r_j}\circ \bmcalL_{j,1}(\bmx-\delta  \bme_{j+1}), \quad 
		\calL_{j,2}\circ\bmg_j^{\circ r_j}\circ \bmcalL_{j,1}(\bmx), \quad   \calL_{j,2}\circ\bmg_j^{\circ r_j}\circ \bmcalL_{j,1}(\bmx+\delta  \bme_{j+1})\Bigg)\\
		&= \hatcalL_3\circ \bmhatG \circ \bmhatcalL_{j+1,2} \Bigg(\bmg_j^{\circ r_j}\circ \bmcalL_{j,1}(\bmx-\delta  \bme_{j+1}), \quad 
		\bmg_j^{\circ r_j}\circ \bmcalL_{j,1}(\bmx), \quad   \bmg_j^{\circ r_j}\circ \bmcalL_{j,1}(\bmx+\delta  \bme_{j+1})\Bigg)\\
		&= \hatcalL_3\circ \bmhatG \circ \bmhatcalL_{j+1,2}\circ \bmhatg_{j+1}^{\circ r_j} \bigg( \bmcalL_{j,1}(\bmx-\delta  \bme_{j+1}), \quad 
		\ \bmcalL_{j,1}(\bmx), \quad    \bmcalL_{j,1}(\bmx+\delta  \bme_{k+1})\bigg)\\
		&= \hatcalL_3\circ \bmhatG \circ \bmhatcalL_{j+1,2}\circ \bmhatg_{j+1}^{\circ r_j}\circ \bmhatcalL_{j+1,1}(\bmx).\\
	\end{split}
\end{equation*}
Clearly, $\bmg_j\in \nn{N_j}{L_j}{d_j}{d_j}$ implies $\bmhatg_{j+1}\in \nn{3N_j}{L_j}{3d_j}{3d_j}$.
By Lemma~$3.1$ of \cite{shijun:5}, $\middleValue(\cdot,\cdot,\cdot)$ can be realized by a ReLU network of width $14$ and depth $2$, implying 
$\bmhatG\in \nn{14}{2}{3}{3}$.  

Then, by Proposition~\ref{prop:two:blocks:three:affine} with 
$\hatN_1=3N_j$, $\hatN_2=14$, $\hatL_1=L_j$, $\hatL_2=2$, $\hatd_0=d$, $\hatd_1=3d_j$, $\hatd_2=3$, $\hatd_3=1$ therein and setting  $\hatd=3d_j=\max\{3d_j,3\}=\max\{\hatd_1,\hatd_2\}$, there exist
\begin{equation*}
\begin{split}
		\bmg_{j+1}
		&\in \nn[\Big]{\hatN_1+\hatN_2+6\hatd+2}{\hspace*{5pt}\max\{\hatL_1+2,\,\hatL_2+1\}}{\hatd+2}{\hatd+2}\\
  &= \nn[\Big]{3N_j+14+18d_j+2}{\hspace*{5pt}\max\{L_j+2,\,2+1\}}{3d_j+2}{3d_j+2}\\
		&=	 \nn[\Big]{3(4^{j+5}d)+18\big(3^j(5d+4)-1\big)+16}{\hspace*{5pt} 3+2j+2}{3(3^j(5d+4)-1)+2}{3(3^j(5d+4)-1)+2}\\
		&\subseteq  \nn[\Big]{3(4^{j+5}d)+ 3^{j+5}d}{\hspace*{5pt} 3+2(j+1)}{3^{j+1}(5d+4)-1}{ 
      3^{j+1}(5d+4)-1}\\
		&\subseteq  \nn[\Big]{4^{(j+1)+5}d}{\hspace*{5pt} 3+2(j+1)}{3^{j+1}(5d+4)-1}{3^{j+1}(5d+4)-1}
		= \nn[\big]{N_{j+1}}{L_{j+1}}{d_{j+1}}{d_{j+1}}
\end{split}
\end{equation*}
and two affine linear maps $\bmcalL_{j+1,1}:\R^d\to\R^{d_{j+1}}$ and $\calL_{j+1,2}:\R^{d_{j+1}}\to\R$ such that
\begin{equation*}
	\hatcalL_3\circ \bmhatG \circ \bmhatcalL_{j+1,2}\circ \bmhatg_{j+1}^{\circ r_j}\circ \bmhatcalL_{j+1,1}(\bmx)=\calL_{j+1,2}\circ \bmg_{j+1}^{\circ (r_j+1+1)} \circ \bmcalL_{j+1,1}(\bmx)=\calL_{j+1,2}\circ \bmg_{j+1}^{\circ r_{j+1}} \circ \bmcalL_{j+1,1}(\bmx)
\end{equation*}
for any $\bmx\in [-A_{j+1},A_{j+1}]$,
where the last equality comes from $r_j+1+1=(3r+2j-1)+1+1=3r+2(j+1)-1=r_{j+1}$. Therefore,
for any $\bmx\in [-A_{j+1},A_{j+1}]$, we have
\begin{equation*}
	\phi_{j+1}(\bmx)
	=\hatcalL_3\circ \bmhatG \circ \bmhatcalL_{j+1,2}\circ \bmhatg_{j+1}^{\circ r_j}\circ \bmhatcalL_{j+1,1}(\bmx)=
	\calL_{j+1,2}\circ \bmg_{j+1}^{\circ r_{j+1}} \circ \bmcalL_{j+1,1}(\bmx).
\end{equation*}
By the principle of mathematical induction,  we finish the proof of the claim. 

Then, by the claim and setting $\tilded=3^d(5d+4)-1$, $\phi=\phi_d$ can be represented as 
\begin{equation*}
	\phi=\phi_d=\calL_{d,2}\circ \bmg_d^{\circ r_d}\circ\bmcalL_{d,1}
=\calL_{d,2}\circ \bmg_d^{\circ (3r+2d-1)}\circ\bmcalL_{d,1} \quad \tn{on $[-A_d,A_d]^d=[-1,1]^d\supseteq [0,1]^d$},
\end{equation*}
where $\bmcalL_{d,1}:\R^d\to\R^{\tilded}$ and $\calL_{i,2}:\R^{\tilded}\to\R$ are two affine linear maps and 
\begin{equation*}
    \bmg_d\in \nn{N_d }{L_d}{\tilded}{\tilded}=\nn{4^{d+5}d}{3+2d}{\tilded}{\tilded}.
\end{equation*}

By defining $\bmcalL_1\coloneqq \bmcalL_{d,1}$, $\bmg\coloneqq\bmg_d$, and $\calL_2\coloneqq \calL_{d,2}$, we have
$\calL_2\circ\bmg^{\circ (3r+2d-1)}\circ \bmcalL_1=\calL_{d,2}\circ \bmg_d^{\circ (3r+2d-1)}\circ\bmcalL_{d,1}=\phi$.
It follows from Equation~\eqref{eq:phi-f:all:x} that
\begin{equation*}
	\Big|\calL_2\circ\bmg^{\circ (3r+2d-1)}\circ \bmcalL_1(\bmx)-f(\bmx)\Big|
	=|\phi(\bmx)-f(\bmx)|\le 6\sqrt{d}\, \omega_f(r^{-1/d})\quad \tn{for any $\bmx\in[0,1]^d$}.
\end{equation*}
Thus, we finish the proof of Theorem~\ref{thm:main:Linfty}.	
\end{proof}

\section{Proof of Theorem~\ref{thm:main:gap}  with Propositions}
\label{sec:proof:thm:main:gap}

In this section, we will provide the proof of the auxiliary theorem, Theorem~\ref{thm:main:gap}, by relying on Propositions~\ref{prop:floor:approx}, \ref{prop:point:fitting}, and \ref{prop:two:blocks:three:affine}. The detailed proofs of these propositions can be found in Sections~\ref{sec:proof:prop:floor:approx}, \ref{sec:proof:prop:point:fitting}, and \ref{sec:proof:prop:two:blocks:three:affine}, respectively. By assuming the validity of these three propositions, we  now proceed to prove Theorem~\ref{thm:main:gap}.



\begin{proof}[Proof of Theorem~\ref{thm:main:gap}]
We may assume $\omega_f(t)>0$ for any $t>0$ since $\omega_f(t_0)=0$ for some $t_0>0$ implies $f$ is a constant function, which is a trivial case.
Clearly, $|f(\bm{x})-f(\bmzero)|\le \omega_f(\sqrt{d})$ for any ${\bm{x}}\in [0,1]^d$. By defining 
\begin{equation*}
	\tildef\coloneqq  f-f(\bmzero)+\omega_f(\sqrt{d}),
\end{equation*} we have
$\omega_\tildef(t)=\omega_f(t)$ for any $t\ge 0$ and  $0\le \tildef (\bm{x}) \le 2\omega_f(\sqrt{d})$ for any ${\bm{x}}\in [0,1]^d$. 

Set $K=\lfloor r^{1/d}\rfloor$ and let $\delta$ be an arbitrary number in $(0,\tfrac{1}{3K}]$.
The proof can be divided into four main steps as follows.
\begin{enumerate}
	\item 
	Divide $[0,1]^d$ into a set of cubes $\{Q_{\bm{\beta}}\}_{\bm{\beta}\in \{0,1,\dotsc,K-1\}^d}$ and  $\Omega([0,1]^d,K,\delta)$. Denote $\bmx_\bmbeta$ as the vertex of $Q_\bmbeta$ with minimum $\|\cdot\|_1$ norm, where $\Omega([0,1]^d,K,\delta)$ is the trifling region defined in Equation~\eqref{eq:triflingRegionDef}.
	
	\item Use Proposition~\ref{prop:floor:approx} to construct a vector function $\bmPhi_1=\bmhatcalL_2\circ \bmG_1^{\circ (r-1)}\circ \bmhatcalL_1$ mapping $\bmx\in Q_\bmbeta$ to  $\bmbeta$ for each $\bmbeta\in \{0,1,\dotsc,K-1\}^d$, i.e., $\bmPhi_1(\bmx)=\bmbeta$ for all $\bmx\in Q_\bmbeta$, where $\bmhatcalL_1$ and $\bmhatcalL_2$ are affine linear maps and $\bmG_1$ is realized by a fixed-size ReLU network.
	
	\item Construct a  function $\phi_2=\hatcalL_5\circ \bmg_2^{\circ (2r-1)}\circ \bmhatcalL_4\circ \bmhatcalL_3$  mapping the index $\bmbeta$ approximately to $\tildef(\bmx_\bmbeta)$ for each $\bmbeta$, where $\bmhatcalL_3$, $\bmhatcalL_4$, and $\hatcalL_5$ are affine linear maps and $\bmg_2$ is realized by a fixed-size ReLU network. This core step can be further divided into two sub-steps:	
	\begin{enumerate}
		\item[3.1.] Design an affine linear map $\bmhatcalL_3$ bijectively mapping the index set $\{0,1,\dotsc,K-1\}^d$ to an auxiliary set $\mathcal{A}_1\subseteq \big\{\tfrac{j}{2K^d}:j=0,1,\dotsc,2K^d\big\}$ defined later. See Figure~\ref{fig:h+A12} for an illustration.
		
		
		
		\item[3.2.] Apply Proposition~\ref{prop:point:fitting} to design a sub-network to realize a function $\hatcalL_5\circ \bmg_2^{\circ (2r-1)}\circ \bmhatcalL_4$ mapping $\bmhatcalL_3(\bmbeta)$ approximately to $\tildef(\bmx_\bmbeta)$ for each $\bmbeta\in\{0,1,\dotsc,K-1\}^d$. Then, $\phi_2=\hatcalL_5\circ \bmg_2^{\circ (2r-1)}\circ \bmhatcalL_4\circ \bmhatcalL_3$ maps $\bmbeta$ approximately to $\tildef(\bmx_\bmbeta)$ for each $\bmbeta$.
	\end{enumerate}
	
	\item Construct the desired function $\phi$ via $\phi= \phi_2 \circ \bmPhi_1 +f(\bmzero)-\omega_f(\sqrt{d})=\hatcalL_5\circ \bmg_2^{\circ (2r-1)}\circ \bmhatcalL_4\circ \bmhatcalL_3\circ \bmhatcalL_2\circ \bmG_1^{\circ (r-1)}\circ \bmhatcalL_1+f(\bmzero)-\omega_f(\sqrt{d})$ and we use Proposition~\ref{prop:two:blocks:three:affine} to show $\phi$ can be represented as $\calL_2\circ \bmg^{\circ (3r-1)}\circ \bmcalL_1$, where $\bmcalL_1$ and $\calL_2$ are affine linear maps and $\bmg$ is realized by a fixed-size ReLU network. Then we have
	$\phi_2 \circ \bmPhi_1(\bmx)=\phi_2(\bmbeta)\approx \tildef(\bmx_\bmbeta)\approx \tildef(\bmx)$ for any $\bmx\in Q_\bmbeta$ and $\bmbeta\in \{0,1,\dotsc,K-1\}^d$, implying $\phi(\bmx)= \phi_2 \circ \bmPhi_1(\bmx) +f(\bmzero)-\omega_f(\sqrt{d})
	\approx \tildef(\bmx) +f(\bmzero)-\omega_f(\sqrt{d})= f(\bmx)$.
\end{enumerate}

%
The details of the above steps are presented below.

\mystep{1}{Divide $[0,1]^d$ into  $\{Q_{\bm{\beta}}\}_{\bm{\beta}\in \{0,1,\dotsc,K-1\}^d}$ and  $\Omega([0,1]^d,K,\delta)$.}

For each $d$-dimensional index  ${\bm{\beta}}= (\beta_1,\beta_2,\dotsc,\beta_d)\in \{0,1,\dotsc,K-1\}^d$, define  $\bmx_\bmbeta \coloneqq \bmbeta/K$ and 
\[
Q_{\bm{\beta}}\coloneqq\Big\{{\bm{x}}= (x_1,x_2,\dotsc,x_d)\in [0,1]^d:x_i\in\big[\tfrac{\beta_i}{K},\tfrac{\beta_i+1}{K}-\delta\cdot \one_{\{\beta_i\le K-2\}}\big], \quad i=1,2,\dotsc,d\Big\}.
\]
Clearly, $\bmx_\bmbeta=\bmbeta/K$ is the vertex of $Q_\bmbeta$ with minimum $\|\cdot\|_1$ norm and 
\[[0,1]^d= \big(\cup_{\bm{\beta}\in \{0,1,\dotsc,K-1\}^d}Q_{\bm{\beta}}\big)\bigcup \Omega([0,1]^d,K,\delta),\]
where $\Omega([0,1]^d,K,\delta)$ is the trifling region defined in Equation~\eqref{eq:triflingRegionDef}.
See Figure~\ref{fig:Q+TR} for illustrations of   $\Omega([0,1]^d,K,\delta)$, $Q_{\bm{\beta}}$, and $\bmx_\bmbeta$ for $\bm{\beta}\in \{0,1,\dotsc,K-1\}^d$.

\begin{figure}[htbp!]

	\centering
	\begin{minipage}{0.8\textwidth}
		\centering
		\begin{subfigure}[b]{0.435\textwidth}
			\centering
			\includegraphics[width=0.75\textwidth]{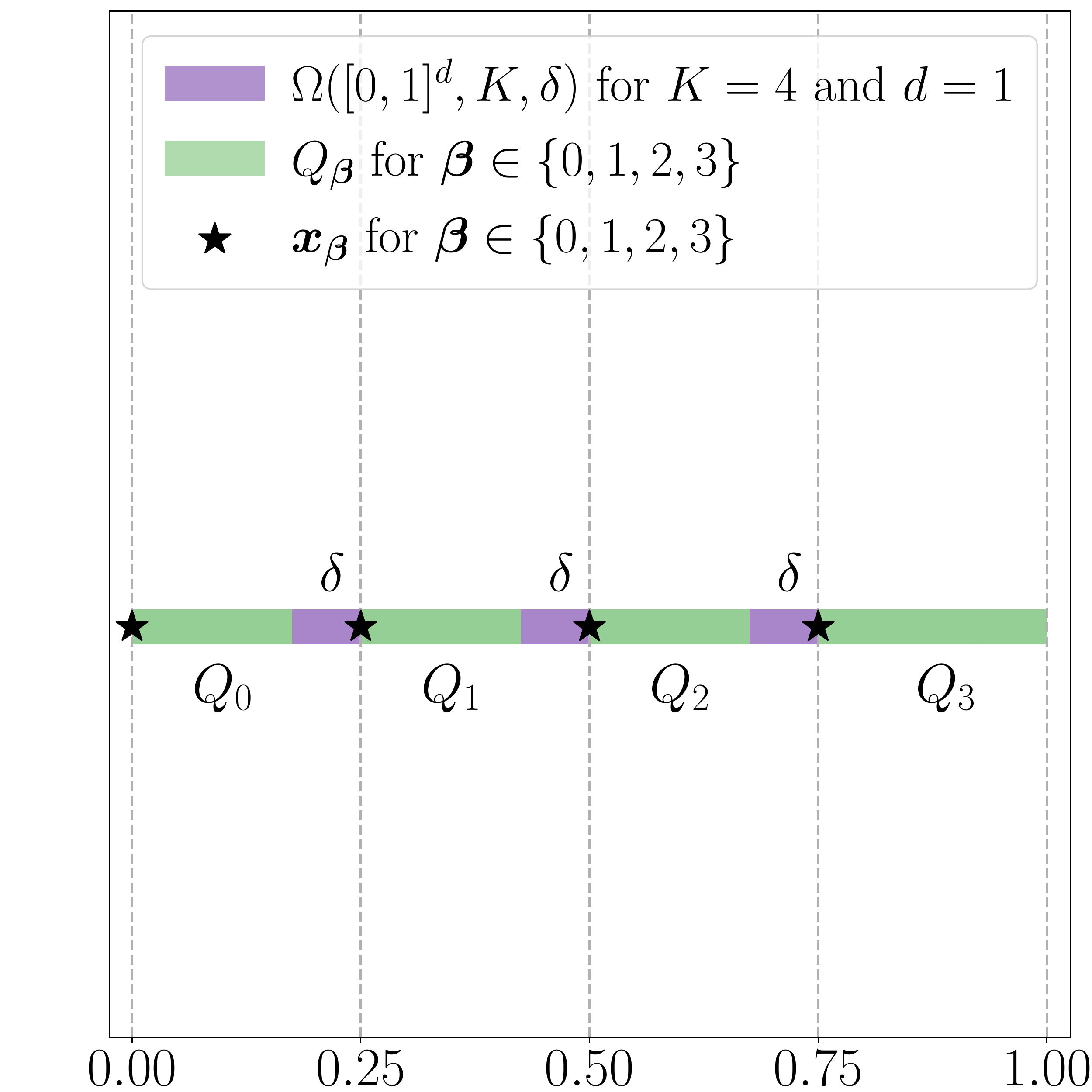}
			\subcaption{}
		\end{subfigure}
		\begin{minipage}{0.064\textwidth}
			\hspace{2pt}
		\end{minipage}
		\begin{subfigure}[b]{0.435\textwidth}
			\centering
			\includegraphics[width=0.75\textwidth]{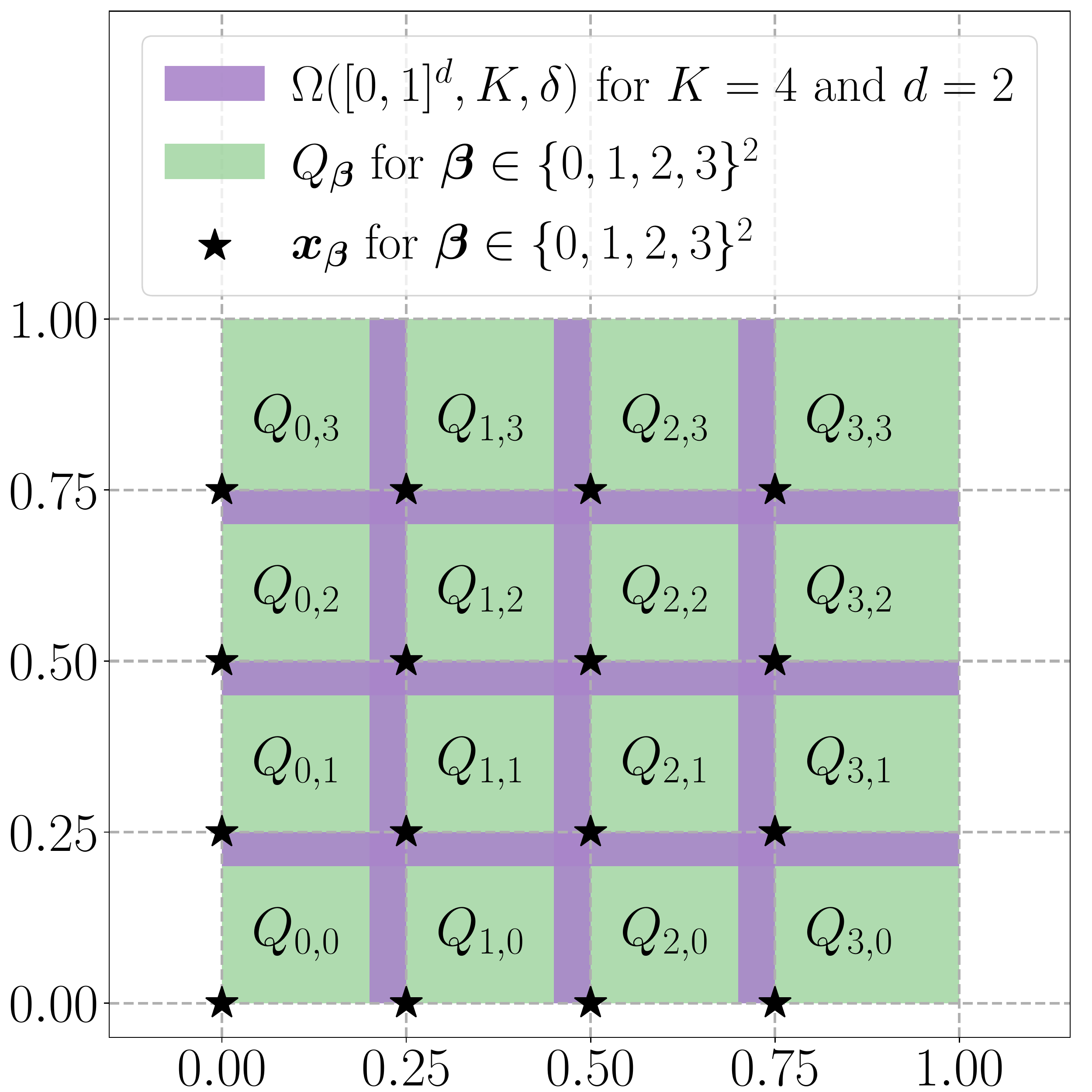}
			\subcaption{}
		\end{subfigure}
	\end{minipage}
	\caption{Illustrations of  $\Omega([0,1]^d,K,\delta)$,  $Q_\bmbeta$, and $\bmx_\bmbeta$ for $\bmbeta\in \{0,1,\dotsc,K-1\}^d$. (a) $K=4$ and $d=1$. (b) $K=4$ and $d=2$. }
	\label{fig:Q+TR}
 
\end{figure}

\mystep{2}{Construct $\bmPhi_1$ mapping $\bmx\in Q_\bmbeta$ to  $\bmbeta$.}

By Proposition~\ref{prop:floor:approx} with   $m=r$ and $n=K=\lfloor r^{1/d}\rfloor \le r= m$ therein and setting $\tildedelta=K\delta$,
		there exist $\bmg_1 \in
\nn[\big]{9}{1}{5}{5}$
and two affine linear maps $\bmtildecalL_1:\R\to\R^5$ and $\tildecalL_2:\R^5\to \R$ such that
\begin{equation*}
	\tildecalL_2\circ\bmg_1^{\circ (r-1)}\circ \bmtildecalL_1(t)=k\quad \tn{for any $t\in\big[k,\,k+1-\tildedelta\cdot \one_{\{k\le K-2\}}\big]$ and $k=0,1,\dotsc,K-1$.}
\end{equation*}

Define $\bmG_1:\R^{5d}\to \R^{5d}$ via
\begin{equation*}	\bmG_1(\bmy_1,\dotsc,\bmy_d)=\Big(\bmg_1(\bmy_1),\  \dotsc,\  \bmg_1(\bmy_d)\Big)\quad \tn{for any $\bmy_1,\dotsc,\bmy_d\in \R^5$},
\end{equation*}
$\bmhatcalL_1:\R^d\to \R^{5d}$ via
\begin{equation*}	\bmhatcalL_1(x_1,\dotsc,x_d)=\bigg(\bmtildecalL_1(Kx_1),\ \dotsc,\  \bmtildecalL_1(Kx_d)\bigg)\quad \tn{for any $(x_1,\dotsc,x_d)\in \R^d$},
\end{equation*}
and $\bmhatcalL_2:\R^{5d}\to \R^{d}$ via
\begin{equation*}
	\bmhatcalL_2(\bmy_1,\dotsc,\bmy_d)=\bigg(\bmtildecalL_2(\bmy_1),\  \dotsc,\  \bmtildecalL_2(\bmy_d)\bigg)\quad \tn{for any $\bmy_1,\dotsc,\bmy_d\in \R^5$}.
\end{equation*}
It is easy to verify that 
$\bmG_1 \in
\nn[\big]{9d}{1}{5d}{5d}$.

For any $\bmx=(x_1,\dotsc,x_d)\in Q_\bmbeta$ and $\bmbeta=(\beta_1,\dotsc,\beta_d)\in \{0,1,\dotsc,K-1\}^d$, we have
\begin{equation*}
	Kx_i\in \big[\beta_i,\, \beta_i+1-K\delta \cdot \one_{\{\beta_i\le K-2\}}\big]
	=\big[\beta_i,\, \beta_i+1-\tildedelta \cdot \one_{\{\beta_i\le K-2\}}\big]
\end{equation*}
\tn{for $i=1,2,\dotsc,d$,} implying 
\begin{equation*}
	\tildecalL_2\circ\bmg_1^{\circ (r-1)}\circ \bmtildecalL_1(Kx_i)=\beta_i.
\end{equation*}

Therefore, for any $\bmx=(x_1,\dotsc,x_d)\in Q_\bmbeta$ and $\bmbeta=(\beta_1,\dotsc,\beta_d)\in \{0,1,\dotsc,K-1\}^d$, we have
\begin{equation*}
	\begin{split}
		\bmhatcalL_2\circ \bmG_1^{\circ (r-1)}\circ \bmhatcalL_1(\bmx)
		&= \bmhatcalL_2\circ \bmG_1^{\circ (r-1)}\bigg(\bmtildecalL_1(Kx_1),\    \dotsc,\    \bmtildecalL_1(Kx_d)\bigg)\\
		& = \bmhatcalL_2\bigg(\bmg_1^{\circ (r-1)}\circ\bmtildecalL_1(Kx_1),\quad  \dotsc,\quad  \bmg_1^{\circ (r-1)}\circ\bmtildecalL_1(Kx_d)\bigg)\\
		&= \bigg(\bmtildecalL_2\circ \bmg_1^{\circ (r-1)}\circ\bmtildecalL_1(Kx_1),\quad  \dotsc,\quad \bmtildecalL_2\circ \bmg_1^{\circ (r-1)}\circ\bmtildecalL_1(Kx_d)\bigg)\\
		&= (\beta_1,\, \dotsc,\, \beta_d)=\bmbeta.
	\end{split}
\end{equation*}

By defining $\bmPhi_1\coloneqq 	\bmhatcalL_2\circ \bmG_1^{\circ (r-1)}\circ \bmhatcalL_1$, we have  
\begin{equation}
	\label{eq:Phi1:output:beta}
	\bmPhi_1(\bmx)=\bmbeta\quad \tn{ for any $\bmx\in Q_\bmbeta$ and $\bmbeta\in \{0,1,\dotsc,K-1\}^d$.}
\end{equation}
%


\mystep{3}{Construct $\phi_2$ mapping $\bmbeta$ approximately to $\tildef(\bmx_\bmbeta)$.}

We will use Proposition~\ref{prop:point:fitting} to construct the desired $\phi_2$.
To meet the requirements of applying Proposition~\ref{prop:point:fitting}, we first define two auxiliary sets $\calA_1$ and $\calA_2$ as 
\begin{equation*}
	\calA_1\coloneqq \bigg\{\frac{i}{K^{d-1}}+\frac{k}{2K^d}:i=0,1,\dotsc,K^{d-1}-1\tn{\quad and \quad} k=0,1,\dotsc,K-1\bigg\}
\end{equation*}
and 
\begin{equation*}
	\calA_2\coloneqq \bigg\{\frac{i}{K^{d-1}}+\frac{K+k}{2K^d}:i=0,1,\dotsc,K^{d-1}\black{-1}\tn{\quad and \quad}k=0,1,\dotsc,K-1\bigg\}.
\end{equation*}
Clearly, 
\begin{equation*}
	\calA_1\cup\calA_2\cup\{1\}=\Big\{\frac{j}{2K^d}:j=0,1,\dotsc,2K^d\Big\}
	\quad \tn{and}\quad \calA_1\cap\calA_2=\emptyset.
\end{equation*} 
See Figure~\ref{fig:Q+TR} for an illustration of $\calA_1$ and $\calA_2$. Next, we further divide this step into two sub-steps.

\mystep{3.1}{Construct $\bmhatcalL_3$ bijectively mapping  $\{0,1,\dotsc,K-1\}^d$ to $\mathcal{A}_1$.}

Inspired by the base-$K$ representation, we define
\begin{equation}
	\hatcalL_3(\bmx)\coloneqq \frac{x_d}{2K^d}+\sum_{i=1}^{d-1}\frac{x_i}{K^i}\quad \tn{for any $\bmx=(x_1,\dotsc,x_d)\in \R^d$.}
\end{equation}
Then $\hatcalL_3$ is a linear function bijectively mapping the index set $\{0,1,\dotsc,K-1\}^d$ to
\begin{equation*}
	\begin{split}
	 \Big\{\hatcalL_3(\bmbeta):\bm{\beta}\in \{0,1,\dotsc,K-1\}^d\Big\}
		& =\bigg\{\frac{\beta_d}{2K^d}+\sum_{i=1}^{d-1}\frac{\beta_i}{K^i}:\bm{\beta}\in \{0,1,\dotsc,K-1\}^d\bigg\}\\
		&=\bigg\{\frac{i}{K^{d-1}}+\frac{k}{2K^d}:i=0,1,\dotsc,K^{d-1}\black{-1}\tn{\quad and\quad } k=0,1,\dotsc,K-1\bigg\}=\calA_1.
	\end{split}
\end{equation*}

\mystep{3.2}{Apply Proposition~\ref{prop:point:fitting} to construct a sub-network mapping $\hatcalL_3(\bmbeta)$ approximate to $\tildef(\bmx_\bmbeta)$.}

Recall that
\begin{equation*}
	\quad \Big\{\hatcalL_3(\bmbeta):\bm{\beta}\in \{0,1,\dotsc,K-1\}^d\Big\}= \calA_1
\end{equation*}
and 
\begin{equation*}
	\Big\{\frac{j}{2K^d}: j=0,1,\dotsc,2K^d\Big\}=\calA_1\cup\calA_2\cup\{1\}.
\end{equation*} 
We will use a set of  $K^d+1$ points  
\begin{equation*}
	\Big\{\big(1,\, \tildef(\bm{1})\big)\Big\}\bigcup \bigg\{\Big(\hatcalL_3(\bmbeta),\, \tildef(\bmx_\bmbeta)\Big): 
	\bmbeta\in \{0,1,\dotsc,K-1\}^d\bigg\}
	\subseteq [0,1]\times \big[0,\,2\omega_f(\sqrt{d})\big]
\end{equation*}
to construct a continuous piecewise linear function $h:[0,1]\to \big[0,\,2\omega_f(\sqrt{d})\big]$, where $\bm{1}=(1,\dotsc,1)\in \R^d$.
Precisely, we design $h$ by making it satisfy the following two conditions.
\begin{itemize}
	\item First, we set $h(1)=\tildef(\bm{1})$ and $h\big(\hatcalL_3(\bmbeta)\big)=\tildef(\bmx_\bmbeta)$ for any $\bmbeta\in \{0,1,\dotsc,K-1\}^d$, where $\bm{1}=(1,\dotsc,1)\in \R^d$.
	\item Next, we let $h$ be linear between any two adjacent points in $\calA_1\cup\{1\}$.
\end{itemize} 
See Figure~\ref{fig:h+A12} for an illustration of $h$.
%
%
%
Recall that $\omega_f(t)=\omega_\tildef(t)$ and $\omega_f(n\cdot  t)\le n\cdot \omega_f(t)$ for any $n\in\N^+$ and $t\in [0,\infty)$.
It is easy to verify that
\begin{equation*}
	\label{eq:gErrorEstimation}
	\Big|h(\tfrac{j}{2K^d})-h(\tfrac{j-1}{2K^d})\Big|\le \max\Big\{\omega_\tildef(\tfrac{\sqrt{d}}{K}),\,
	\tfrac{\omega_\tildef({\sqrt{d}})}{K}\Big\}\le  \omega_\tildef(\tfrac{\sqrt{d}}{K})=\omega_f(\tfrac{\sqrt{d}}{K})
\end{equation*} 
for $j=1,2,\dotsc,2K^d$. 

\begin{figure}[htbp!]
	\centering
	\includegraphics[width=0.78\textwidth]{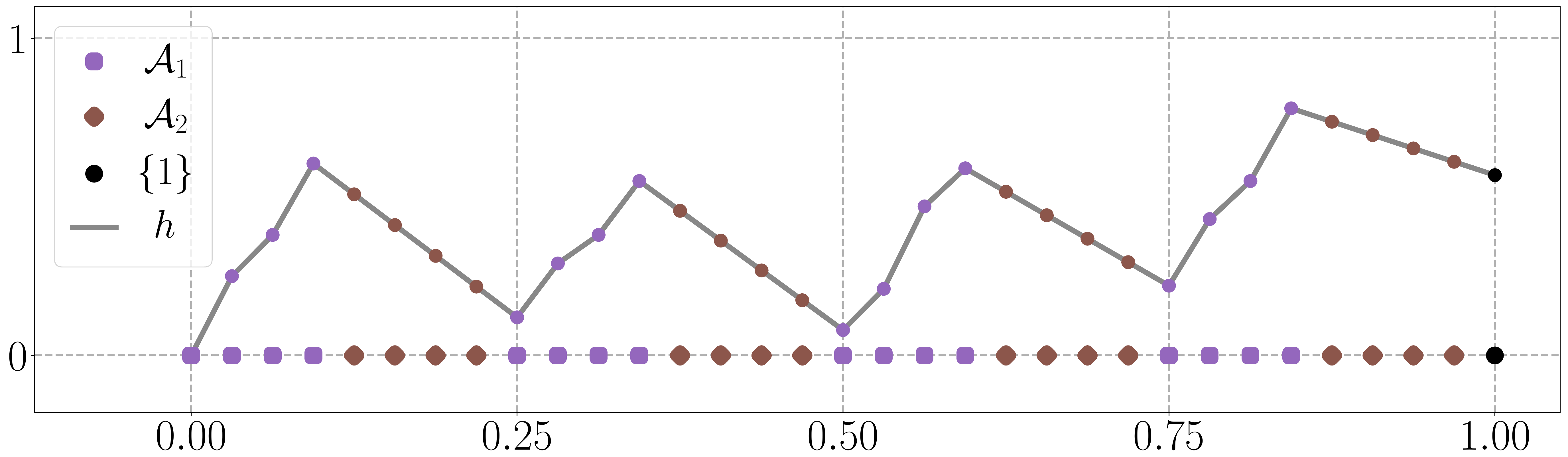}
	\caption{An illustration of $\mathcal{A}_1$, $\mathcal{A}_2$, $\{1\}$, and $h$ for $K=4$ and $d=2$.}
	\label{fig:h+A12} 
\end{figure}


By Proposition~\ref{prop:point:fitting} with $y_j=h(\tfrac{j}{2K^d})$, $\varepsilon=\omega_f(\tfrac{\sqrt{d}}{K})>0$, $m=2r$, and  $n=2K^d=2\lfloor r^{1/d}\rfloor^d\le 2r=m$ therein, there exist 
$\bmg_2\in\nn{16}{2}{6}{6}$ and two affine linear maps $\bmtildecalL_4:\R\to \R^6$ and $\hatcalL_5:\R^6\to \R$ such that
\begin{equation*}
	\Big|\hatcalL_5\circ \bmg_2^{\circ (2r-1)}\circ \bmtildecalL_4(j)-h(\tfrac{j}{2K^d})\Big|\le \omega_f(\tfrac{\sqrt{d}}{K})\quad  \tn{for } j=0,1,\dotsc,2K^d-1.
\end{equation*}

By defining $\bmhatcalL_4(x)\coloneqq \bmtildecalL_4(2K^dx)$ for any $x\in \R$, we have 
\begin{equation}
	\label{eq:L5:g2:L4-h}
	\begin{split}
		\Big|\hatcalL_5\circ \bmg_2^{\circ (2r-1)}\circ \bmhatcalL_4(\tfrac{j}{2K^d})-h(\tfrac{j}{2K^d})\Big|
		=\Big|\hatcalL_5\circ \bmg_2^{\circ (2r-1)}\circ \bmtildecalL_4(j)-h(\tfrac{j}{2K^d})\Big|\le \omega_f(\tfrac{\sqrt{d}}{K}) 
	\end{split}
\end{equation}
\tn{for $ j=0,1,\dotsc,2K^d-1$}. Then, we can define $\phi_2$ via  $\phi_2\coloneqq 	\hatcalL_5\circ \bmg_2^{\circ (2r-1)}\circ \bmhatcalL_4\circ \bmhatcalL_3$. 

By Equation~\eqref{eq:L5:g2:L4-h} and $\bmhatcalL_3(\bmbeta)\in \calA_1\subseteq \{\tfrac{j}{2K^d}:j=0,1,\dotsc,2K^d-1\}$ for any $\bmbeta\in \{0,1,\dotsc,K-1\}^d$, we have
\begin{equation}
	\label{eq:phi2-tildef}
	\begin{split}
		\big|\phi_2(\bmbeta)-\tildef(\bmx_\bmbeta)\big|
		&= \Big|\hatcalL_5\circ \bmg_2^{\circ (2r-1)}\circ \bmhatcalL_4\circ \bmhatcalL_3(\bmbeta)- \tildef(\bmx_\bmbeta)\Big|\\
		&= \Big|\hatcalL_5\circ \bmg_2^{\circ (2r-1)}\circ \bmhatcalL_4\big( \bmhatcalL_3(\bmbeta)\big) -h\big(\bmhatcalL_3(\bmbeta)\big)\Big|\le \omega_f(\tfrac{\sqrt{d}}{K}).
	\end{split}
\end{equation}

\mystep{4}{Construct the desired function $\phi$ and show it can be represented by the desired form.}

We are ready to define the desired function $\phi$ via
\begin{equation*}
	\phi\coloneqq \phi_2\circ\bmPhi_1+f(\bmzero)-\omega_f(\sqrt{d})=
	\hatcalL_5\circ \bmg_2^{\circ (2r-1)}\circ \bmhatcalL_4\circ \bmhatcalL_3\circ \bmhatcalL_2\circ \bmG_1^{\circ (r-1)}\circ \bmhatcalL_1+f(\bmzero)-\omega_f(\sqrt{d}).
\end{equation*}
By defining
$\bmhatcalL_6\coloneqq \bmhatcalL_4\circ \bmhatcalL_3\circ \bmhatcalL_2$ and $\hatcalL_7:\R^6\to \R$ via 
\begin{equation*}
	\hatcalL_7(\bmz)\coloneqq \hatcalL_5(\bmz)+f(\bmzero)-\omega_f(\sqrt{d})\quad \tn{for any $\bmz\in \R^6$,}
\end{equation*}
we have $\phi=\hatcalL_7\circ \bmg_2^{\circ (2r-1)}\circ\bmhatcalL_6\circ \bmG_1^{\circ (r-1)}\circ \bmhatcalL_1$.

Recall that $\bmG_1 \in
\nn[\big]{9d}{1}{5d}{5d}$ and $\bmg_2 \in
\nn[\big]{16}{2}{6}{6}$.
By Proposition~\ref{prop:two:blocks:three:affine} with $N_1=9d$, $N_2=16$, $L_1=1$, $L_2=2$, $d_0=d$, $d_1=5d$, $d_2=6$ and $d_3=1$ therein  and setting $\tilded=5d+1\ge \max\{5d,\, 6\}$, there exist
\begin{equation*}
	\begin{split}
			\bmg&\in\nn[\big]{9d+16+6\tilded + 2}{\max\{1+2,\,2+1\}}{\tilded+2}{\tilded+2}\\
			&= \nn[\big]{39d+24}{3}{5d+3}{5d+3}\\
	\end{split}
\end{equation*}
and two affine linear maps $\bmcalL_1:\R^d\to \R^{5d+3}$ and $\calL_2:\R^{5d+3}\to \R$ such that
\begin{equation*}
	\phi(\bmx)=\hatcalL_7\circ \bmg_2^{\circ (2r-1)}\circ\bmhatcalL_6\circ \bmG_1^{\circ (r-1)}\circ \bmhatcalL_1(\bmx)=
	\calL_2\circ \bmg^{\circ (2r-1+r-1+1)}\circ \bmcalL_1(\bmx)
	=\calL_2\circ \bmg^{\circ (3r-1)}\circ \bmcalL_1(\bmx)
\end{equation*}
for any $\bmx\in [-1,1]^d\supseteq [0,1]^d$.


Next, let us estimate the approximation error.
Recall that $f=\tildef+f(\bmzero)-\omega_f(\sqrt{d})$
 and $\phi=\phi_2\circ\bmPhi_1+f(\bmzero)-\omega_f(\sqrt{d})$. 
By Equations~\eqref{eq:Phi1:output:beta} and \eqref{eq:phi2-tildef}, for any $\bmx\in Q_\bmbeta$ and $\bmbeta\in \{0,1,\dotsc,K-1\}^d$, we have
\begin{equation*}
	\begin{split}
		\big|\calL_2\circ \bmg^{\circ (3r-1)}\circ \bmcalL_1(\bmx)-f(\bmx)\big|
		&=|\phi(\bmx)-f(\bmx)|
		=\big|\phi_2\circ\bmPhi_1(\bmx)-\tildef(\bmx)\big|=\big|\phi_2(\bmbeta)-\tildef(\bmx)\big|\\
		&\le \big|\phi_2(\bmbeta)-\tildef(\bmx_\bmbeta)\big|+ \big|\tildef(\bmx_\bmbeta)-\tildef(\bmx)\big|\\
		&\le \omega_f(\tfrac{\sqrt{d}}{K})
		+\omega_\tildef\big(\|\bmx_\bmbeta-\bmx\|_2\big)
		\le   \omega_f(\tfrac{\sqrt{d}}{K}) + \omega_\tildef(\tfrac{\sqrt{d}}{K}),
	\end{split}
\end{equation*}
where the last inequality comes from $\|\bmx_\bmbeta-\bmx\|_2\le \tfrac{\sqrt{d}}{K}$.

Recall that  $K=\lfloor r^{1/d}\rfloor \ge    \tfrac{r^{1/d}}{2}$,  $\omega_f(t)=\omega_\tildef(t)$, and $\omega_f(n\cdot  t)\le n\cdot \omega_f(t)$ for any $n\in\N^+$ and $t\in [0,\infty)$. Therefore, for any $\bmx\in \bigcup_{\bm{\beta}\in \{0,1,\dotsc,K-1\}^d} Q_\bmbeta\black{=} [0,1]^d\backslash \Omega([0,1]^d,K,\delta)$, we have
\begin{equation*}
	\begin{split}
		\big|\calL_2\circ \bmg^{\circ (3r-1)}\circ \bmcalL_1(\bmx)-f(\bmx)\big|
		&\le  \omega_f(\tfrac{\sqrt{d}}{K}) + \omega_\tildef(\tfrac{\sqrt{d}}{K})
		\le 2 \omega_f(\tfrac{\sqrt{d}}{K})
		=2 \omega_f(\tfrac{\sqrt{d}}{\lfloor r^{1/d}\rfloor})\le  2 \omega_f(2\sqrt{d}\,r^{-1/d})\\
		&\le  2 \omega_f\big(\big\lceil 2\sqrt{d}\big\rceil\,r^{-1/d}\big)
		\le 2\big\lceil 2\sqrt{d}\,\big\rceil \omega_f(r^{-1/d})
		\le 5\sqrt{d}\,\omega_f(r^{-1/d}), 
	\end{split}
\end{equation*}
where the last equality comes from the fact $2\big\lceil 2\sqrt{n}\,\big\rceil\le 5\sqrt{n}$ for any $n\in \N^+$. 
So we finish the proof of Theorem~\ref{thm:main:gap}.
\end{proof}

%
%

\section{Proof of Proposition~\ref{prop:floor:approx}}
\label{sec:proof:prop:floor:approx}
	

 The main idea behind proving Proposition~\ref{prop:floor:approx} lies in the composition architecture of neural networks.	
 To streamline the proof, we begin by introducing a lemma, Lemma~\ref{lem:floor:approx} below, which can be seen as a weaker version of Proposition~\ref{prop:floor:approx}.

\begin{lemma}
	\label{lem:floor:approx}
	Given any $\delta\in (0,1)$ and $n\in \N^+$ with $n\ge 2$,
	there exist $\bmg \in
	\nn[\big]{9}{1}{5}{5}$
	and two affine linear maps $\bmcalL_1:\R\to\R^5$ and $\calL_2:\R^5\to \R$ such that
	\begin{equation*}
		\calL_2\circ\bmg^{\circ (n-1)}\circ \bmcalL_1(x)=\lfloor x\rfloor \quad \tn{for any $x\in \bigcup_{\ell=0}^{n-1}\big[\ell,\,\ell+1-\delta\big]$}
	\end{equation*}
\end{lemma}


We will prove
Proposition~\ref{prop:floor:approx} with Lemma~\ref{lem:floor:approx} in Section~\ref{sec:proof:prop:floor:approx:with:lemma}.
The proof of Lemma~\ref{lem:floor:approx} can be found in Section~\ref{sec:proof:lem:floor:approx}. 

\subsection{Proof of Proposition~\ref{prop:floor:approx} with Lemma~\ref{lem:floor:approx}}
\label{sec:proof:prop:floor:approx:with:lemma}

Now, let us provide the detailed proof of Proposition~\ref{prop:floor:approx} by assuming Lemma~\ref{lem:floor:approx} is true. 
\begin{proof}[Proof of Proposition~\ref{prop:floor:approx}]
	We may assume $m\ge  n\ge 2$ since $n=1$ is a trivial case.
Set $\tildedelta=\tfrac{(1-\delta)\delta}{n}\in (0,1)$. By Lemma~\ref{lem:floor:approx}, 
there exist $\bmg \in
\nn[\big]{9}{1}{5}{5}$
and two affine linear maps $\bmtildecalL_1:\R\to\R^5$ and $\calL_2:\R^5\to \R$ such that
\begin{equation}
	\label{eq:floor:r-delta}
	\calL_2\circ\bmg^{\circ (m-1)}\circ \bmtildecalL_1(y)=\lfloor y\rfloor \quad \tn{for any $y\in \bigcup_{k=0}^{m-1}\big[k,\,k+1-\tildedelta\big]\supseteq \bigcup_{k=0}^{n-1}\big[k,\,k+1-\tildedelta\big]$}.
\end{equation}
	Define $\calL_0(x)\coloneqq \tfrac{n-\delta-\tildedelta}{n}x+\delta$  for any $x\in\R$ and $\bmcalL_1\coloneqq \bmtildecalL_1\circ \calL_0$.
%
%
We claim 
\begin{equation}\label{eq:calL0:k}
	\calL_0\Big(\big[k,\,k+1-\delta\cdot\one_{\{k\le n-2\}}\big]\Big)\subseteq [k,\, k+1-\tildedelta]\quad \tn{for $k=0,1,\dotsc,n-1$.}
\end{equation}
Then,
 by Equations~\eqref{eq:floor:r-delta} and \eqref{eq:calL0:k}, 
 for any $x\in [k,\,k+1-\delta\cdot\one_{\{k\le n-2\}}]$ and $k=0,1,\dotsc,n-1$, 
we have 
\begin{equation*}
    y=\calL_0(x)\in [k,\, k+1-\tildedelta]\quad \tn{for $k=0,1,\dotsc,n-1$,}
\end{equation*}
from which we deduce
\begin{equation*}
	\begin{split}
		\calL_2\circ\bmg^{\circ (m-1)}\circ \bmcalL_1(x)
		=\calL_2\circ\bmg^{\circ (r-1)}\circ \bmtildecalL_1\circ \calL_0(x)
		&=\calL_2\circ\bmg^{\circ (m-1)}\circ \bmtildecalL_1\Big(\calL_0(x)\Big)\\
  &=\calL_2\circ\bmg^{\circ (m-1)}\circ \bmtildecalL_1(y)
		=\lfloor y\rfloor =k. 		
	\end{split}
\end{equation*}

%
%


It remains to prove Equation~\eqref{eq:calL0:k}. 
Clearly,
\begin{equation*}
    \tfrac{n-\delta-\tildedelta}{n}=\tfrac{1}{n}\big(n-\delta-\tfrac{(1-\delta)\delta}{n}\big)\ge 
\tfrac{1}{n}(1-\delta-(1-\delta)\delta)=\tfrac{1}{n}(1-\delta)^2> 0,
\end{equation*}
implying $\calL_0$ is increasing. 
To prove Equation~\eqref{eq:calL0:k},
we only need to prove
\begin{equation}\label{eq:k:calL0:left}
	k\le \calL_0(k)\quad \tn{for $k=0,1,\dotsc,n-1$ }
\end{equation}
and 
\begin{equation}\label{eq:k:calL0:right}
	 \calL_0\big(k+1-\delta\cdot\one_{\{k\le n-2\}}\big)\le k+1-\tildedelta\quad \tn{for $k=0,1,\dotsc,n-1$. }
\end{equation}

Let us first prove Equation~\eqref{eq:k:calL0:left}. Clearly, for $k=0,1,\dotsc,n-1$, we have
\begin{equation*}
	\begin{split}
		\calL_0(k)=\tfrac{n-\delta-\tildedelta}{n}k+\delta
		&=k+(-\delta-\tildedelta)\tfrac{k}{n}+\delta
		=k+ \big(-\delta-\tfrac{(1-\delta)\delta}{n}\big)\tfrac{k}{n}+\delta	\\
		&=k  + \Big(- \tfrac{k}{n}-\tfrac{(1-\delta)}{n}\tfrac{k}{n}+1\Big)\delta	
    =k  + \tfrac{-kn-(1-\delta)k+n^2}{n^2}\delta\ge k,
	\end{split}
\end{equation*}
where the inequality comes from the fact $-kn-(1-\delta)k+n^2=(n-k)n-(1-\delta)k\ge n-k\ge 0$.

Next, let us  prove Equation~\eqref{eq:k:calL0:right}. 
In the case of $k=n-1$, we have
\begin{equation*}
	\calL_0\big(k+1-\delta\cdot\one_{\{k\le n-2\}}\big)=\calL_0(n)=\tfrac{n-\delta-\tildedelta}{n}n+\delta=n-\tildedelta=k+1-\tildedelta.
\end{equation*}
In the case of $k\in \{0,1,\dotsc,n-2\}$, we have 
\begin{equation*}
	\begin{split}
		\calL_0\big(k+1-\delta\cdot\one_{\{k\le n-2\}}\big)
		&=\calL_0(k+1-\delta)
		=\tfrac{n-\delta-\tildedelta}{n}(k+1-\delta)+\delta\\
		&=\big(1-\tfrac{\delta+\tildedelta}{n}\big)(k+1-\delta)+\delta
  =(k+1-\delta)-\tfrac{\delta+\tildedelta}{n}(k+1-\delta) +\delta\\
		&= (k+1)-\tfrac{\delta+\tildedelta}{n}(k+1-\delta)\le (k+1) -\tfrac{\delta}{n}(1-\delta)= k+1-\tildedelta.
	\end{split}
\end{equation*}
So we finish the proof of Proposition~\ref{prop:floor:approx}.
\end{proof}

\subsection{Proof of Lemma~\ref{lem:floor:approx}}
\label{sec:proof:lem:floor:approx}

To ensure the completeness of the proof of Proposition~\ref{prop:floor:approx},
we now provide the proof of Lemma~\ref{lem:floor:approx}.
\begin{proof}[Proof of Lemma~\ref{lem:floor:approx}]
		
		Define
		\begin{equation*}
			\begin{split}
				h_k(x)&\coloneqq
				\sigma\big(\tfrac{k}{\delta}(x-k+\delta)\big)
				- \sigma\big(\tfrac{k}{\delta}(x-k)\big)
				+\sigma\big(\tfrac{k}{\delta}(-x+k+1)\big)
				-\sigma\big(\tfrac{k}{\delta}(-x+k+1-\delta)\big)-k\\
				&\phantom{:}=
				\sigma\big(\tfrac{kx}{\delta}-\tfrac{k^2}{\delta}+ k\big)
				-
				\sigma\big(\tfrac{kx}{\delta}-\tfrac{k^2}{\delta}\big)
				+\sigma\big(-\tfrac{kx}{\delta}+\tfrac{k^2}{\delta}+ \tfrac{k}{\delta}\big)
				-\sigma\big(-\tfrac{kx}{\delta}+\tfrac{k^2}{\delta}+ \tfrac{k}{\delta}-k\big)-k
			\end{split}
		\end{equation*}
		for $k=0,1,\dotsc,n-1$ and any $x\in \R$. See an illustration of $h_k$ in Figure~\ref{fig:hk}.

\begin{figure}[ht]
	
	\begin{center}		\includegraphics[width=0.6\columnwidth]{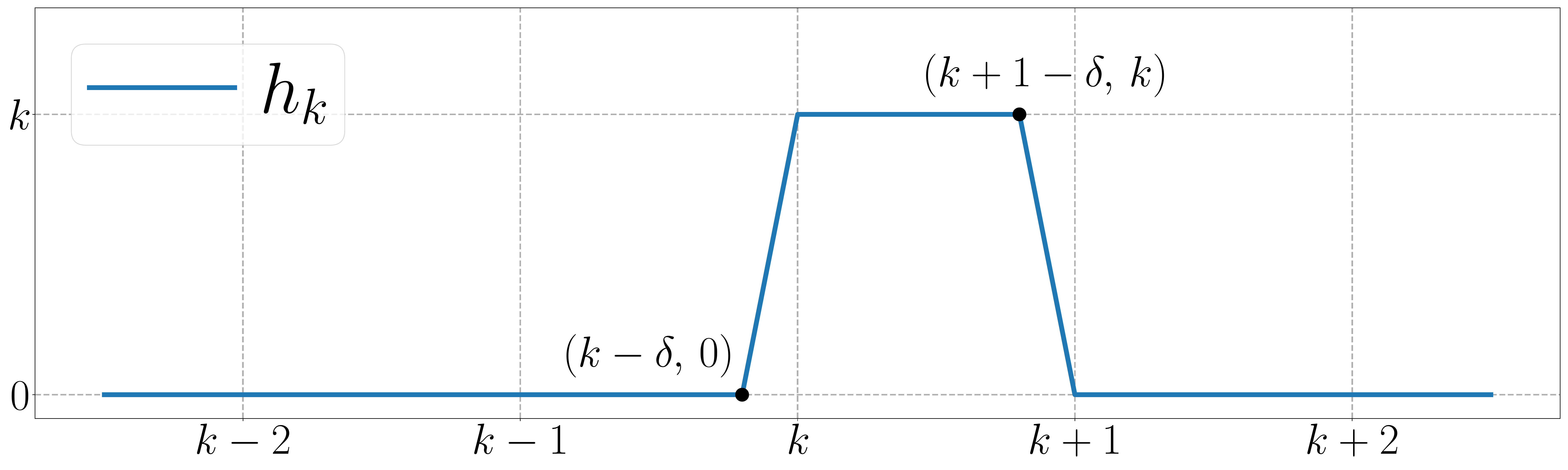}
		\caption{An illustration of $h_k$.}
		\label{fig:hk}
	\end{center}
	
\end{figure}

  It is easy to verify that
		\begin{equation*}
			h_k(x)=\begin{cases}
				k & \tn{if}\ x\in [k,k+1-\delta]\\
				0 &  \tn{if}\ x\in (-\infty,k-\delta]\cup[k+1,\infty).
			\end{cases}
		\end{equation*}
To see this, let us fix $k\in \{0,1,\dotsc,n-1\}$ and consider three cases below.
If $x\in [k,k+1-\delta]$, we have $x-k+\delta\ge 0$, $x-k\ge 0$, $-x+k+1\ge 0$, and $-x+k+1-\delta \ge 0$, implying
\begin{equation*}
	\begin{split}
		h_k(x)&=
		\sigma\big(\tfrac{k}{\delta}(x-k+\delta)\big)
		- \sigma\big(\tfrac{k}{\delta}(x-k)\big)
		+\sigma\big(\tfrac{k}{\delta}(-x+k+1)\big)
		-\sigma\big(\tfrac{k}{\delta}(-x+k+1-\delta)\big)-k\\
		&=
		\underbrace{\tfrac{k}{\delta}(x-k+\delta)
			- \tfrac{k}{\delta}(x-k)}_{=k}
		+
			\underbrace{\tfrac{k}{\delta}(-x+k+1)
		-\tfrac{k}{\delta}(-x+k+1-\delta)}_{=k}-k=k.\\
	\end{split}
\end{equation*}
	If $x\in (-\infty,k-\delta]$, we have $x-k+\delta\le 0$, $x-k\le 0$, $-x+k+1\ge 0$, and $-x+k+1-\delta \ge 0$, implying
\begin{equation*}
	\begin{split}
		h_k(x)&=
		\sigma\big(\tfrac{k}{\delta}(x-k+\delta)\big)
		- \sigma\big(\tfrac{k}{\delta}(x-k)\big)
		+\sigma\big(\tfrac{k}{\delta}(-x+k+1)\big)
		-\sigma\big(\tfrac{k}{\delta}(-x+k+1-\delta)\big)-k\\
		&=
		0
			- 0
		+
		\underbrace{\tfrac{k}{\delta}(-x+k+1)
			-\tfrac{k}{\delta}(-x+k+1-\delta)}_{=k}-k=0.\\
	\end{split}
\end{equation*}
	If $x\in [k+1,\infty)$, we have $x-k+\delta\ge 0$, $x-k\ge 0$, $-x+k+1\le 0$, and $-x+k+1-\delta \le 0$, implying
\begin{equation*}
	\begin{split}
		h_k(x)&=
		\sigma\big(\tfrac{k}{\delta}(x-k+\delta)\big)
		- \sigma\big(\tfrac{k}{\delta}(x-k)\big)
		+\sigma\big(\tfrac{k}{\delta}(-x+k+1)\big)
		-\sigma\big(\tfrac{k}{\delta}(-x+k+1-\delta)\big)-k\\
		&=
		\underbrace{\tfrac{k}{\delta}(x-k+\delta)
			- \tfrac{k}{\delta}(x-k)}_{=k}
		+
		0
			-0-k=0.\\
	\end{split}
\end{equation*}

Obviously, for any $x\in [k,k+1-\delta]$ and $k=0,1,\dotsc,n-1$, we have
\begin{equation*}
	\sum_{i=0}^{n-1}h_i(x)= h_k(x)=k=\lfloor x\rfloor.
\end{equation*}
It remains to construct $\bmg$, $\bmcalL_1$, and $\calL_2$ such that
\begin{equation*}
	\calL_2\circ \bmg^{\circ n}\circ \bmcalL_1(x) =\sum_{i=0}^{n-1}h_i(x)\quad \tn{for any $x\in \bigcup_{i=0}^{n-1}[i,\, i+1-\delta]$}.
\end{equation*}

By defining $h:\R^3 \to \R$ via 
\begin{equation*}
	h(x_1,x_2,x_3)\coloneqq \sigma\big(\tfrac{x_1}{\delta}-\tfrac{x_2}{\delta}+x_3\big)
	-\sigma\big(\tfrac{x_1}{\delta}-\tfrac{x_2}{\delta}\big)
	+\sigma\big(-\tfrac{x_1}{\delta}+\tfrac{x_2}{\delta}+\tfrac{x_3}{\delta}\big)
	-\sigma\big(-\tfrac{x_1}{\delta}+\tfrac{x_2}{\delta}+\tfrac{x_3}{\delta}-x_3\big)-\sigma(x_3),
\end{equation*}
we have
\begin{equation}
	\label{eq:h:to:hk}
	\begin{split}
		h\big(kx,k^2,k\big)=
		\sigma\big(\tfrac{kx}{\delta}-\tfrac{k^2}{\delta}+ k\big)
		-
		\sigma\big(\tfrac{kx}{\delta}-\tfrac{k^2}{\delta}\big)
		+\sigma\big(-\tfrac{kx}{\delta}+\tfrac{k^2}{\delta}+ \tfrac{k}{\delta}\big)
		-\sigma\big(-\tfrac{kx}{\delta}+\tfrac{k^2}{\delta}+ \tfrac{k}{\delta}-k\big)-k=h_k(x)
	\end{split}
\end{equation}
for $k=0,1,\dotsc,n-1$.

Now we are ready to construct $\bmg:\R^5\to\R^5$.  Define 
\begin{equation*}
	\bmg(x_1,x_2,x_3,x_4,x_5)\coloneqq
	\Bigg(\sigma(x_1+x_4), \quad
	\sigma(x_2+2x_3+1), \quad
	\sigma(x_3)+1, \quad 
	\sigma(x_4), \quad
	\sigma(x_5)+h(x_1,x_2,x_3)\Bigg),
\end{equation*}
for any $(y_1,y_2,y_3,y_4,y_5)\in \R^5$, where
\begin{equation*}
		h(x_1,x_2,x_3)= \sigma\big(\tfrac{x_1}{\delta}-\tfrac{x_2}{\delta}+x_3\big)
	-\sigma\big(\tfrac{x_1}{\delta}-\tfrac{x_2}{\delta}\big)
	+\sigma\big(-\tfrac{x_1}{\delta}+\tfrac{x_2}{\delta}+\tfrac{x_3}{\delta}\big)
	-\sigma\big(-\tfrac{x_1}{\delta}+\tfrac{x_2}{\delta}+\tfrac{x_3}{\delta}-x_3\big)-\sigma(x_3).
\end{equation*}
See an illustration of the ReLU network realizing $\bmg:\R^5\to\R^5$ in Figure~\ref{fig:g:floor:approx}.
Clearly, 
$\bmg\in \nn{9}{1}{5}{5}$.

\begin{figure}[ht]
	
	\begin{center}
		\includegraphics[width=0.9\columnwidth]{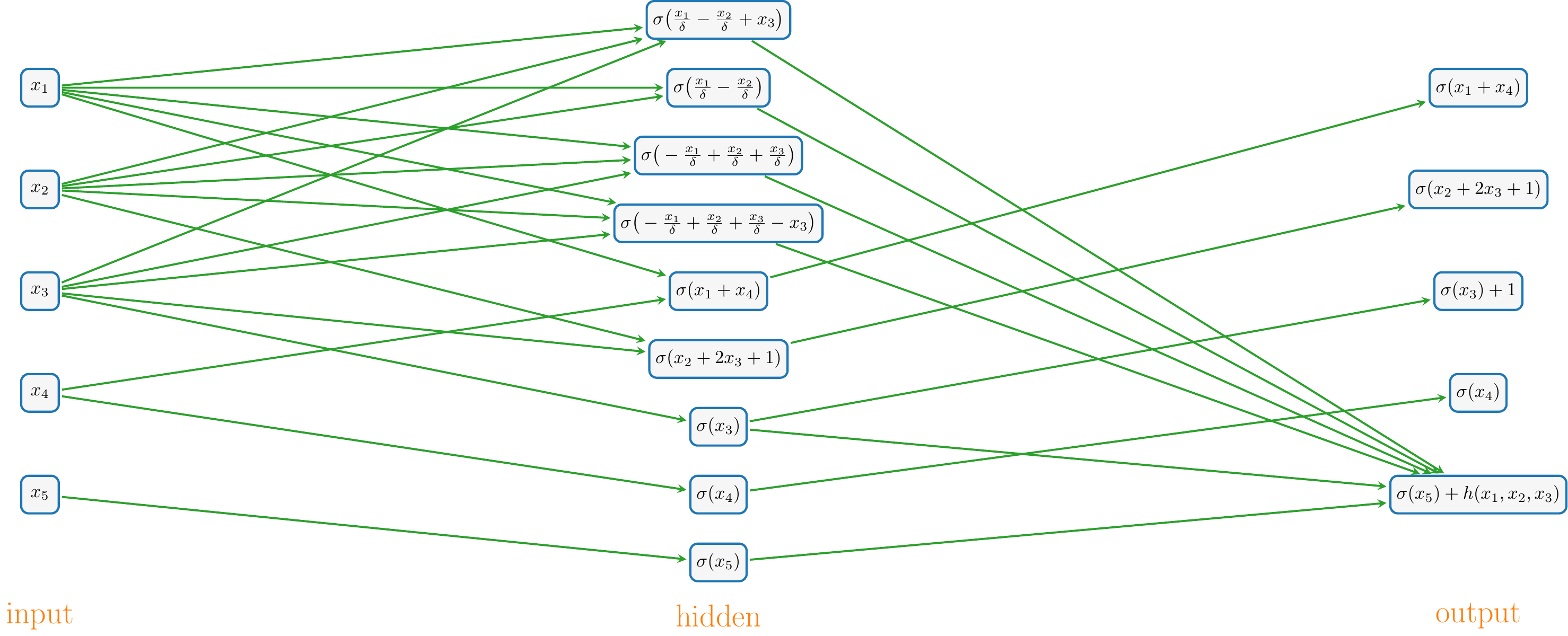}
		\caption{An illustration of $\bmg:\R^5\to\R^5$.}
		\label{fig:g:floor:approx}
	\end{center}
	
\end{figure}

Fix $x\in \cup_{i=0}^{n-1}[i,\, i+1-\delta]$ and set 
\begin{equation*}
	\bmxi_k=	\bmxi_k(x)=\Bigg(kx,\  k^2,\  k,\  x,\   \sum_{i=0}^{k-1}h_i(x)\Bigg)\in [0,\infty)^5\quad \tn{for $k=1,2,\dotsc,n$.}
\end{equation*}

For $k=1,2,\dotsc,n-1$, we have
\begin{equation*}
	\begin{split}
		\bmg(\bmxi_k)
		&=\bmg\Bigg(kx,\  k^2,\  k,\  x,\   \sum_{i=0}^{k-1}h_i(x)\Bigg)\\
		&=	\Bigg(\sigma(kx+x), \quad
		\sigma(k^2 + 2k +1), \quad
		\sigma(k)+1, \quad 
		\sigma(x), \quad
		\sigma\bigg(\sum_{i=0}^{k-1}h_i(x)\bigg)+\underbrace{h(kx,k^2,k)}_{=h_k(x)\tn{ by } \eqref{eq:h:to:hk}}\Bigg)\\
		&=	\Bigg((k+1)x, \quad
		(k+1)^2, \quad
		(k+1), \quad 
		x, \quad		\sum_{i=0}^{(k+1)-1}h_i(x)\Bigg)=\bmxi_{k+1},
	\end{split}
\end{equation*}
implying $\bmxi_{n}=\bmg(\bmxi_{n-1})=\cdots=\bmg^{\circ (n-1)}(\bmxi_1)$. 

Define $\bmcalL_1:\R \to\R^5$ via $\bmcalL_1(x)\coloneqq (x,1,1,x,0)$ and $\calL_2:\R^5\to\R$ via $\calL_2(x_1,x_2,x_3,x_4,x_5)\coloneqq x_5$. Then, 
we have 
\begin{equation*}
	\bmxi_1=\big(x,1,1,x,h_0(x)\big)=(x,1,1,x,0)=\bmcalL_1(x),
\end{equation*}
from which we deduce
\begin{equation*}
	\begin{split}
			\calL_2\circ \bmg^{\circ (n-1)}\circ \bmcalL_1(x)
		&= \calL_2\circ \bmg^{\circ (n-1)}(\bmxi_1) \\
		&= \calL_2(\bmxi_n)=\big[\bmxi_n\big]_{[5]}
	= \sum_{i=0}^{n-1}h_i(x).
	\end{split}
\end{equation*}
So we finish the proof of Lemma~\ref{lem:floor:approx}.
\end{proof}

\section{Proof of Proposition~\ref{prop:point:fitting}}
\label{sec:proof:prop:point:fitting}


The bit extraction technique proposed in \cite{Bartlett98almostlinear} plays a crucial role in proving Proposition~\ref{prop:point:fitting}. Before we delve into the proof of Proposition~\ref{prop:point:fitting}, we first establish Lemma~\ref{lem:bit:extraction}, which serves as a key intermediate step in the proof of Proposition~\ref{prop:point:fitting}.

\begin{lemma}
	\label{lem:bit:extraction}
	Given any $r\in \N^+$,
	there exist
	$\bmg\in \nn[\big]{8}{2}{3}{3}$ and two affine linear maps $\bmcalL_1:\R^2\to\R^5$ and $\calL_2:\R^5\to \R$
	such that: For any $\theta_1,\theta_2,\dotsc,\theta_{r}\in \{0,1\}$, it holds that
	\begin{equation}\label{eq:bit:extraction}
		\calL_2\circ\bmg^{\circ r}\circ \bmcalL_1\big(k,\ \bin 0.\theta_1\theta_2\cdots\theta_{r}\big)=\sum_{\ell=1}^{k}\theta_\ell \quad \tn{for $k=0,1,\dotsc,r$.}
	\end{equation}
\end{lemma}

We will prove Proposition~\ref{prop:point:fitting} by assuming the validity of Lemma~\ref{lem:bit:extraction}, which
will be proved later in Section~\ref{sec:proof:lem:bit:extraction}.

\subsection{Proof of Proposition~\ref{prop:point:fitting} with Lemma~\ref{lem:bit:extraction}}

Now we are ready to give the proof of Proposition~\ref{prop:point:fitting} by assuming Lemma~\ref{lem:bit:extraction} is true.
\begin{proof}[Proof of Proposition~\ref{prop:point:fitting}]
	
	We may assume $n=m$ since we can set $y_{n-1}=y_{n}=\cdots=y_{m-1}$ if $n<m$. 	
	Set 
	\begin{equation*}
		a_{i}= \big\lfloor \tfrac{y_{i}}{\varepsilon} \big\rfloor 
		\quad \tn{for $i=0,1,\dotsc,n-1$}
	\end{equation*}
	and 
	\begin{equation*}
		b_{i}=a_{i}-a_{i-1}\quad \tn{for $i=1,2,\dotsc,n-1$.}
	\end{equation*}

	Since $|y_{i}-y_{i -1}|\le \varepsilon$ for $i=1,2,\dotsc,n-1$, we have
	$y_{i}\in [y_{i-1}-\varepsilon,\, y_{i-1}+\varepsilon]$. Thus, for $i=1,2,\dotsc,n-1$, we have
	\begin{equation*}
		-1
		=\big\lfloor \tfrac{y_{i-1}-\eps}{\eps}\big\rfloor-\big\lfloor \tfrac{y_{i-1}}{\eps}\big\rfloor
		\le \big\lfloor \tfrac{y_{i}}{\eps}\big\rfloor-\big\lfloor \tfrac{y_{i-1}}{\eps}\big\rfloor
		\le \big\lfloor \tfrac{y_{i-1}+\eps}{\eps}\big\rfloor-\big\lfloor \tfrac{y_{i-1}}{\eps}\big\rfloor
		=1,
	\end{equation*}
	implying 
	\begin{equation*}
		b_i=a_i-a_{i-1}=\big\lfloor \tfrac{y_{i}}{\eps}\big\rfloor-\big\lfloor \tfrac{y_{i-1}}{\eps}\big\rfloor\in [-1,1].
	\end{equation*}

	It follows from $b_i=a_i-a_{i-1}\in\Z$ that
	$b_{i}\in \{-1,0,1\}$ for $i=1,2,\dotsc,n-1$.
	Hence, there exist $c_{i}\in\{0,1\}$ and $d_{i }\in\{0,1\}$ such that 
	\begin{equation*}
		b_{i}=c_{i}-d_{i}\quad \tn{ for $i=1,2,\dotsc,n-1$.}
	\end{equation*}
	
	Then, for any $k\in \{1,2,\dotsc,n-1\}$, we have
	\begin{equation*}
		\begin{split}
			a_{k}=a_{0}+\sum_{i=1}^{k }(a_{i}-a_{i-1})
			&=a_0+\sum_{i=1}^{k }(a_{i}-a_{i-1})\\
			&=a_0+\sum_{i=1}^{k }b_i
			=a_0+\sum_{i=1}^{ k }c_{i}
			-\sum_{i=1}^{ k }d_{i}.
		\end{split}
	\end{equation*}
	Clearly, $a_0=a_0+0-0=a_0+\sum_{i=1}^{0}c_i-\sum_{i=1}^{0}d_i$. Thus, we have
	\begin{equation*}
		\begin{split}
			a_{k}=a_0+\sum_{i=1}^{ k }c_{i}
			-\sum_{i=1}^{ k }d_{i}\quad \tn{for $k=0,1,\dotsc,n-1$.}
		\end{split}
	\end{equation*}
	
	
	By Lemma~\ref{lem:bit:extraction} with $r=n-1$ therein, there exist
	$\bmtildeg,\bmhatg\in \nn{8}{2}{3}{3}$ and four affine linear maps 
	$\bmtildecalL_1,\bmhatcalL_1:\R^2\to\R^3$ and $\tildecalL_2,\hatcalL_2:\R^3\to\R$ 
	such that
	\begin{equation*}
		\tildecalL_2\circ \bmtildeg^{\circ (n-1)}\circ \bmtildecalL_1\big(k,\  \bin 0.c_1\cdots  c_{n-1}\big)
		=\sum_{i=1}^{ k }c_{i}
		\quad \tn{and}\quad 
		\hatcalL_2\circ \bmhatg^{\circ (n-1)}\circ \bmhatcalL_1\big(k,\  \bin 0.d_1\cdots  d_{n-1}\big)
		=\sum_{i=1}^{ k }d_{i} 
	\end{equation*}
	for $k=0,1,\dotsc,n-1$, implying
	\begin{equation}
		\label{eq:ak:decom}
		\begin{split}
			a_{k}
			&=a_{0}
			+\sum_{i=1}^{ k }c_{i}
			-\sum_{i=1}^{ k }d_{i}\\
			&=a_0+ 	\tildecalL_2\circ \bmtildeg^{\circ (n-1)}\circ \bmtildecalL_1\big(k,\  \bin 0.c_1\cdots  c_{n-1}\big)
			-
			\hatcalL_2\circ \bmhatg^{\circ (n-1)}\circ \bmhatcalL_1\big(k,\  \bin 0.d_1\cdots  d_{n-1}\big).
		\end{split}
	\end{equation}
	
	Define $\bmg:\R^6\to\R^6$ via 
	\begin{equation*}
		\bmg(\bmx,\bmy)\coloneqq \Big(\bmtildeg(\bmx),\  \bmhatg(\bmy)\Big)\quad \tn{for any $\bmx,\bmy\in \R^3$,}
	\end{equation*}
	$\bmcalL_1:\R\to\R^6$ via
	\begin{equation*}
		\bmcalL_1(x)\coloneqq \bigg(\bmtildecalL_1(x, \  \bin 0.c_1\cdots c_{n-1})
		,\quad
		\bmhatcalL_1(x, \  \bin 0.d_1\cdots d_{n-1})\bigg)\quad \tn{for any $x\in \R$,}
	\end{equation*}
	and $\calL_{2}:\R^6\to \R$ via
	\begin{equation*}
		\calL_{2}(\bmx,\bmy)\coloneqq \eps  \Big(a_0+\tildecalL_2(\bmx)-\hatcalL_2(\bmy)\Big)\quad \tn{for any $\bmx,\bmy\in \R^3$.}
	\end{equation*}

	
	It is easy to verify that $\bmg\in \nn{16}{2}{6}{6}$. Moreover, 
	we have 
	\begin{equation*}
		\bmg^{\circ (n-1)}(\bmx,\bmy)=\Big(\bmtildeg^{\circ (n-1)}(\bmx),\  \bmhatg^{\circ (n-1)}(\bmy)\Big)\quad \tn{for any $\bmx,\bmy\in\R^3$.}
	\end{equation*}
	
	Therefore, for $k=0,1,\dotsc,n-1$, we have
	\begin{equation*}
		\begin{split}
			\calL_2\circ \bmg^{\circ (n-1)}\circ \bmcalL_1(k)
			&= \calL_2\circ \bmg^{\circ (n-1)}\bigg(\bmtildecalL_1(k, \  \bin 0.c_1\cdots c_{n-1})
			,\quad
			\bmhatcalL_1(k, \  \bin 0.d_1\cdots d_{n-1})\bigg)\\
			&=\calL_2\bigg(\bmtildeg^{\circ (n-1)}\circ\bmtildecalL_1(k, \  \bin 0.c_1\cdots c_{n-1})
			,\quad
			\bmhatg^{\circ (n-1)}\circ\bmhatcalL_1(k, \  \bin 0.d_1\cdots d_{n-1})\bigg)\\
			&= \eps \bigg(a_0 + \tildecalL_2\circ \bmtildeg^{\circ (n-1)}\circ\bmtildecalL_1(k, \  \bin 0.c_1\cdots c_{n-1})
			-
			\hatcalL_2\circ\bmhatg^{\circ (n-1)}\circ\bmhatcalL_1(k, \  \bin 0.d_1\cdots d_{n-1})\bigg)\\
			&=\eps  a_k,
		\end{split}
	\end{equation*}
	where the last equality comes from
	Equation~\eqref{eq:ak:decom}. It follows that, for $k=0,1,\dotsc,n-1$, 
	\begin{equation*}
		\Big|\calL_2\circ \bmg^{\circ (n-1)}\circ \bmcalL_1(k)-y_k\Big|
		=\big|\eps   a_k - y_k\big|
		=\Big| \eps \lfloor \tfrac{y_k}{\eps}\rfloor  - \eps \tfrac{ y_k}{\eps}\Big|
		=\eps \Big| \lfloor \tfrac{y_k}{\eps}\rfloor  - \tfrac{ y_k}{\eps} \Big|\le \eps.
	\end{equation*}
	So we finish the proof of Proposition~\ref{prop:point:fitting}.
\end{proof}

\subsection{Proof of Lemma~\ref{lem:bit:extraction}}
\label{sec:proof:lem:bit:extraction}

To make the proof of Proposition~\ref{prop:point:fitting} complete, we now provide the proof of Lemma~\ref{lem:bit:extraction}.
\begin{proof}[Proof of Lemma~\ref{lem:bit:extraction}]
	Set $\delta=2^{-r}$ and define 
	\begin{equation*}
		\calT(x)=\sigma\big(\tfrac{x}{\delta}+1\big)-\sigma\big(\tfrac{x}{\delta}\big)\quad \tn{for any $x\in\R$.}
	\end{equation*}
See an illustration of $\calT$ in Figure~\ref{fig:calT}.

\begin{figure}[ht]
	
	\begin{center}
		\includegraphics[width=0.6\columnwidth]{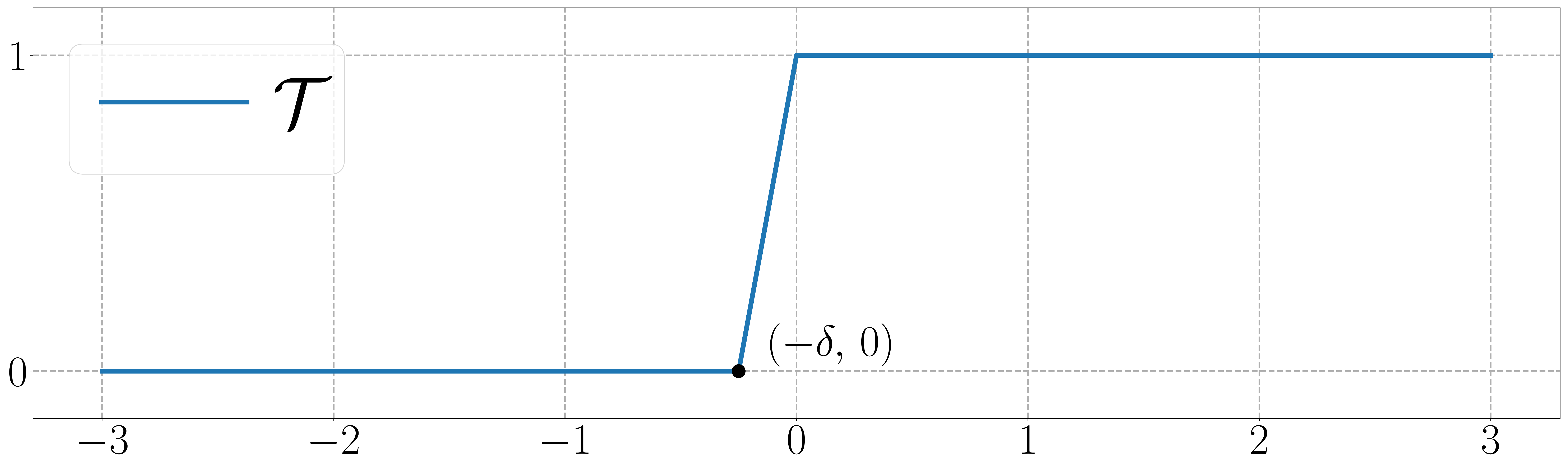}
		\caption{An illustration of $\calT$.}
		\label{fig:calT}
	\end{center}
	
\end{figure}

For any $\theta_1,\theta_2,\dotsc,\theta_{r}\in \{0,1\}$,  set
\begin{equation*}
	\beta_i=\bin 0.\theta_{i}\cdots \theta_r\quad \tn{for $i=1,2,\dotsc,r$.}
\end{equation*}
It is easy to verify that
\begin{equation*}
	\theta_i = \calT\big(\bin 0.\theta_{i}\cdots \theta_r-\tfrac{1}{2}\big)
	= \calT\big(\beta_i-\tfrac{1}{2}\big)\quad \tn{for $i=1,2,\dotsc,r$,}
\end{equation*}
implying
\begin{equation*}
	\beta_{i+1}= 2\beta_i - \theta_i = 2\beta_i- \calT\big(\beta_i-\tfrac{1}{2}\big)\quad \tn{for $i=1,2,\dotsc,r-1$.}
\end{equation*}
By setting  $\beta_{r+1}=2\beta_r-\calT(\beta_r-\tfrac{1}{2})=0$, we have
\begin{equation*}
	\beta_{i+1}= 2\beta_i- \calT\big(\beta_i-\tfrac{1}{2}\big)\quad \tn{for $i=1,2,\dotsc,r$.}
\end{equation*}

Fix $k\in \{0,1,\dotsc,r\}$.
The fact that $xy=\max\{0,\, x+y-1\}=\sigma(x+y-1)$ for any $x,y\in \{0,1\}$ implies 
\begin{equation}
	\label{eq:theta:i:sum:k}
	\begin{split}
		\sum_{i=1}^k \theta_i = 
		\sum_{i=1}^k \theta_i + \sum_{i=k+1}^r 0
		= \sum_{i=1}^{r}\theta_i\cdot\calT(k-i)
		&= \sum_{i=1}^{r}\sigma\Big(\theta_i+\calT(k-i)-1\Big)\\
		&= \sum_{i=1}^{r}\sigma\Big(\calT(\beta_i-\tfrac{1}{2})+\calT(k-i)-1\Big).
	\end{split}
\end{equation}

Define $\bmg:\R\times [0,\infty)^2\to\R^3$ via
\begin{equation*}
	\bmg(x_1,x_2,x_3)
	\coloneqq \Bigg(x_1-1,\quad  2x_2 -\calT(x_2-\tfrac{1}{2}),\quad 
	\sigma\Big(\calT(x_2-\tfrac{1}{2})+\calT(x_1)-1\Big)+x_3
	\Bigg)
\end{equation*}
for any $(x_1,x_2,x_3)\in \R\times [0,\infty)^2$.

For $\ell=1,2,\dotsc,r+1$, we set
\begin{equation*}
	\bmxi_\ell=\Bigg(k-\ell,\quad  \beta_\ell,\quad   \sum_{i=1}^{\ell-1}\sigma\Big(\calT(\beta_i-\tfrac{1}{2})+\calT(k-i)-1\Big)\Bigg)\in \R\times [0,\infty)^2.
\end{equation*}

Then, for $\ell=1,2,\dotsc,r$, we have
\begin{equation*}
	\begin{split}
		\bmg(\bmxi_\ell)
		&=\bmg\Bigg(k-\ell,\quad  \beta_\ell,\quad   \sum_{i=1}^{\ell-1}\sigma\Big(\calT(\beta_i-\tfrac{1}{2})+\calT(k-i)-1\Big)\Bigg)\\
		&=	\Bigg((k-\ell)-1, \quad
		2\beta_\ell-\calT(\beta_\ell-\tfrac{1}{2}), \quad
		\sigma\Big(\calT(\beta_\ell-\tfrac{1}{2})+\calT(k-\ell)-1\Big)
		+
		\sum_{i=1}^{\ell-1}\sigma\Big(\calT(\beta_i-\tfrac{1}{2})+\calT(k-i)-1\Big)
		\Bigg)\\
		&=	\Bigg(k-(\ell+1), \quad
		\beta_{\ell+1}, \quad
		\sum_{i=1}^{(\ell+1)-1}\sigma\Big(\calT(\beta_i-\tfrac{1}{2})+\calT(k-i)-1\Big)\Bigg)=\bmxi_{\ell+1},
	\end{split}
\end{equation*}
implying $\bmxi_{r+1}=\bmg(\bmxi_{r})=\cdots=\bmg^{\circ r}(\bmxi_1)$. 

Define $\bmcalL_1:\R^2 \to\R^3$ via $\bmcalL_1(x_1,x_2)\coloneqq(x_1-1,x_2,0)$ and $\calL_2:\R^3\to\R$ via $\calL_2(x_1,x_2,x_3)\coloneqq   x_3$ for any $x_1,x_2,x_3\in\R$. Then, 
we have 
\begin{equation*}
	\bmxi_1=\bigg(k-1,\ \beta_1,\ \sum_{i=1}^{0}\sigma\Big(\calT(\beta_i-\tfrac{1}{2})+\calT(k-i)-1\Big)\bigg)=\Big(k-1,\, \bin 0.\theta_1\cdots\theta_r,\, 0\Big)=\bmcalL_1\big(k,\,\bin 
 0.\theta_1\cdots\theta_r\big),
\end{equation*}
from which we deduce
\begin{equation*}
	\begin{split}
		\calL_2\circ \bmg^{\circ r}\circ \bmcalL_1(k,\, \bin 
 0.\theta_1\cdots\theta_r)
		&= \calL_2\circ \bmg^{\circ r}(\bmxi_1) 
		= \calL_2(\bmxi_{r+1})=\big[\bmxi_{r+1}\big]_{[3]}\\
		&= \sum_{i=1}^{(r+1)-1}\sigma\Big(\calT(\beta_i-\tfrac{1}{2})+\calT(k-i)-1\Big)=\sum_{i=1}^k\theta_i,
	\end{split}
\end{equation*}
where the last equality comes from Equation~\eqref{eq:theta:i:sum:k}. Furthermore, $\bmg$, $\bmcalL_1$, and $\calL_2$ are independent of $\theta_1,\cdots,\theta_r$.

It remains to show $\bmg$ can be realized by a ReLU network with the desired size.

\begin{figure}[ht]
	
	\begin{center}
		\includegraphics[width=0.96\columnwidth]{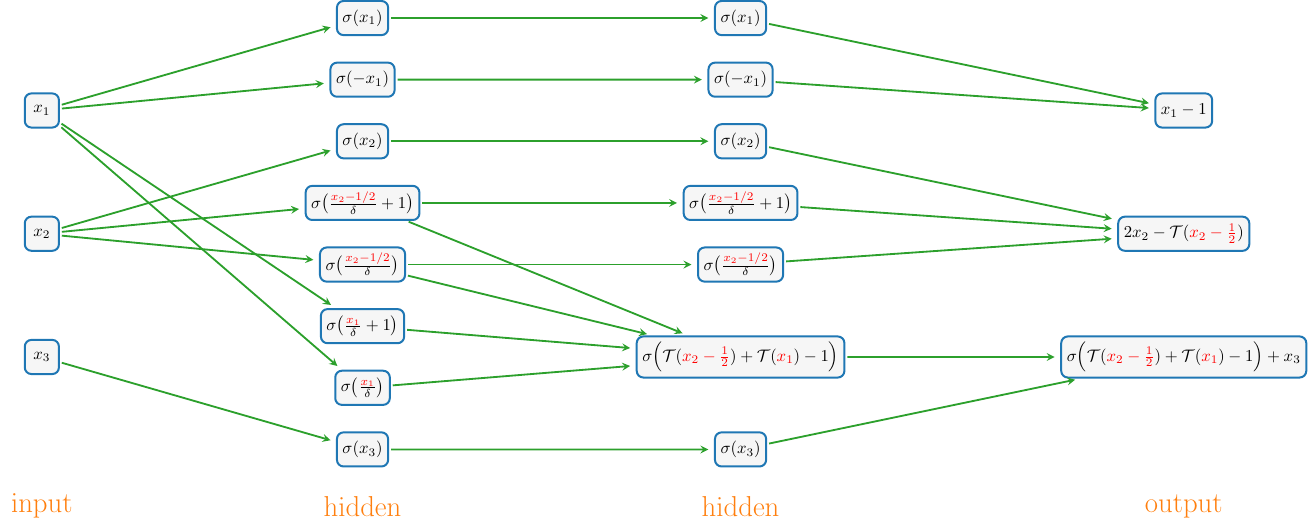}
		\caption{An illustration of the ReLU network realizing $\bmg(x_1,x_2,x_3)$ for  $(x_1,x_2,x_3)\in \R\times [0,\infty)^2$ based on $\calT(t)=\sigma\big(\tfrac{t}{\delta}+1\big)-\sigma\big(\tfrac{t}{\delta}\big)$ and  $t=\sigma(t)-\sigma(-t)$ for any $t\in\R$.}
		\label{fig:g:bits:extraction}
	\end{center}
	
\end{figure}

As shown in Figure~\ref{fig:g:bits:extraction},
$\bmg(x_1,x_2,x_3)$ can be realized by a ReLU network of width $8$  and depth $2$ for   $(x_1,x_2,x_3)\in \R\times [0,\infty)^2$. That means, $\bmg\in \nn{8}{2}{3}{3}$.
So we finish the proof of Lemma~\ref{lem:bit:extraction}.
\end{proof}


%


\section{Proof of Proposition~\ref{prop:two:blocks:three:affine} }
\label{sec:proof:prop:two:blocks:three:affine}



The objective of this section is to prove Proposition~\ref{prop:two:blocks:three:affine}. To facilitate the proof, we introduce the following lemma.

\begin{lemma}
	\label{lem:two:blocks}
	For any $A>0$, $\bmg_i\in \nn[\big]{N_i}{L_i}{d}{d}$ and $r_i\in \N^+$ for $i=1,2$,
	there exists 
	\begin{equation*}
		\bmPhi\in \nn[\Big]{N_1+N_2+2d}{\max\{L_1,L_2\}+1}{d+1}{d+1}
	\end{equation*} 
	such that
	\begin{equation*}
		\bmPhi^{\circ (r_1+r_2)}(\bmx,\,2r_1+1)=\Big(\bmg_2^{\circ r_2}\circ \bmg_1^{\circ r_1}(\bmx),\  -2r_2+1\Big)
	\end{equation*}
	for any $\bmx \in [-A,A]^d$. 
\end{lemma}
We will prove Proposition~\ref{prop:two:blocks:three:affine} with Lemma~\ref{lem:two:blocks}, which will be proved later in Section~\ref{sec:proof:lem:two:blocks}.

\subsection{Proof of Proposition~\ref{prop:two:blocks:three:affine} with Lemma~\ref{lem:two:blocks}}

Now, let us present the proof of Proposition~\ref{prop:two:blocks:three:affine} by assuming Lemmas~\ref{lem:two:blocks} is true.
\begin{proof}[Proof of Proposition~\ref{prop:two:blocks:three:affine}]
	Set 
	\begin{equation*}
		\tildeA=100(r_1+r_2+1)+ \sup_{\bmx\in [-A,A]^{d_0}}\big\|\bmtildecalL_1(\bmx)\big\|_{\ell^\infty}.
	\end{equation*}
	Since $d\ge \max\{d_1,d_2\}$, we can define $\bmhatg_1:\R^d\to\R^d$ via
	\begin{equation*}
		\bmhatg_1(\bmu,\bmv)\coloneqq \Big(\bmg_1(\bmu),\   \bmzero\Big)\in\R^d\quad\tn{for any $\bmu\in \R^{d_1}$ and $\bmv\in \R^{d-d_1}$}
	\end{equation*}
	and $\bmhatg:\R^d\to\R^d$ via
	\begin{equation*}
		\bmhatg(\bmu,\bmv)\coloneqq \Big(\bmtildecalL_2(\bmu),\   \bmzero\Big)\in\R^d\quad\tn{for any $\bmu\in \R^{d_1}$ and $\bmv\in \R^{d-d_1}$}.
	\end{equation*}
	Clearly, $\bmg_1\in \nn{N_1}{L_1}{d_1}{d_1}$ implies 
	$\bmhatg_1\in \nn{N_1}{L_1}{d}{d}$. Note that
	$\bmtildecalL_2:\R^{d_1}\to\R^{d_2}$ can be represented as 
	\begin{equation*}
		\bmtildecalL_2(\bmu)=\sigma\Big(\bmtildecalL_2(\bmu)\Big)-
		\sigma\Big(-\bmtildecalL_2(\bmu)\Big) \quad \tn{for any $\bmu\in \R^{d_1}$},
	\end{equation*} 
	which means $\bmtildecalL_2$ can be realized by a one-hidden-layer ReLU network of width $2d_2$. Thus,  $\bmhatg\in \nn{2d_2}{1}{d}{d}$.

	By Lemma~\ref{lem:two:blocks}, there exists
	$\bmG_1\in\nn{N_1+2d_2+2d}{L_1+1}{d+1}{d+1}$ such that
	\begin{equation*}
		\bmG_1^{\circ (r_1+1)}(\bmw,\,2r_1+1)=
		\Big(\bmhatg\circ \bmhatg_1^{\circ r_1}(\bmw),\,   -1\Big)
		\quad \tn{for any $\bmw\in [-\tildeA,\tildeA]^{d}$.}
	\end{equation*}
	Define $\bmhatg_2:\R^{d+1}\to\R^{d+1}$ via
	\begin{equation*}
		\bmhatg_2(\bmu,\bmv)\coloneqq \Big(\bmg_2(\bmu),\   \bmzero\Big)\in\R^{d+1}\quad\tn{for any $\bmu\in \R^{d_2}$ and $\bmv\in \R^{d+1-d_2}$}.
	\end{equation*}
	Clearly, $\bmg_2\in \nn{N_2}{L_2}{d_2}{d_2}$ implies 
	$\bmhatg_2\in \nn{N_2}{L_2}{d+1}{d+1}$.
	
	By Lemma~\ref{lem:two:blocks}, there exists
	\begin{equation*}
		\begin{split}
			\bmg&\in\nn[\Big]{(N_1+2d_2+2d)+N_2+2(d+1)}{\max\{L_1+1,L_2\}+1}{d+2}{d+2}\\
			& \subseteq  \nn[\Big]{N_1+N_2+6d+2}{\max\{L_1+2,L_2+1\}}{d+2}{d+2}
		\end{split}
	\end{equation*}
	such that
	\begin{equation*}
		\bmg^{\circ (r_1+1+r_2)}\Big(\bmz,\  2(r_1+1)+1\Big)=
		\Big(\bmhatg_2^{\circ r_2}\circ \bmG_1^{\circ (r_1+1)}(\bmz),\ -2r_2+1\Big)
		\quad \tn{for any $\bmz\in [-\tildeA,\tildeA]^{d+1}$,}
	\end{equation*}
	implying 
	\begin{equation*}
		\Big[\bmg^{\circ (r_1+r_2+1)}\Big(\bmz,\  2r_1+3\Big)\Big]_{[1:d+1]}=
		\bmhatg_2^{\circ r_2}\circ \bmG_1^{\circ (r_1+1)}(\bmz)
		\quad \tn{for any $\bmz\in [-\tildeA,\tildeA]^{d+1}$.}
	\end{equation*}

	Therefore, for any $\bmy\in [-\tildeA,\tildeA]^{d_1}$,
	we have
	\begin{equation*}
		\begin{split}
			\bigg[\bmg^{\circ (r_1+r_2+1)}\Big(
			\underbrace{\bmy,\   \bmzero,\  2r_1+1}_{\in [-\tildeA,\tildeA]^{d+1}},
			\  2r_1+3\Big)\bigg]_{[1:d+1]}
			&=
			\bmhatg_2^{\circ r_2}\circ \bmG_1^{\circ (r_1+1)}\big(
			\bmy,\   \bmzero,\  2r_1+1\big)\\
			& = \bmhatg_2^{\circ r_2}\Big(\bmhatg \circ \bmhatg_1^{\circ r_1}(\bmy,\, \bmzero),\    -1\Big) 
			= \bmhatg_2^{\circ r_2}\bigg(\bmhatg \Big(\bmg_1^{\circ r_1}(\bmy),\, \bmzero\Big),\    -1\bigg)\\
			&  = \bmhatg_2^{\circ r_2} \Big(\bmtildecalL_2\circ\bmg_1^{\circ r_1}(\bmy),\   \bmzero,\  -1\Big)  =  \Big(\bmg_2^{\circ r_2}\circ\bmtildecalL_2\circ\bmg_1^{\circ r_1}(\bmy),\   \bmzero\Big),\\
		\end{split}
	\end{equation*}
	implying
	\begin{equation*}
		\begin{split}
			\bigg[\bmg^{\circ (r_1+r_2+1)}\Big(
			\bmy,\   \bmzero,\  2r_1+1,
			\  2r_1+3\Big)\bigg]_{[1:d_2]}
			= \bmg_2^{\circ r_2}\circ\bmtildecalL_2\circ\bmg_1^{\circ r_1}(\bmy).
		\end{split}
	\end{equation*}
	
	Define $\bmcalL_1:\R^{d_0}\to \R^{d+2}$ via
	\begin{equation*}
		\bmcalL_1(\bmx)\coloneqq \Big(
		\bmtildecalL_1(\bmx),\   \bmzero,\  2r_1+1,
		\  2r_1+3\Big)\in\R^{d+2} \quad \tn{for any $\bmx\in \R^{d_0}$}
	\end{equation*}
	and 
	$\bmcalL_2:\R^{d+2}\to \R^{d_3}$ via
	\begin{equation*}
		\bmcalL_2(\bmu,\bmv)\coloneqq 
		\bmtildecalL_3(\bmu) \quad \tn{for any $\bmu\in \R^{d_2}$ and $\bmv\in \R^{d+2-d_2}$}.
	\end{equation*}
	Then, for any $\bmx\in [-A,A]^{d_0}$, we have $\bmy=\bmtildecalL_1(\bmx)\in [-\tildeA,\tildeA]^{d_1}$, implying
	\begin{equation*}
		\begin{split}
			\bmcalL_2\circ \bmg^{\circ (r_1+r_2+1)}\circ \bmcalL_1(\bmx)
			& = \bmcalL_2\circ \bmg^{\circ (r_1+r_2+1)}\Big(
			\bmtildecalL_1(\bmx),\   \bmzero,\  2r_1+1,
			\  2r_1+3\Big)\\
			& = \bmcalL_2\bigg(\bmg^{\circ (r_1+r_2+1)}\Big(
			\bmy,\   \bmzero,\  2r_1+1,
			\  2r_1+3\Big)\bigg)\\
			& = \bmtildecalL_3\bigg(\Big[\bmg^{\circ (r_1+r_2+1)}\Big(
			\bmy,\   \bmzero,\  2r_1+1,
			\  2r_1+3\Big)\Big]_{[1:d_2]}\bigg)\\
			&= \bmtildecalL_3\Big(\bmg_2^{\circ r_2}\circ\bmtildecalL_2\circ\bmg_1^{\circ r_1}(\bmy)\Big)
			=\bmtildecalL_3\circ \bmg_2^{\circ r_2}\circ\bmtildecalL_2\circ\bmg_1^{\circ r_1}\circ \bmtildecalL_1(\bmx).
		\end{split}
	\end{equation*}
	So we finish the proof of Proposition~\ref{prop:two:blocks:three:affine}. 
\end{proof}


\subsection{Proof of Lemma~\ref{lem:two:blocks}}
\label{sec:proof:lem:two:blocks}


The proof of Lemma~\ref{lem:two:blocks} will be provided after establishing an auxiliary lemma, namely Lemma~\ref{lem:binary:selector} below. As we shall see later, the auxiliary lemma plays a crucial role in the proof of Lemma~\ref{lem:two:blocks}.

\begin{lemma}
	\label{lem:binary:selector}
	For any $M>0$ and $d\in \N^+$, there exists $\bmphi\in \nn{2d+2}{1}{2d+1}{d+1}$
	such that
	\begin{equation*}
		\bmphi(\bmx,\bmy,t)=(\bmz,t) \quad \tn{with} \quad
		\bmz=\begin{cases}
			\bmx \quad & \tn{if}\  t\ge 1   \\
			\bmy \quad  & \tn{if}\  t\le -1
		\end{cases}
	\end{equation*}
	for any $\bmx,\bmy \in [-M,M]^d$ and $t\in (-\infty,-1]\cup[1,\infty)$.
\end{lemma}
\begin{proof}
	The key idea for proving this lemma is to use a ReLU network to realize a selector function $g:\R^3\to\R$ such that
	\begin{equation*}
		g(u,v,t)=\begin{cases}
			u \quad & \tn{if}\  t\ge 1   \\
			v \quad  & \tn{if}\  t\le -1
		\end{cases}
	\end{equation*} 
	for any $u,v\in [-M,M]$ and $t\in (-\infty,-1]\cup[1,\infty)$.
	To this end, we define
	\begin{equation}
		\label{eq:def:g:uvt}
		g(u,v,t)\coloneqq\sigma\big(u+Mt\big)+\sigma\big(v-Mt\big)-M\sigma(t)-M\sigma(-t).
	\end{equation}
	Let us verify that $g$ meets the requirements.
	
	In the case of $t\ge 1$, we have $u+Mt\ge 0$ and $v-Mt\le 0$ for any $u,v\in [-M,M]$, implying
	\begin{equation*}
		\begin{split}
			g(u,v,t)
			&=\sigma\big(u+Mt\big)+\sigma\big(v-Mt\big)-M\sigma(t)-M\sigma(-t)\\
			&= (u+Mt) + 0 - Mt -0 =u.
		\end{split}
	\end{equation*}

	In the case of $t\le -1$, we have $u+Mt\le 0$ and $v-Mt\ge 0$ for any $u,v\in [-M,M]$, implying
	\begin{equation*}
		\begin{split}
			g(u,v,t)
			&=\sigma\big(u+Mt\big)+\sigma\big(v-Mt\big)-M\sigma(t)-M\sigma(-t)\\
			&= 0 +  (v- Mt) -0  - M\cdot(-t)=v.
		\end{split}
	\end{equation*}

	Based on $g$, we can design a ReLU network to realize $\bmphi:\R^{2d+1}\to \R^{d+1}$ that maps $(\bmx,\bmy,t)$ to 
	\begin{equation*}
		(\bmz,t)=(z_1,\dotsc,z_d,t)=\Big( g(x_1,y_1,t),\ \dotsc,\ g(x_d,y_d,t),\quad   t\Big)
	\end{equation*}
	for any $\bmx=(x_1,\cdots,x_d),\bmy=(y_1,\cdots,y_d) \in [-M,M]^d$ and $t\in (-\infty,-1]\cup[1,\infty)$.
	
	\begin{figure}[htbp!]
 
		\centering
		\includegraphics[width=0.667\linewidth]{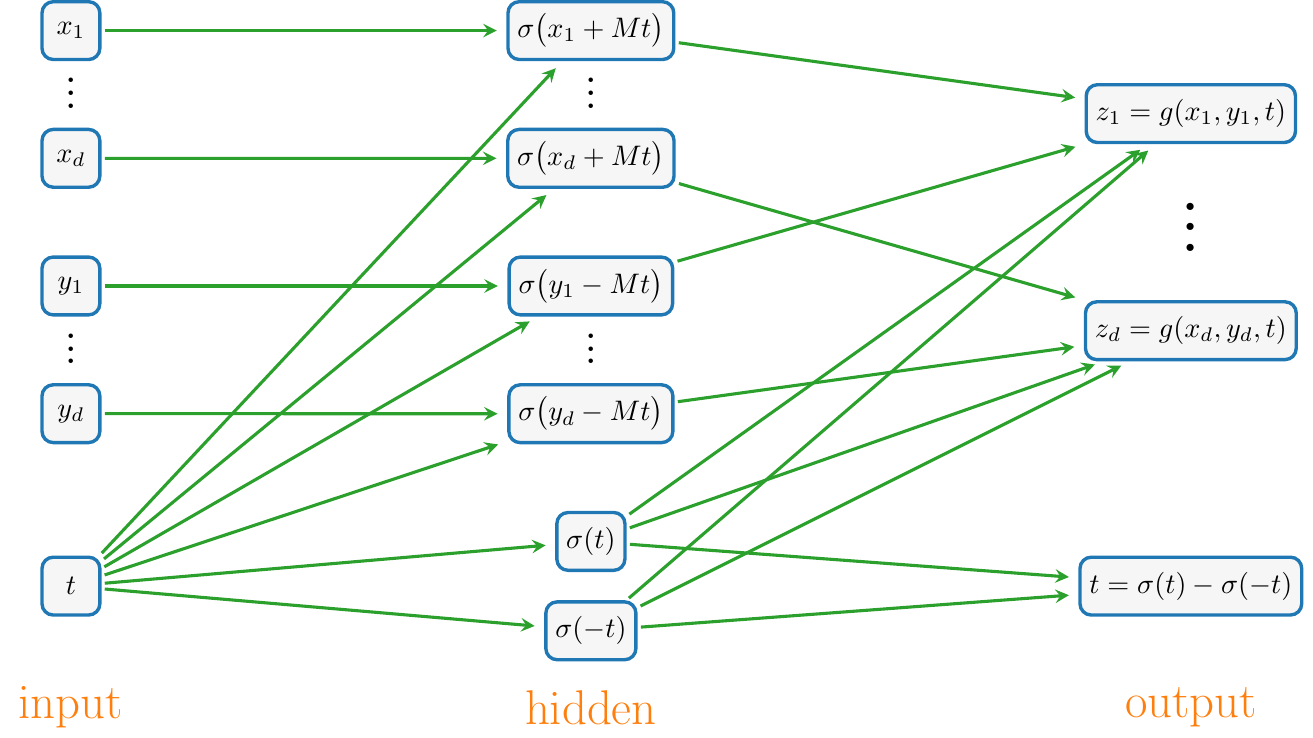}
		\caption{An illustration of the target ReLU network realizing $\bmphi$ based on Equation~\eqref{eq:def:g:uvt}.  }
		\label{fig:selector}
  
	\end{figure}
	
	We present the ReLU network realizing $\bmphi$ in Figure~\ref{fig:selector}. Clearly, 
	$\bmphi\in \nn{2d+2}{1}{2d+1}{d+1}$. So we finish the proof of Lemma~\ref{lem:binary:selector}.
	%
	%
\end{proof}

Equipped with Lemma~\ref{lem:binary:selector}, we are prepared to demonstrate the proof of Lemma~\ref{lem:two:blocks}.
\begin{proof}[Proof of Lemma~\ref{lem:two:blocks}]
	We will construct $\bmPhi=\bmpsi\circ \bmG$ via two steps below.
	\begin{itemize}
		\item First, we construct $\bmG:\R^{d+1}\to\R^{2d+1}$  by stacking $\bmg_1$, $\bmg_2$, and $g_0$, where $g_0(t)=t-2$.
		\item Next, we will apply Lemma~\ref{lem:binary:selector} to construct a selector function $\bmpsi:\R^{2d+1}\to\R^{d+1}$, determining which sub-block ($\bmg_1$ or $\bmg_2$) in $\bmG$ is used in each composition.
	\end{itemize} 
	More details can be found below.

	\mystep{1}{Constructing $\bmG$.}	
	
			Recall the $\ell^\infty$-norm of a vector $\bma=(a_1,a_2,\dotsc,a_d)\in \R^d$ is given by
	\begin{equation*}
		\|\bma\|_{\ell^\infty}=\|\bma\|_{\infty}\coloneqq \max\big\{|a_i|:i=1,2,\dotsc,d\big\}.
	\end{equation*}
	
	Set 	
		\begin{equation*}
		M=\max\{M_k:k=0,1,\dotsc,r_1+r_2\},
	\end{equation*}
	where $M_0=\max\{A,\, 100(r_1+r_2+1)\}$ and $M_k$ is given by
	\begin{equation*}
		M_k= \sup\Big\{\big\|\bmh_{k}\circ\cdots\circ\bmh_1(\bmx)\big\|_{\ell^\infty}:\bmx\in [-A,A]^d,\quad   \bmh_1,\dotsc,\bmh_k\in \{\bmg_1,\bmg_2\}\Big\}
	\end{equation*}
	for $k=1,2,\dotsc,r_1+r_2$.
	
	Define $\bmG:[-M,M]^{d+1}\to [-M,M]^{2d+1}$ via
	\begin{equation*}
		\bmG(\bmx,t)\coloneqq \Big(\bmg_1(\bmx),\   \bmg_2(\bmx),\   g_0(t)\Big)\quad \tn{for any $(\bmx,t)\in [-M,M]^{d+1}$,}
	\end{equation*}
	where $g_0(t)=t-2$. 
	Recall that $\bmg_i\in \nn[\big]{N_i}{L_i}{d}{d}$ for $i=1,2$.
	To make $\bmg_1$, $\bmg_2$,  and $g_0$ have the same number of hidden layers, we need to manually add some ``trifling'' layers.

		
		Then, by setting $L=\max\{L_1,L_2\}$ and 

		we have
		\begin{equation*}
			\bmg_1(\bmx)=\sigma^{\circ (L-L_1)}\circ (\bmg_1+M) (\bmx) -M,
		\end{equation*}
		\begin{equation*}
			\bmg_2(\bmx)=\sigma^{\circ (L-L_2)}\circ (\bmg_2+M) (\bmx) -M,
		\end{equation*}
		and 
		\begin{equation*}
			g_0(t)=\sigma^{\circ L}\circ g_0(t+M+2)-M-2
		\end{equation*}
		for any $(\bmx,t)\in [-M,M]^{d+1}$.
		Then, 
		$\bmg_i\in \nn[\big]{\max\{N_i,d\}}{L}{d}{d}$ for $i=1,2$ and 
		 $g_0\in \nnOneD[\big]{1}{L}{\R}{\R}$.
		It follows that 
		\begin{equation*}
			\bmG\in \nn[\Big]{\max\{N_1,d\}+\max\{N_2,d\}+1}{L=\max\{L_1,L_2\}}{d+1}{2d+1}.
		\end{equation*}
		
		\mystep{2}{Constructing $\bmpsi$.}
		Next, let construct a selector function $\bmpsi$ to ``select'' $\bmg_1$ or $\bmg_2$ in $\bmG$.

		By Lemma~\ref{lem:binary:selector}, there exists 
		\begin{equation*}
			\bmpsi\in \nn{2d+2}{1}{2d+1}{d+1}
		\end{equation*}
		such that
		\begin{equation}
			\label{eq:bmphi0:def}
			\bmpsi(\bmu,\bmv,t)=(\bmw,t)\quad \tn{with}\quad 
			\bmw=\begin{cases}
				\bmu \quad & \tn{if}\  t\ge 1  \\
				\bmv \quad  & \tn{if}\  t\le -1
			\end{cases}
		\end{equation}
		for any $\bmu,\bmv \in [-M,M]^d$ and $t\in (-\infty,-1]\cup[1,\infty)$.
		Then, we can define the desired $\bmPhi$ via $\bmPhi\coloneqq \bmpsi\circ \bmG$. Clearly, 
		$\bmG\in \nn[\big]{\max\{N_1,d\}+\max\{N_2,d\}+1}{\max\{L_1,L_2\}}{d+1}{2d+1}$ and  $\bmpsi\in \nn{2d+2}{1}{2d+1}{d+1}$
		imply
		\begin{equation*}
			\begin{split}
				\bmPhi=\bmpsi\circ\bmG
				&\in \nn[\bigg]{\max\Big\{\max\{N_1,d\}+\max\{N_2,d\}+1,\  2d+2\Big\}}{\max\{L_1,L_2\}+1}{d+1}{d+1}\\
				&\subseteq \nn[\Big]{N_1+N_2+2d}{\max\{L_1,L_2\}+1}{d+1}{d+1}.
			\end{split}
		\end{equation*}
		It remains to verify that
		\begin{equation*}
			\bmPhi^{\circ (r_1+r_2)}(\bmx,\,2r_1+1)=\Big(\bmg_2^{\circ r_2}\circ \bmg_1^{\circ r_1}(\bmx),\  -2r_2+1\Big)\quad \tn{for any $\bmx\in [-A,A]^d$.}
		\end{equation*}
		
		Fix $\bmx\in [-A,A]^d$
		and we can write
		\begin{equation*}
			\big(\bmxi_k,\, t_k\big)=\bmPhi^{\circ k}(\bmx,\,2r_1+1)\quad \tn{with $\bmxi_k\in \R^d$ \quad for  $k=0,1,\dotsc,r_1+r_2$},
		\end{equation*}
  where $\bmPhi^{\circ 0}$ means the identity map.   
  
		Observe that, for $k=1,2,\dotsc,r_1+r_2$,
		\begin{equation*}
			\begin{split}
				\big(\bmxi_k,\, t_k\big)
    &=
    \bmPhi^{\circ k}(\bmx,\,2r_1+1)=
    \bmPhi\circ\bmPhi^{\circ (k-1)}(\bmx,\,2r_1+1)\\
    &=\bmPhi(\bmxi_{k-1},t_{k-1})
				=\bmpsi\circ\bmG(\bmxi_{k-1},t_{k-1})
				=\bmpsi\Big(\bmg_1(\bmxi_{k-1}),\,\bmg_2(\bmxi_{k-1}),\,g_0(t_{k-1})\Big).
			\end{split}
		\end{equation*}		
 Then, by Equation~\eqref{eq:bmphi0:def}, it is easy to verify that
   $t_k=g_0(t_{k-1})=t_{k-1}-2$ and 
  $\bmg_1(\bmxi_{k-1}),\,\bmg_2(\bmxi_{k-1})\in [-M,M]^d$ 
   for $k=1,2,\cdots,r_1+r_2$,
 from which we deduce 
 \begin{equation*}
     t_k=t_0-2k=2r_1+1-2k=2(r_1-k)+1
 \end{equation*}
 and 
		\begin{equation*}
			\bmxi_{k}=\begin{cases}
				\bmg_1(\bmxi_{k-1}) &  \tn{if} \   t_k=g_0(t_{k-1})\ge 1\\ 
				\bmg_2(\bmxi_{k-1}) &  \tn{if} \   t_k=g_0(t_{k-1})\le -1\\ 
			\end{cases}
		\end{equation*}
		for $k=1,2,\dotsc,r_1+r_2$.

		Moreover, for $k=1,2,\dotsc,r_1$, we have $t_k=2(r_1-k)+1\ge 1$ and hence $\bmxi_k=\bmg_1(\bmxi_{k-1})$, implying
  		\begin{equation*}
			\begin{split}
				\bmxi_{r_1}=\bmg_1(\bmxi_{r_1-1})=\cdots=\bmg_1^{\circ r_1}(\bmxi_0)=\bmg_1^{\circ r_1}(\bmx).
			\end{split}
		\end{equation*}
		
		For $k=r_1+1,r_1+2,\dotsc,r_1+r_2$, we have $t_k=2(r_1-k)+1\le -1$ and hence $\bmxi_k=\bmg_2(\bmxi_{k-1})$, implying
		\begin{equation*}
			\bmxi_{r_1+r_2}=\bmg_2(\bmxi_{r_1+r_2-1})=\cdots=\bmg_2^{\circ r_2}(\bmxi_{r_1})
			=\bmg_2^{\circ r_2}\big(\bmg_1^{\circ r_1}(\bmx)\big)
			=\bmg_2^{\circ r_2}\circ \bmg_1^{\circ r_1}(\bmx).
		\end{equation*}
		Therefore, we have
		\begin{equation*}
			\Big[\bmPhi^{\circ (r_1+r_2)}(\bmx,\,2r_1+1)\Big]_{[1:d]}=\bmg_2^{\circ r_2}\circ \bmg_1^{\circ r_1}(\bmx)
		\end{equation*} 
	and
	\begin{equation*}
		\Big[\bmPhi^{\circ (r_1+r_2)}(\bmx,\,2r_1+1)\Big]_{[d+1]}=t_{r_1+r_2}=\big(r_1-(r_1+r_2)+1\big)=-2r_2+1,
	\end{equation*}
	from which we deduce
	\begin{equation*}
		\bmPhi^{\circ (r_1+r_2)}(\bmx,\,2r_1+1)=\Big(\bmg_2^{\circ r_2}\circ \bmg_1^{\circ r_1}(\bmx),\  -2r_2+1\Big).
	\end{equation*}
	Thus, we complete the proof of Lemma~\ref{lem:two:blocks}.
	\end{proof}

\end{document}